\documentclass[sigconf, nonacm]{acmart}

\usepackage{amsmath,amsfonts,bm}

\def\eqref#1{equation~\ref{#1}}

\def\floor#1{\lfloor #1 \rfloor}
\def\1{\bm{1}}

\DeclareMathAlphabet{\mathsfit}{\encodingdefault}{\sfdefault}{m}{sl}
\SetMathAlphabet{\mathsfit}{bold}{\encodingdefault}{\sfdefault}{bx}{n}

\newcommand{\E}{\mathbb{E}}

\newcommand{\R}{\mathbb{R}}

\DeclareMathOperator*{\argmin}{arg\,min}

\usepackage{cleveref}
\usepackage{xspace}
\usepackage{todonotes}
\usepackage{subcaption}
\usepackage{xcolor}
\usepackage{tcolorbox}
\usepackage{multirow}

\newcommand\vldbdoi{XX.XX/XXX.XX}
\newcommand\vldbpages{XXX-XXX}
\newcommand\vldbvolume{14}
\newcommand\vldbissue{1}
\newcommand\vldbyear{2020}
\newcommand\vldbauthors{\authors}
\newcommand\vldbtitle{\shorttitle} 
\newcommand\vldbavailabilityurl{https://github.com/hanlu-nju/channel_independent_MTSF}
\newcommand\vldbpagestyle{plain} 
\newcommand{\x}{{\boldsymbol{x}}}
\newcommand{\y}{{\boldsymbol{y}}}

\newcommand{\W}{{\boldsymbol{W}}}
\newcommand{\A}{{\boldsymbol{A}}}
\newcommand{\X}{{\boldsymbol{X}}}
\newcommand{\Y}{{\boldsymbol{Y}}}
\newcommand{\B}{{\boldsymbol{B}}}

\newcommand{\cR}{{\mathcal{R}}}

\def\ie{\emph{i.e.}\xspace}

\theoremstyle{plain}
\newtheorem{theorem}{Theorem}[section]
\newtheorem{proposition}[theorem]{Proposition}

\theoremstyle{definition}
\newtheorem{definition}[theorem]{Definition}

\theoremstyle{remark}

\begin{document}
\title{The Capacity and Robustness Trade-off: Revisiting the Channel Independent Strategy for Multivariate Time Series Forecasting}


\author{Lu Han, Han-Jia Ye, De-Chuan Zhan}
\affiliation{%
  \institution{State Key Laboratory for Novel Software Technology, Nanjing University}
}
\email{{hanlu,yehj}@lamda.nju.edu.cn, zhandc@nju.edu.cn}



\begin{abstract}
Multivariate time series data comprises various channels of variables. The multivariate forecasting models need to capture the relationship between the channels to accurately predict future values. However, recently, there has been an emergence of methods that employ the Channel Independent (CI) strategy. These methods view multivariate time series data as separate univariate time series and disregard the correlation between channels. Surprisingly, our empirical results have shown that models trained with the CI strategy outperform those trained with the Channel Dependent (CD) strategy, usually by a significant margin. Nevertheless, the reasons behind this phenomenon have not yet been thoroughly explored in the literature. This paper provides comprehensive empirical and theoretical analyses of the characteristics of multivariate time series datasets and the CI/CD strategy. Our results conclude that the CD approach has higher capacity but often lacks robustness to accurately predict distributionally drifted time series. In contrast, the CI approach trades capacity for robust prediction. Practical measures inspired by these analyses are proposed to address the capacity and robustness dilemma, including a modified CD method called Predict Residuals with Regularization (PRReg) that can surpass the CI strategy. We hope our findings can raise awareness among researchers about the characteristics of multivariate time series and inspire the construction of better forecasting models.

\end{abstract}

\maketitle

\pagestyle{\vldbpagestyle}
\begingroup\small\noindent\raggedright\textbf{PVLDB Reference Format:}\\
\vldbauthors. \vldbtitle. PVLDB, \vldbvolume(\vldbissue): \vldbpages, \vldbyear.\\
\href{https://doi.org/\vldbdoi}{doi:\vldbdoi}
\endgroup
\begingroup
\renewcommand\thefootnote{}\footnote{\noindent
This work is licensed under the Creative Commons BY-NC-ND 4.0 International License. Visit \url{https://creativecommons.org/licenses/by-nc-nd/4.0/} to view a copy of this license. For any use beyond those covered by this license, obtain permission by emailing \href{mailto:info@vldb.org}{info@vldb.org}. Copyright is held by the owner/author(s). Publication rights licensed to the VLDB Endowment. \\
\raggedright Proceedings of the VLDB Endowment, Vol. \vldbvolume, No. \vldbissue\ %
ISSN 2150-8097. \\
\href{https://doi.org/\vldbdoi}{doi:\vldbdoi} \\
}\addtocounter{footnote}{-1}\endgroup
W
\ifdefempty{\vldbavailabilityurl}{}{
\vspace{.3cm}
\begingroup\small\noindent\raggedright\textbf{PVLDB Artifact Availability:}\\
The source code, data, and/or other artifacts have been made available at \url{\vldbavailabilityurl}.
\endgroup
}

\section{Introduction}
Time series forecasting is a critical area of research that finds applications in both industry and academia. Multivariate time series are common and comprise multiple channels of variates that are usually correlated, with examples ranging from stock market prices and traffic flows to solar power plant outputs and temperatures across various cities~\cite{lai2018modeling}. With the powerful representation capability of deep models, channel correlation can be implicitly learned or explicitly modeled by performing forecasting tasks~\cite{DBLP:journals/corr/abs-2004-13408,WuPL0CZ20,pvldb/WuZGHYJ21,cirstea2018correlated,CuiZCXDHZ21}. Two widely used methods for time series forecasting are recurrent neural networks (RNNs) and convolutional neural networks (CNNs). RNNs model successive time points based on the Markov assumption ~\cite{hochreiter1997long,Cho14Properties,Rangapuram18Deep},  while CNNs extract variation information along the temporal dimension using techniques such as temporal convolutional networks (TCNs)~\cite{Bai18Empirical,Franceschi19Unsupervised}. However, due to the Markov assumption in RNN and the local reception property in TCN, both of the two models are unable to capture the long-term dependencies in sequential data. Recently, Transformers with attention mechanisms have gained increasing popularity in other fields like natural language processing~\cite{Bert/NAACL/Jacob}, speech recognition~\cite{DBLP:conf/icassp/DongXX18}, and even computer vision~\cite{Dosovitskiy21Image}. Researchers have also explored the potential of Transformer models in long-term multivariate time series forecasting (MTSF) tasks~\cite{Zhou2021informer,Wu2021Autoformer,FedFormer,liu2022pyraformer}. 

Despite the significant progress made by Transformer-based methods in forecasting long-term future values, a recent paper questions the effectiveness of Transformer~\cite{Zeng2022Transformers}. The authors have demonstrated that a simple linear model can outperform all state-of-the-art Transformer-based methods. However, it is important to note that the linear model used by the authors employs a channel-independent training strategy that is different from previous works. Instead of considering all the channels as a whole, the authors train a univariate forecast model that is shared across all the channels. This training strategy is closely related to the global~\cite{David2020deepAR} or cross-learning~\cite{smyl2020hybrid} approach when there is a set of related univariate time series. Global methods assume that all the time series in the set come from the same process and fit a single univariate forecasting function~\cite{Rabanser20Effectiveness}. Despite the heterogeneity of real-world time series, global methods have demonstrated unexpectedly good performance~\cite{laptev2017time,gasthaus2019probabilistic,nbeats}. \cite{montero2021principles} attribute the improvement to the relief of overfit by larger number of training examples. Multivariate time series can be viewed as a collection of multiple interdependent series that are synchronized in time. However, it is necessary to consider all channels of the variables in order to fully capture the characteristics of the object at each time step. Moreover, disregarding the correlation between channels can result in incomplete modeling. Therefore, the effectiveness of a channel-dependent strategy in improving the modeling of multivariate time series remains to be thoroughly investigated, along with the underlying reasons for its success.

This paper conducts a comprehensive investigation of the two training strategies that have emerged in recent works on multivariate time series forecasting. The first strategy is the Channel-Dependent (CD) approach, which predicts future values by taking into account the historical data of all the channels~\cite{Zhou2021informer,Wu2021Autoformer,FedFormer,liu2022pyraformer}. The second strategy is the Channel-Independent (CI) approach, which treats multivariate time series as separate univariate time series and constructs a multivariate forecaster using univariate forecasting functions~\cite{Zeng2022Transformers,Nie22Time}. With this strategy, the predicted value of a particular channel depends solely on its own historical values, while the other channels are ignored. Intuitively, since an object cannot be fully described by considering only one of its features, the CI is supposed to perform poorly. We test the two strategies with various kinds of machine learning algorithms on 9 commonly used long-term forecast datasets. Interestingly, our results indicate that, regardless of the algorithm used, the CI strategy consistently outperforms the CD strategy, often by a substantial margin. 

To explore the reasons behind this, we examined the linear model as an illustrative example, both theoretically and empirically. First, we observed the distribution drift in the real-world multivariate time series. Specifically, we found that the autocorrelation functions (ACFs) of each channel, which are relevant to the linear model, exhibit substantial differences between training and testing phases. Next, we demonstrated that the linear model using CI strategy relies solely on the mean of ACFs across all channels, while CD strategy relies on each ACF separately. Given that the mean ACF drifts less than most of the channel ACFs, this leads to CI strategy achieving superior performance. Our analysis led us to the conclusion that CI and CD exhibit different trade-offs in terms of capacity and robustness. Specifically, CI has lower capacity but better robustness, whereas CD is the opposite.

Through our analyses, we give some practice to improve the performance of existing algorithms. First, we propose an new objective called Predict Residuals with Regularization (PRReg). This objective is designed to address the non-robustness of the CD strategy and has demonstrated superior performance compared to both the original CD and CI strategies in the majority of cases. Furthermore, we have identified several factors that may influence algorithm performance. By taking these factors into consideration and implementing the PRReg objective, it may be possible to further enhance algorithm performance.

We conclude our contribution as follows:
\begin{itemize}
    \item We present the Channel Dependent (CD) and Channel Independent (CI) training strategies for multivariate time series forecasting, and find that CI outperformed CD by a significant margin, despite ignoring channel correlation.
    \item Through theoretical and empirical analysis on linear model, we identified that CI has high capacity and low robustness, while CD has low capacity and high robustness. In real-world non-stationary time series forecasting, robustness is more important, which explains CI's superior performance in most cases.
    \item We presented practical strategies for improving forecasting model performance, including the use of the Predict Residuals with Regularization (PRReg) objective and other factors that can influence CD and CI performance.
\end{itemize}



\section{Preliminaries}
In this section, we introduce the concepts of Multivariate Time Series Forecasting (MTSF), Channel Dependent (CD) Strategy, and Channel Independent (CI) Strategy.
\subsection{Multivariate Time Series Forecasting}
MTSF deals with time series data that contain multiple variables, or channels, at each time step. Given historical values $\X \in \R ^{L\times C}$ where $L$ represents the length of the look-back window, and $C$ is the number of channels. the goal of MTSF is to predict the future values $\Y\in \R ^{H\times C} $, where $H > 0$ is the forecast horizon.
\subsection{Channel Dependent (CD) Strategy}
 The CD strategy involves building a model that forecasts the future values of each channel by considering all the history of all the channels. Most of the multivariate forecaster employ this strategy~\cite{David2020deepAR,Zhou2021informer,Wu2021Autoformer,FedFormer}. To be specified, the objective of CD model is the minimize the forecasting risk $\cR$:
\begin{equation}
	\label{eq:multi_expect}
	\min_f \cR(f) = \min_f \E_{(\X,\Y)} \ell(f(\X),\Y).
\end{equation}
$\ell$ is the regression loss. We apply the commonly used L-2 (MSE) loss unless specified otherwise~\cite{elsayed2021we,Zhou2021informer,Wu2021Autoformer,FedFormer,Zeng2022Transformers}. To minimize the expectation objective (\cref{eq:multi_expect}), the model $f$ is trained by the empirical loss on the training set $\{(\X^{(i)}, \Y^{(i)})\}_{i=1}^N$. This is referred to as the \emph{Channel Dependent (CD) loss}:
\begin{equation}
	\label{eq:multi_emp}
	 \min_f \frac{1}{N} \sum_{i=1}^N \ell(f(\X^{(i)}),\Y^{(i)}).
\end{equation}
Here $N$ is the number of time series used for training. 

\subsection{Channel Independent (CI) Strategy}
Alternatively, multivariate time series can also be viewed as a set of multiple time series, \ie, the given look-back window $\X= \left[\x_1, \x_2, \ldots, \x_C \right]$ and the target $\Y= \left[\y_1, \y_2, \ldots, \y_C \right]$, where $\x_c \in R^L, \y_c \in R^H, 1 \le c \le C$ is history and future values of the univariate time seires for $c$-th channel. In this case, a forecast model can be learned by the following Channel Independent (CI) loss: 
\begin{equation}
	\label{eq:uni_emp}
	\min_f \frac{1}{NC} \sum_{i=1}^N \sum_{c=1}^C \ell(f(\x^{(i)}_{c}),\y^{(i)}_{c}).
\end{equation}
The CI loss is the mean of the losses of all channels, with each channel's loss being minimized independently.

In \cref{fig:multi_uni}, we illustrate the difference between the Channel Dependent (CD) and Channel Independent (CI) strategies for multivariate time series forecasting. CD takes all the channels of a time series as input and aims to capture the relationships between them, while CI handles each channel independently. It is natural to assume that CD would outperform CI, but in the next section, we demonstrate that the opposite is true across different benchmarks and algorithms, including both non-deep and deep methods.

\begin{figure*}[tp]
\includegraphics[width=0.9\linewidth]{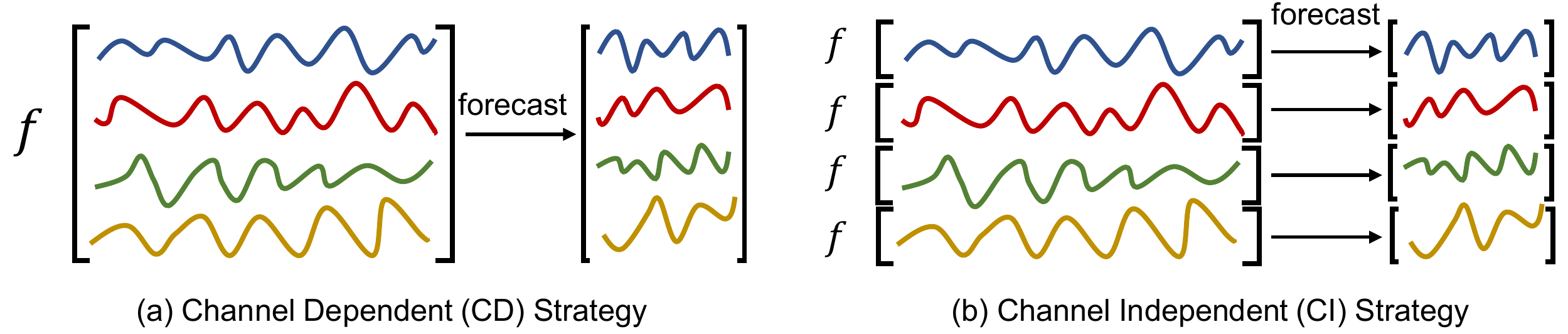}
\vspace{-4mm}
	\caption{Comparison of two training strategies for Multivariate Time Series Forecasting (MTSF) tasks. The left shows the Channel Dependent (CD) strategy where all the channels are taken as input and forecasted future values depend on the history of all the channels. The right shows the Channel Independent (CI) strategy, which treats the multivariate series as multiple univariate series and trains a unified model on these series. The prediction of each channel depends solely on its own historical values, and the relationship between different channels is ignored. }
 \vspace{-2mm}
	\label{fig:multi_uni}
\end{figure*}

\section{Empirical Comparison of CD and CI}
\label{sec:empirical_comparison}
The previous section introduced two strategies -- CD and CI -- for solving multivariate time series forecasting tasks. While CD considers all channels, one might assume that models trained with CD would outperform CI by a significant margin. Surprisingly, the opposite is true: \textbf{CI outperforms CD in most cases}. In this section, we present empirical comparisons of CD and CI across diverse datasets on various methods, including recent Transformer-based methods. Additionally, we provide both theoretical and empirical analyses to explain the reasons behind these results.

\subsection{Experiment Setup}
\paragraph{Datasets.} We conduct extensive experiments on nine widely-used, real-world datasets that cover five mainstream time series forecasting applications, namely energy, traffic, economics, weather, and disease. The datasets include: 
\begin{itemize}
	\item \textbf{ETT (Electricity Transformer Temperature)}~\cite{Zhou2021informer}\footnote{\url{https://github.com/zhouhaoyi/ETDataset}} comprises two hourly-level datasets (ETTh) and two 15-minute-level datasets (ETTm). Each dataset contains seven oil and load features of electricity transformers from July $2016$ to July $2018$.
	\item \textbf{Traffic}\footnote{\url{http://pems.dot.ca.gov}} describes the road occupancy rates. It contains the hourly data recorded by the sensors of San Francisco freeways from $2015$ to $2016$. 
	\item \textbf{Electricity}\footnote{\url{https://archive.ics.uci.edu/ml/datasets/ElectricityLoadDiagrams20112014}} collects the hourly electricity consumption of $321$ clients from $2012$ to $2014$.
	\item \textbf{Exchange-Rate}~\cite{lai2018modeling}\footnote{\url{https://github.com/laiguokun/multivariate-time-series-data}} collects the daily exchange rates of $8$ countries from $1990$ to $2016$.
	\item \textbf{Weather}\footnote{\url{https://www.bgc-jena.mpg.de/wetter/}} includes $21$ indicators of weather, such as air temperature, and humidity. Its data is recorded every 10 min for $2020$ in Germany.
	\item \textbf{ILI}\footnote{\url{https://gis.cdc.gov/grasp/fluview/fluportaldashboard.html}} describes the ratio of patients seen with influenza-like illness and the number of patients. It includes weekly data from the Centers for Disease Control and Prevention of the United States from $2002$ to $2021$.
\end{itemize}
We also summarize the datasets in~\cref{tb:dataset_summary}.

\begin{table}[htp]
\vspace{-2mm}
	\caption{Statistics of nine multivariate time series datasets.}
 \vspace{-3mm}
	\begin{tabular}{cccc}
		\toprule
		Dataset(s)    & channels & Timesteps & Granularity \\
		\midrule
		ETTh1\&ETTh2  & 7 & 17,420    & 1hour\\
		ETTm1\&ETTm2  & 7 & 69,680    & 5min \\
		Traffic& 862      & 17,544    & 1hour\\
		Electricity   & 321      & 26,304    & 1hour\\
		Exchange-Rate & 8 & 7,588     & 1day \\
		Weather& 21& 52,696    & 10min\\
		ILI  & 7 & 966& 1week      \\
		\bottomrule
	\end{tabular}
\label{tb:dataset_summary}
\vspace{-4mm}
\end{table}
\paragraph{Evaluation metrics.} In line with previous research~\cite{Zhou2021informer,Wu2021Autoformer,FedFormer,Zeng2022Transformers}, we compare the performance of different methods using two primary evaluation metrics: Mean Squared Error (MSE) and Mean Absolute Error (MAE).

\paragraph{Compared methods.} Our analysis includes a range of methods, including non-deep models, Transformer-based models, and other deep learning models. Specifically:
\begin{itemize}
	\item \textbf{Non-deep methods.} We select two popular non-deep models in recent time -- \textbf{Linear}~\cite{Zeng2022Transformers} and \textbf{GBRT}~\cite{elsayed2021we}. Following the practice in \cite{Zeng2022Transformers}, we use the auto-gradient framework~\cite{NEURIPS2019_9015_pytorch} to implement the Linear model and optimize it using gradient descent, even though it is non-deep. The results are reproduced by their codes in the  public repository~\footnote{\url{https://github.com/cure-lab/LTSF-Linear}}. For GBRT, we integrate the XGBoost~\cite{chen2016xgboost} implementation in the repository~\footnote{\url{https://github.com/Daniela-Shereen/GBRT-for-TSF}}.
	\textbf{Linear} is a representative linear model and \textbf{GBRT} is a non-linear model.
	\item \textbf{Deep methods.} Transformers are popular and enjoy rapid development in long-term multivariate forecast tasks. We include two recent Transformer-based methods: 
 \textbf{Informer}~\cite{Zhou2021informer} and traditional \textbf{Transformer}~\cite{attention_is_all_you_need}. Codes are taken from the Informer repository~\footnote{\url{https://github.com/zhouhaoyi/Informer2020}}. For generality, we also include the CNN-based method \textbf{TCN}~\cite{Bai18Empirical}, RNN-based method \textbf{DeepAR}~\cite{David2020deepAR} and a simple two-layer \textbf{MLP} model with ReLU activation. 
\end{itemize}

\paragraph{Other details.} 
For each experiment, we set the length of the look-back window to 36 for ILI and 96 for other datasets. These values follow the setup in~\cite{Zhou2021informer} and differ from the values in~\cite{Zeng2022Transformers}. When using the CD strategy for Linear and GBRT, we flatten the input as the feature for these models. Specifically, for a look-back window with $L$ time steps and $C$ channels, the input feature has a dimensionality of $LC$. However, this approach may result in high input dimensionality when dealing with datasets with many channels, such as the Traffic dataset, which has 862 channels. This can lead to computational and storage issues for dense methods like Linear. Therefore, we report only those results that are feasible for one RTX 3090 GPU.

\subsection{Main Results}
The study's results are presented in tables \ref{tb:mae_analysis} and \ref{tb:mse_analysis} where the algorithms' performance is measured by MAE and MSE, respectively. A violin plot in \cref{fig:perform_dist} illustrates the performance distribution across the seven algorithms. The study reveals several noteworthy results.

\noindent{\textbf{CI outperforms CD in the majority of cases.}} 
(1) CI significantly enhances the performance of almost all algorithms, with an average improvement of at least 20\%. OOn complex and dense algorithms like MLP, Transformer, and Informer, the improvement exceeds 30\%. Simple methods like linear experience less improvement. In most cases, replacing the CD strategy with CI yields significant improvement (>10\%). On all algorithms, the improvement is observed in more than half of the cases. Only \textbf{3} cases show a significant drop (<-10\%) in MAE and \textbf{9} in MSE, while the number of significant improvements is \textbf{92} in MAE and 95 in MSE.
(2) On most benchmarks, CI improves performance consistently. This is apparent in the left seven datasets, where the improvement is consistent. On ETTh2, CI improves performance by at least 30\%, while Weather and ILI show less improvement. Nonetheless, CI remains superior, as evidenced by the performance distribution in \cref{fig:perform_dist}.


\noindent{\textbf{CI strategy narrows the performance difference.}} \Cref{fig:perform_dist} reveals that the CI strategy has not only a lower error mean but also a smaller variance than the CD strategy. This indicates that when using the CI strategy, the model performance does not differ significantly. With the exception of Weather, methods with the CI strategy achieve the best results on the other datasets. But the best methods vary. For instance, on Electricity, GBRT and Transformer yield the best results. MLP achieves the best outcome on the 96 horizon of ETTh2. While most state-of-the-art (SOTA) results are achieved using linear, other methods are not too far behind.


\noindent{\textbf{Conclusion.}} This section demonstrates that changing the CD strategy to the CI strategy can significantly enhance the performance of multivariate forecasting methods. Hence, the superiority of some recent methods is mainly due to the training strategy rather than the algorithm's design~\cite{Zeng2022Transformers}. For a fair comparison, the training strategy and algorithm should be decoupled. The subsequent section explores why the CI strategy outperforms and elucidates the trade-off between capacity and robustness.

\begin{table*}[htp]
\centering
\caption{MAE on nine multivariate time series datasets across various forecasting models. CD means taking the Channel Dependent strategy where the algorithm takes all the channels in the look-back window as input. CI means the algorithm takes each channel as an individual univariate series and trains a shared model. For each benchmark, we mark the best model results in \textbf{bold}. We also display the improvement percentage by CI relative to CD. The significant improvement ($>10\%$) and significant drop ($<-10\%$) are marked by {\color[HTML]{FF0000} \textbf{bold red}} and {\color[HTML]{70AD47} \textbf{bold green}} respectively. The last column displays the number of {\color[HTML]{FF0000} \textbf{significant improvement/total cases}}, {\color[HTML]{70AD47} 
\textbf{significant drop/total cases}} and average improvement (\%) respectively.}
\vspace{-3mm}
\scalebox{0.67}{
\begin{tabular}{lccccccccccccccccccc}
\toprule
Dataset & \multicolumn{2}{c}{Electricity} & \multicolumn{2}{c}{ETTh1} & \multicolumn{2}{c}{ETTh2} & \multicolumn{2}{c}{ETTm1} & \multicolumn{2}{c}{ETTm2} & \multicolumn{2}{c}{Exchange\_Rate} & \multicolumn{2}{c}{Traffic} & \multicolumn{2}{c}{Weather} & \multicolumn{2}{c}{ILI} & Mean \\
Horizon & 48 & 96 & 48 & 96 & 48 & 96 & 48 & 96 & 48 & 96 & 48 & 96 & 48 & 96 & 48 & 96 & 24 & 36 &  \\
\midrule
Linear (CD) & 0.488 & 0.493 & 0.426 & 0.497 & 0.645 & 0.961 & 0.427 & 0.441 & 0.277 & 0.372 & 0.246 & 0.366 & - & - & 0.219 & 0.247 & \textbf{0.955} & \textbf{0.945} & {\color[HTML]{FF0000} \textbf{11/16}} \\
Linear (CI) & 0.275 & 0.279 & \textbf{0.374} & \textbf{0.398} & \textbf{0.302} & 0.373 & 0.382 & \textbf{0.377} & \textbf{0.251} & \textbf{0.285} & \textbf{0.164} & \textbf{0.218} & 0.428 & 0.397 & 0.229 & 0.261 & 1.163 & 1.135 & {\color[HTML]{70AD47} \textbf{2/16}} \\
Improve (\%) & {\color[HTML]{FF0000} \textbf{+43.57}} & {\color[HTML]{FF0000} \textbf{+43.39}} & {\color[HTML]{FF0000} \textbf{+12.26}} & {\color[HTML]{FF0000} \textbf{+19.87}} & {\color[HTML]{FF0000} \textbf{+53.28}} & {\color[HTML]{FF0000} \textbf{+61.15}} & {\color[HTML]{FF0000} \textbf{+10.49}} & {\color[HTML]{FF0000} \textbf{+14.35}} & +9.37 & {\color[HTML]{FF0000} \textbf{+23.29}} & {\color[HTML]{FF0000} \textbf{+33.29}} & {\color[HTML]{FF0000} \textbf{+40.52}} & - & - & -4.81 & -5.36 & {\color[HTML]{70AD47} \textbf{-21.75}} & {\color[HTML]{70AD47} \textbf{-20.09}} & {\color[HTML]{FF0000} \textbf{+19.55}} \\
\midrule
GBRT (CD) & - & - & 0.499 & 0.560 & 0.732 & 0.936 & 0.424 & 0.477 & 0.404 & 0.517 & 0.719 & 0.912 & - & - & 0.236 & 0.268 & 1.597 & 1.554 & {\color[HTML]{FF0000} \textbf{12/14}} \\
GBRT (CI) & \textbf{0.249} & 0.256 & 0.385 & 0.415 & 0.454 & 0.614 & \textbf{0.360} & 0.380 & 0.297 & 0.355 & 0.336 & 0.401 & \textbf{0.282} & 0.286 & \textbf{0.190} & \textbf{0.232} & 1.459 & 1.501 & {\color[HTML]{70AD47} \textbf{0/14}} \\
Improve (\%) & - & - & {\color[HTML]{FF0000} \textbf{+22.87}} & {\color[HTML]{FF0000} \textbf{+25.84}} & {\color[HTML]{FF0000} \textbf{+37.92}} & {\color[HTML]{FF0000} \textbf{+34.38}} & {\color[HTML]{FF0000} \textbf{+15.00}} & {\color[HTML]{FF0000} \textbf{+20.41}} & {\color[HTML]{FF0000} \textbf{+26.49}} & {\color[HTML]{FF0000} \textbf{+31.49}} & {\color[HTML]{FF0000} \textbf{+53.19}} & {\color[HTML]{FF0000} \textbf{+56.05}} & - & - & {\color[HTML]{FF0000} \textbf{+19.49}} & {\color[HTML]{FF0000} \textbf{+13.61}} & +8.62 & +3.44 & {\color[HTML]{FF0000} \textbf{+26.34}} \\\midrule
MLP (CD) & 0.385 & 0.398 & 0.523 & 0.625 & 1.028 & 1.543 & 0.480 & 0.511 & 0.439 & 0.410 & 0.617 & 0.676 & 26.834 & 26.054 & 0.218 & 0.251 & 1.161 & 1.254 & {\color[HTML]{FF0000} \textbf{12/18}} \\
MLP (CI) & 0.287 & 0.289 & 0.395 & 0.422 & 0.319 & \textbf{0.365} & 0.453 & 0.483 & 0.266 & 0.294 & 0.265 & 0.255 & 0.406 & 0.388 & 0.230 & 0.261 & 1.358 & 1.369 & {\color[HTML]{70AD47} \textbf{1/18}} \\
Improve (\%) & {\color[HTML]{FF0000} \textbf{+25.43}} & {\color[HTML]{FF0000} \textbf{+27.43}} & {\color[HTML]{FF0000} \textbf{+24.45}} & {\color[HTML]{FF0000} \textbf{+32.52}} & {\color[HTML]{FF0000} \textbf{+68.93}} & {\color[HTML]{FF0000} \textbf{+76.36}} & +5.66 & +5.38 & {\color[HTML]{FF0000} \textbf{+39.39}} & {\color[HTML]{FF0000} \textbf{+28.09}} & {\color[HTML]{FF0000} \textbf{+57.07}} & {\color[HTML]{FF0000} \textbf{+62.33}} & {\color[HTML]{FF0000} \textbf{+98.49}} & {\color[HTML]{FF0000} \textbf{+98.51}} & -5.67 & -4.00 & {\color[HTML]{70AD47} \textbf{-16.89}} & -9.14 & {\color[HTML]{FF0000} \textbf{+34.13}} \\
\midrule
DeepAR (CD) & 0.401 & 0.378 & 0.668 & 0.763 & 0.938 & 1.042 & 0.594 & 0.612 & 0.505 & 0.662 & 0.795 & 0.874 & 0.361 & 0.386 & 0.395 & 0.456 & 1.593 & 1.570 & {\color[HTML]{FF0000} \textbf{13/18}} \\
DeepAR (CI) & 0.330 & 0.342 & 0.587 & 0.594 & 0.541 & 0.597 & 0.511 & 0.520 & 0.304 & 0.354 & 0.623 & 0.660 & 0.370 & 0.410 & 0.240 & 0.287 & 1.449 & 1.454 & {\color[HTML]{70AD47} \textbf{0/18}} \\
Improve (\%) & {\color[HTML]{FF0000} \textbf{+17.72}} & +9.51 & {\color[HTML]{FF0000} \textbf{+12.14}} & {\color[HTML]{FF0000} \textbf{+22.17}} & {\color[HTML]{FF0000} \textbf{+42.28}} & {\color[HTML]{FF0000} \textbf{+42.66}} & {\color[HTML]{FF0000} \textbf{+14.03}} & {\color[HTML]{FF0000} \textbf{+15.09}} & {\color[HTML]{FF0000} \textbf{+39.76}} & {\color[HTML]{FF0000} \textbf{+46.54}} & {\color[HTML]{FF0000} \textbf{+21.65}} & {\color[HTML]{FF0000} \textbf{+24.55}} & -2.46 & -6.08 & {\color[HTML]{FF0000} \textbf{+39.29}} & {\color[HTML]{FF0000} \textbf{+37.09}} & +9.00 & +7.38 & {\color[HTML]{FF0000} \textbf{+21.80}} \\
\midrule
TCN (CD) & 0.423 & 0.440 & 0.647 & 0.746 & 0.985 & 0.985 & 0.803 & 0.712 & 0.769 & 0.841 & 0.971 & 0.955 & 0.627 & 0.637 & 0.427 & 0.399 & 1.600 & 1.482 & {\color[HTML]{FF0000} \textbf{12/18}} \\
TCN (CI) & 0.322 & 0.349 & 0.405 & 0.471 & 0.441 & 0.585 & 0.555 & 0.502 & 0.358 & 0.386 & 0.929 & 0.971 & 0.441 & 0.469 & 0.388 & 0.411 & 1.837 & 1.593 & {\color[HTML]{70AD47} \textbf{1/18}} \\
Improve (\%) & {\color[HTML]{FF0000} \textbf{+23.75}} & {\color[HTML]{FF0000} \textbf{+20.69}} & {\color[HTML]{FF0000} \textbf{+37.42}} & {\color[HTML]{FF0000} \textbf{+36.92}} & {\color[HTML]{FF0000} \textbf{+55.20}} & {\color[HTML]{FF0000} \textbf{+40.65}} & {\color[HTML]{FF0000} \textbf{+30.83}} & {\color[HTML]{FF0000} \textbf{+29.45}} & {\color[HTML]{FF0000} \textbf{+53.44}} & {\color[HTML]{FF0000} \textbf{+54.16}} & +4.29 & -1.69 & {\color[HTML]{FF0000} \textbf{+29.63}} & {\color[HTML]{FF0000} \textbf{+26.37}} & +9.26 & -2.89 & {\color[HTML]{70AD47} \textbf{-14.81}} & -7.49 & {\color[HTML]{FF0000} \textbf{+23.62}} \\\midrule
Informer (CD) & 0.424 & 0.424 & 0.766 & 0.959 & 0.906 & 1.386 & 0.477 & 0.568 & 0.428 & 0.478 & 0.717 & 0.769 & 0.403 & 0.416 & 0.402 & 0.371 & 1.565 & 1.590 & {\color[HTML]{FF0000} \textbf{15/18}} \\
Informer (CI) & 0.285 & 0.285 & 0.509 & 0.655 & 0.372 & 0.427 & 0.408 & 0.447 & 0.264 & 0.350 & 0.308 & 0.312 & 0.337 & 0.297 & 0.228 & 0.343 & 1.486 & 1.552 & {\color[HTML]{70AD47} \textbf{0/18}} \\
Improve (\%) & {\color[HTML]{FF0000} \textbf{+32.90}} & {\color[HTML]{FF0000} \textbf{+32.90}} & {\color[HTML]{FF0000} \textbf{+33.56}} & {\color[HTML]{FF0000} \textbf{+31.77}} & {\color[HTML]{FF0000} \textbf{+58.89}} & {\color[HTML]{FF0000} \textbf{+69.23}} & {\color[HTML]{FF0000} \textbf{+14.54}} & {\color[HTML]{FF0000} \textbf{+21.29}} & {\color[HTML]{FF0000} \textbf{+38.41}} & {\color[HTML]{FF0000} \textbf{+26.74}} & {\color[HTML]{FF0000} \textbf{+56.99}} & {\color[HTML]{FF0000} \textbf{+59.48}} & {\color[HTML]{FF0000} \textbf{+16.31}} & {\color[HTML]{FF0000} \textbf{+28.58}} & {\color[HTML]{FF0000} \textbf{+43.27}} & +7.31 & +5.08 & +2.40 & {\color[HTML]{FF0000} \textbf{+32.20}} \\
\midrule
Transformer (CD) & 0.352 & 0.357 & 0.734 & 0.774 & 0.829 & 1.111 & 0.458 & 0.533 & 0.404 & 0.547 & 0.571 & 0.769 & 0.364 & 0.359 & 0.343 & 0.452 & 1.508 & 1.555 & {\color[HTML]{FF0000} \textbf{17/18}} \\
Transformer (CI) & 0.281 & \textbf{0.255} & 0.565 & 0.501 & 0.347 & 0.461 & 0.407 & 0.466 & 0.254 & 0.321 & 0.227 & 0.312 & 0.303 & \textbf{0.273} & 0.232 & 0.287 & 1.348 & 1.525 & {\color[HTML]{70AD47} \textbf{0/18}} \\
Improve (\%) & {\color[HTML]{FF0000} \textbf{+20.06}} & {\color[HTML]{FF0000} \textbf{+28.66}} & {\color[HTML]{FF0000} \textbf{+23.03}} & {\color[HTML]{FF0000} \textbf{+35.36}} & {\color[HTML]{FF0000} \textbf{+58.18}} & {\color[HTML]{FF0000} \textbf{+58.51}} & {\color[HTML]{FF0000} \textbf{+11.30}} & {\color[HTML]{FF0000} \textbf{+12.62}} & {\color[HTML]{FF0000} \textbf{+37.15}} & {\color[HTML]{FF0000} \textbf{+41.27}} & {\color[HTML]{FF0000} \textbf{+60.25}} & {\color[HTML]{FF0000} \textbf{+59.48}} & {\color[HTML]{FF0000} \textbf{+16.70}} & {\color[HTML]{FF0000} \textbf{+23.78}} & {\color[HTML]{FF0000} \textbf{+32.57}} & {\color[HTML]{FF0000} \textbf{+36.49}} & {\color[HTML]{FF0000} \textbf{+10.62}} & +1.88 & {\color[HTML]{FF0000} \textbf{+31.55}}   \\
\bottomrule
\end{tabular}
}
\label{tb:mae_analysis}
\end{table*}

\begin{table*}[htp]
	\centering
	\caption{MSE on nine multivariate time series datasets across various forecasting models. }
  \vspace{-3mm}
	\scalebox{0.67}{
		\begin{tabular}{lccccccccccccccccccc}
			\toprule
Dataset & \multicolumn{2}{c}{Electricity} & \multicolumn{2}{c}{ETTh1} & \multicolumn{2}{c}{ETTh2} & \multicolumn{2}{c}{ETTm1} & \multicolumn{2}{c}{ETTm2} & \multicolumn{2}{c}{Exchange\_Rate} & \multicolumn{2}{c}{Traffic} & \multicolumn{2}{c}{Weather} & \multicolumn{2}{c}{ILI} & Mean \\
Horizon & 48 & 96 & 48 & 96 & 48 & 96 & 48 & 96 & 48 & 96 & 48 & 96 & 48 & 96 & 48 & 96 & 24 & 36 & \textbf{} \\\midrule
Linear (CD) & 0.442 & 0.444 & 0.402 & 0.514 & 0.711 & 1.520 & 0.404 & 0.433 & 0.161 & 0.269 & 0.119 & 0.274 & - & - & 0.142 & 0.165 & \textbf{2.343} & \textbf{2.436} & {\color[HTML]{FF0000} \textbf{11/16}} \\
Linear (CI) & 0.195 & 0.196 & \textbf{0.345} & \textbf{0.386} & \textbf{0.226} & \textbf{0.319} & 0.354 & \textbf{0.351} & \textbf{0.147} & \textbf{0.189} & \textbf{0.051} & \textbf{0.088} & 0.703 & 0.651 & 0.169 & 0.202 & 2.847 & 2.857 & {\color[HTML]{70AD47} \textbf{4/16}} \\
Improve (\%) & {\color[HTML]{FF0000} \textbf{+55.88}} & {\color[HTML]{FF0000} \textbf{+55.91}} & {\color[HTML]{FF0000} \textbf{+14.17}} & {\color[HTML]{FF0000} \textbf{+24.94}} & {\color[HTML]{FF0000} \textbf{+68.16}} & {\color[HTML]{FF0000} \textbf{+79.04}} & {\color[HTML]{FF0000} \textbf{+12.20}} & {\color[HTML]{FF0000} \textbf{+18.85}} & +8.45 & {\color[HTML]{FF0000} \textbf{+29.73}} & {\color[HTML]{FF0000} \textbf{+56.95}} & {\color[HTML]{FF0000} \textbf{+67.84}} & - & - & {\color[HTML]{70AD47} \textbf{-19.15}} & {\color[HTML]{70AD47} \textbf{-22.16}} & {\color[HTML]{70AD47} \textbf{-21.52}} & {\color[HTML]{70AD47} \textbf{-17.29}} & {\color[HTML]{FF0000} \textbf{+25.75}} \\\midrule
GBRT (CD) & - & - & 0.497 & 0.592 & 1.039 & 1.633 & 0.428 & 0.500 & 0.370 & 0.606 & 0.919 & 1.387 & - & - & 0.539 & 0.475 & 5.128 & 4.845 & {\color[HTML]{FF0000} \textbf{12/14}} \\
GBRT (CI) & \textbf{0.165} & 0.171 & 0.365 & 0.414 & 0.636 & 1.167 & \textbf{0.341} & 0.367 & 0.236 & 0.318 & 0.270 & 0.335 & \textbf{0.532} & 0.550 & \textbf{0.146} & \textbf{0.185} & 5.186 & 4.983 & {\color[HTML]{70AD47} \textbf{0/14}} \\
Improve (\%) & - & - & {\color[HTML]{FF0000} \textbf{+26.63}} & {\color[HTML]{FF0000} \textbf{+29.99}} & {\color[HTML]{FF0000} \textbf{+38.77}} & {\color[HTML]{FF0000} \textbf{+28.57}} & {\color[HTML]{FF0000} \textbf{+20.43}} & {\color[HTML]{FF0000} \textbf{+26.63}} & {\color[HTML]{FF0000} \textbf{+36.13}} & {\color[HTML]{FF0000} \textbf{+47.50}} & {\color[HTML]{FF0000} \textbf{+70.58}} & {\color[HTML]{FF0000} \textbf{+75.87}} & - & - & {\color[HTML]{FF0000} \textbf{+72.96}} & {\color[HTML]{FF0000} \textbf{+61.01}} & -1.13 & -2.85 & {\color[HTML]{FF0000} \textbf{+37.93}} \\\midrule
MLP (CD) & 0.293 & 0.305 & 0.517 & 0.695 & 1.664 & 3.651 & 0.453 & 0.507 & 0.323 & 0.303 & 0.590 & 0.802 & 1257.104 & 1118.137 & 0.140 & 0.167 & 2.959 & 3.494 & {\color[HTML]{FF0000} \textbf{12/18}} \\
MLP (CI) & 0.199 & 0.199 & 0.360 & 0.408 & 0.254 & 0.321 & 0.457 & 0.513 & 0.157 & 0.197 & 0.172 & 0.118 & 0.666 & 0.639 & 0.169 & 0.202 & 3.618 & 3.840 & {\color[HTML]{70AD47} \textbf{3/18}} \\
Improve (\%) & {\color[HTML]{FF0000} \textbf{+32.31}} & {\color[HTML]{FF0000} \textbf{+34.57}} & {\color[HTML]{FF0000} \textbf{+30.40}} & {\color[HTML]{FF0000} \textbf{+41.36}} & {\color[HTML]{FF0000} \textbf{+84.74}} & {\color[HTML]{FF0000} \textbf{+91.22}} & -0.73 & -1.14 & {\color[HTML]{FF0000} \textbf{+51.33}} & {\color[HTML]{FF0000} \textbf{+34.85}} & {\color[HTML]{FF0000} \textbf{+70.87}} & {\color[HTML]{FF0000} \textbf{+85.28}} & {\color[HTML]{FF0000} \textbf{+99.95}} & {\color[HTML]{FF0000} \textbf{+99.94}} & {\color[HTML]{70AD47} \textbf{-21.35}} & {\color[HTML]{70AD47} \textbf{-21.37}} & {\color[HTML]{70AD47} \textbf{-22.28}} & -9.88 & {\color[HTML]{FF0000} \textbf{+37.78}} \\\midrule
DeepAR (CD) & 0.316 & 0.293 & 0.755 & 0.918 & 1.326 & 1.609 & 0.736 & 0.735 & 0.444 & 0.747 & 0.912 & 1.093 & 0.644 & 0.691 & 0.380 & 0.473 & 5.593 & 5.418 & {\color[HTML]{FF0000} \textbf{14/18}} \\
DeepAR (CI) & 0.231 & 0.247 & 0.723 & 0.724 & 0.601 & 0.714 & 0.616 & 0.566 & 0.200 & 0.268 & 0.824 & 0.878 & 0.641 & 0.708 & 0.173 & 0.221 & 4.590 & 4.501 & {\color[HTML]{70AD47} \textbf{0/18}} \\
Improve (\%) & {\color[HTML]{FF0000} \textbf{+26.73}} & {\color[HTML]{FF0000} \textbf{+15.65}} & +4.30 & {\color[HTML]{FF0000} \textbf{+21.08}} & {\color[HTML]{FF0000} \textbf{+54.71}} & {\color[HTML]{FF0000} \textbf{+55.63}} & {\color[HTML]{FF0000} \textbf{+16.26}} & {\color[HTML]{FF0000} \textbf{+22.92}} & {\color[HTML]{FF0000} \textbf{+54.91}} & {\color[HTML]{FF0000} \textbf{+64.16}} & +9.67 & {\color[HTML]{FF0000} \textbf{+19.65}} & +0.55 & -2.48 & {\color[HTML]{FF0000} \textbf{+54.38}} & {\color[HTML]{FF0000} \textbf{+53.22}} & {\color[HTML]{FF0000} \textbf{+17.94}} & {\color[HTML]{FF0000} \textbf{+16.92}} & {\color[HTML]{FF0000} \textbf{+28.12}} \\\midrule
TCN (CD) & 0.359 & 0.383 & 0.735 & 0.890 & 1.453 & 1.539 & 1.095 & 0.834 & 0.858 & 1.114 & 1.453 & 1.334 & 1.088 & 1.095 & 0.377 & 0.348 & 5.224 & 4.775 & {\color[HTML]{FF0000} \textbf{13/18}} \\
TCN (CI) & 0.258 & 0.290 & 0.401 & 0.507 & 0.404 & 0.663 & 0.614 & 0.534 & 0.251 & 0.313 & 1.488 & 1.562 & 0.784 & 0.835 & 0.290 & 0.339 & 6.671 & 5.142 & {\color[HTML]{70AD47} \textbf{2/18}} \\
Improve (\%) & {\color[HTML]{FF0000} \textbf{+28.28}} & {\color[HTML]{FF0000} \textbf{+24.47}} & {\color[HTML]{FF0000} \textbf{+45.40}} & {\color[HTML]{FF0000} \textbf{+42.99}} & {\color[HTML]{FF0000} \textbf{+72.17}} & {\color[HTML]{FF0000} \textbf{+56.94}} & {\color[HTML]{FF0000} \textbf{+43.93}} & {\color[HTML]{FF0000} \textbf{+35.90}} & {\color[HTML]{FF0000} \textbf{+70.77}} & {\color[HTML]{FF0000} \textbf{+71.87}} & -2.41 & {\color[HTML]{70AD47} \textbf{-17.11}} & {\color[HTML]{FF0000} \textbf{+27.98}} & {\color[HTML]{FF0000} \textbf{+23.73}} & {\color[HTML]{FF0000} \textbf{+23.05}} & +2.56 & {\color[HTML]{70AD47} \textbf{-27.70}} & -7.68 & {\color[HTML]{FF0000} \textbf{+28.62}} \\\midrule
Informer (CD) & 0.326 & 0.349 & 0.689 & 0.959 & 1.270 & 3.137 & 0.517 & 0.632 & 0.310 & 0.370 & 0.790 & 0.894 & 0.715 & 0.736 & 0.322 & 0.301 & 5.377 & 5.288 & {\color[HTML]{FF0000} \textbf{16/18}} \\
Informer (CI) & 0.208 & 0.183 & 0.560 & 0.532 & 0.311 & 0.382 & 0.366 & 0.426 & 0.156 & 0.262 & 0.169 & 0.190 & 0.601 & 0.549 & 0.162 & 0.260 & 4.980 & 5.254 & {\color[HTML]{70AD47} \textbf{0/18}} \\
Improve (\%) & {\color[HTML]{FF0000} \textbf{+36.07}} & {\color[HTML]{FF0000} \textbf{+47.47}} & {\color[HTML]{FF0000} \textbf{+18.67}} & {\color[HTML]{FF0000} \textbf{+44.58}} & {\color[HTML]{FF0000} \textbf{+75.49}} & {\color[HTML]{FF0000} \textbf{+87.83}} & {\color[HTML]{FF0000} \textbf{+29.29}} & {\color[HTML]{FF0000} \textbf{+32.61}} & {\color[HTML]{FF0000} \textbf{+49.70}} & {\color[HTML]{FF0000} \textbf{+29.10}} & {\color[HTML]{FF0000} \textbf{+78.64}} & {\color[HTML]{FF0000} \textbf{+78.69}} & {\color[HTML]{FF0000} \textbf{+15.95}} & {\color[HTML]{FF0000} \textbf{+25.43}} & {\color[HTML]{FF0000} \textbf{+49.74}} & {\color[HTML]{FF0000} \textbf{+13.41}} & +7.38 & +0.65 & {\color[HTML]{FF0000} \textbf{+40.04}} \\\midrule
Transformer (CD) & 0.250 & 0.257 & 0.861 & 0.966 & 1.031 & 1.868 & 0.458 & 0.554 & 0.281 & 0.520 & 0.511 & 0.659 & 0.645 & 0.650 & 0.251 & 0.423 & 5.309 & 5.406 & {\color[HTML]{FF0000} \textbf{17/18}} \\
Transformer (CI) & 0.185 & \textbf{0.163} & 0.655 & 0.533 & 0.274 & 0.466 & 0.379 & 0.496 & 0.148 & 0.237 & 0.101 & 0.137 & 0.558 & \textbf{0.526} & 0.168 & 0.225 & 4.307 & 5.033 & {\color[HTML]{70AD47} \textbf{0/18}} \\
Improve (\%) & {\color[HTML]{FF0000} \textbf{+26.10}} & {\color[HTML]{FF0000} \textbf{+36.59}} & {\color[HTML]{FF0000} \textbf{+23.85}} & {\color[HTML]{FF0000} \textbf{+44.84}} & {\color[HTML]{FF0000} \textbf{+73.43}} & {\color[HTML]{FF0000} \textbf{+75.07}} & {\color[HTML]{FF0000} \textbf{+17.32}} & {\color[HTML]{FF0000} \textbf{+10.43}} & {\color[HTML]{FF0000} \textbf{+47.27}} & {\color[HTML]{FF0000} \textbf{+54.40}} & {\color[HTML]{FF0000} \textbf{+80.27}} & {\color[HTML]{FF0000} \textbf{+79.30}} & {\color[HTML]{FF0000} \textbf{+13.43}} & {\color[HTML]{FF0000} \textbf{+19.13}} & {\color[HTML]{FF0000} \textbf{+33.14}} & {\color[HTML]{FF0000} \textbf{+46.87}} & {\color[HTML]{FF0000} \textbf{+18.88}} & +6.89 & {\color[HTML]{FF0000} \textbf{+39.29}}   \\
			\bottomrule
		\end{tabular}
	}
\label{tb:mse_analysis}
\end{table*}

\begin{figure}[htp]
    \centering
    \includegraphics[width=0.47\linewidth]{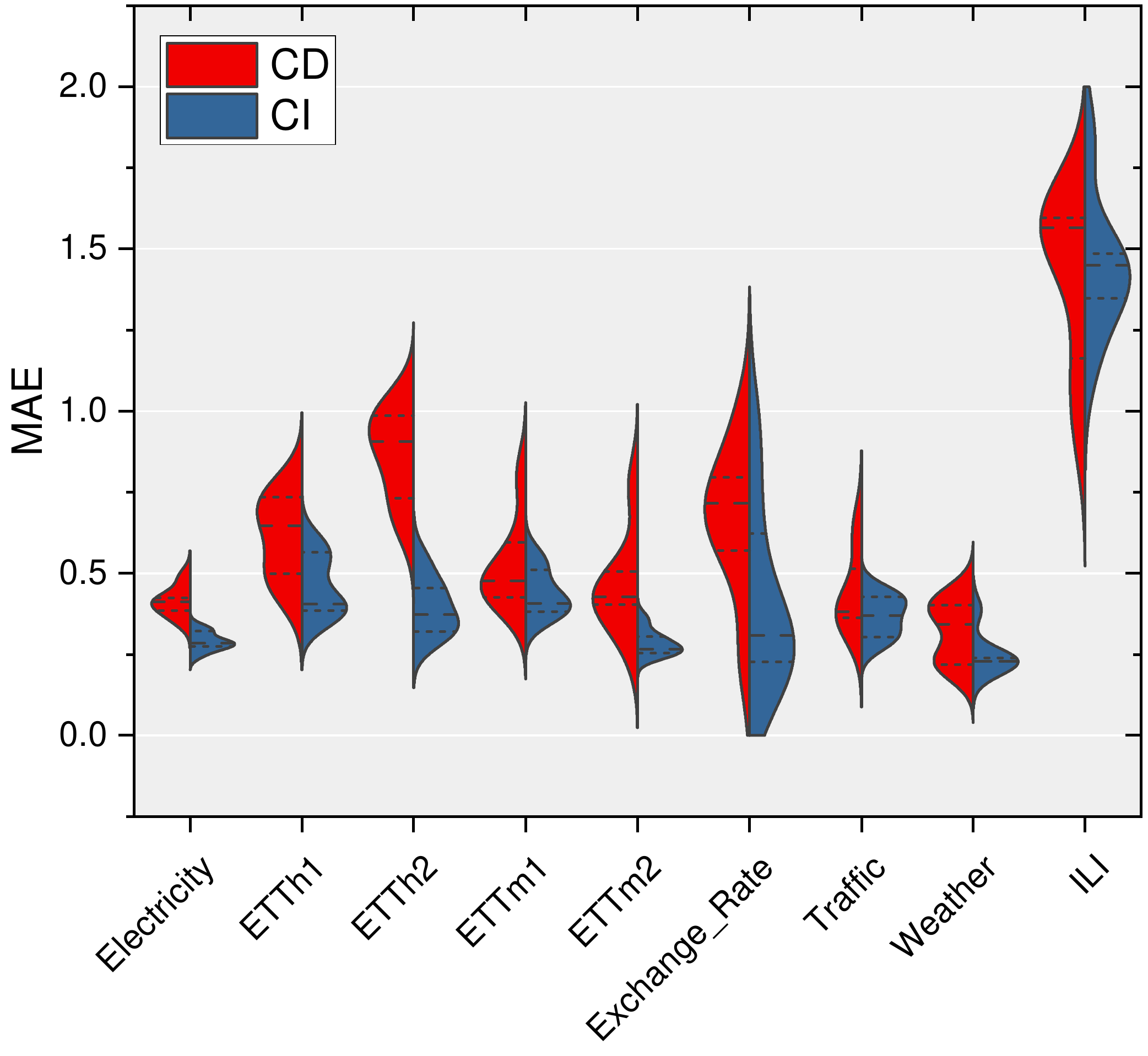}
    \includegraphics[width=0.49\linewidth]{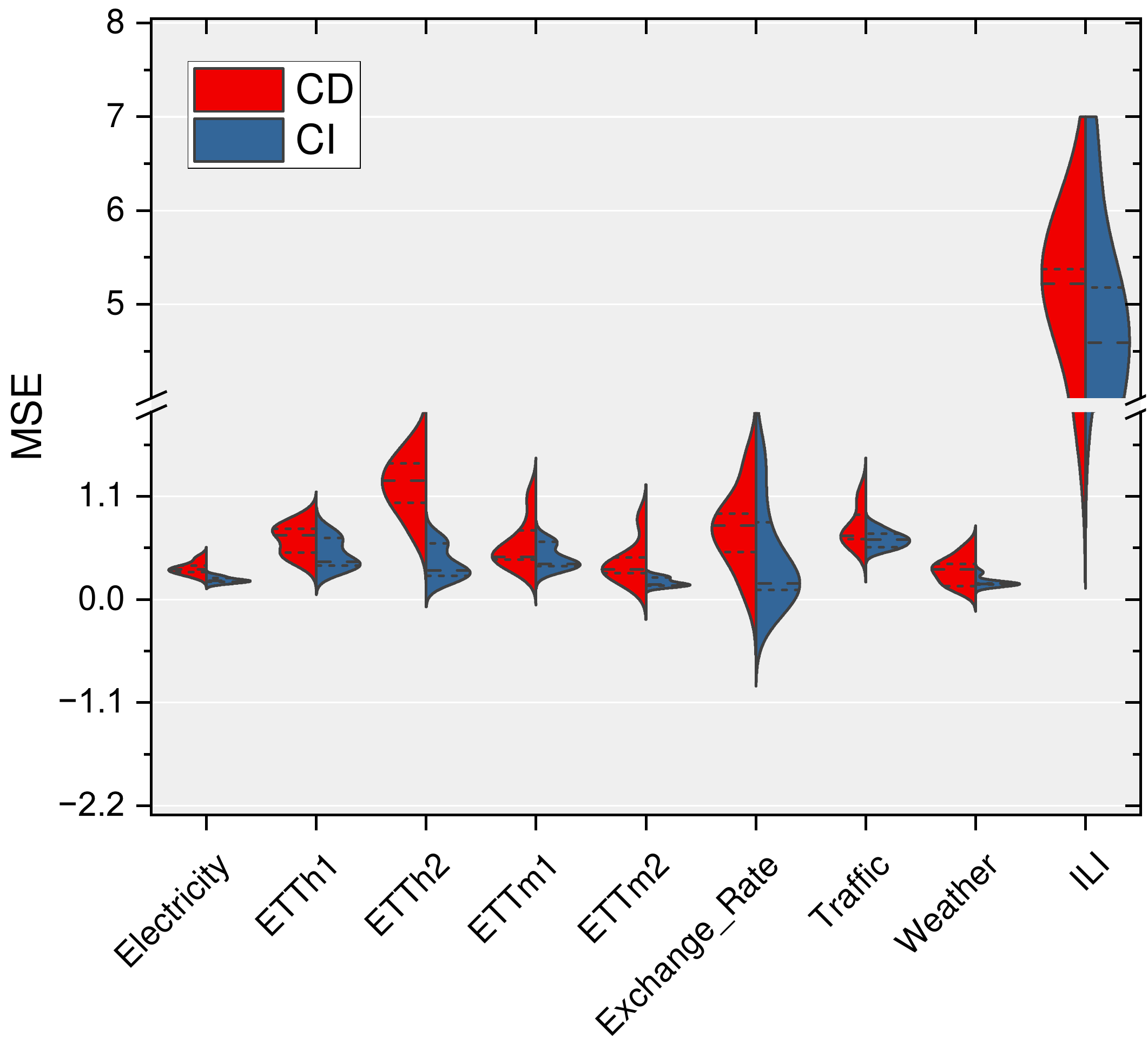}
    \vspace{-4mm}
    \caption{The performance distribution of 7 models utilizing the CI and CD strategy. Values come from \cref{tb:mae_analysis} and \cref{tb:mse_analysis}. The prediction length is 24 for ILI dataset and 48 for the others. In most cases, CI has a lower error mean and a smaller variance than CD strategy. It means that CI performs better than CD. Also, when using CI strategy, the model performance does not differ very much.}
    \label{fig:perform_dist}
\end{figure}

\section{Analysis}
This section aims to provide an in-depth analysis of why CI is superior for multivariate forecasting tasks in most cases, using the Linear~\cite{Zeng2022Transformers} model as an example. It is closely related to the AutoRegression (AR) in statistics~\cite{box2015time}. We begin by examining the presence of distribution shift in real-world datasets. Subsequently, we evaluate the Linear model with CI and CD strategies, demonstrating how the drifted statistics impact its performance. Our analysis highlights the fact that CI reduces the statistics gap between the training and test data. Finally, we decompose the risk to demonstrate that CI trades capacity for robustness, which translates to improved performance on many real-world non-stationary time series.

\subsection{Distribution Drift}
Real-world datasets are characterized by time series with changing values over time, often accompanied by changes in the underlying distribution, referred to as non-stationarity~\cite{anderson1976time,hyndman2018forecasting}. In this section, we investigate the AutoCorrelation Function (ACF), which is commonly used in time series analysis:
\begin{definition}[AutoCorrelation Function (ACF)~\cite{madsen2007time}] The autocorrelation function of a stochastic process, $\{X(t)\}$, is defined as: 
	$$
	\rho\left(t_1, t_2\right)=\frac{\gamma \left(t_1, t_2\right)}{\sqrt{\sigma^2\left(t_1\right) \sigma^2\left(t_2\right)}}
	$$ 
	where $\gamma\left(t_1, t_2\right)=\operatorname{Cov}\left[X(t_1), X(t_2)\right]$ is the covariance function, and $\sigma^2(t)=\gamma (t, t)$ is the variance at time $t$.
	
	If the process is stationary, then ACF is only a function of the time difference $\tau = t_2 - t_1$, \ie:
	\begin{equation}
		\rho(\tau)=\frac{\gamma(\tau)}{\gamma(0)},
	\end{equation}
	where $\gamma(\tau) = \operatorname{Cov}\left[X(t), X(t+\tau)\right]$.
\end{definition}

In this paper, we temporarily assume that the stochastic process of each channel is stationary. But we will show that our analysis results still holds in real-world data. To estimate the AutoCorrelation Function (ACF) of a given time series $\x \in \R^T$, we employ the method from~\cite{jenkins1957spectral}, which is a commonly used practice. Specifically, we use the following equation to estimate the ACF:
\begin{equation}
	\hat{\rho}(\tau)=\frac{\hat{\gamma}(\tau)}{\hat{\gamma}(0)} 
\end{equation}
where $\hat{\gamma}(\tau)=\frac{1}{T-k} \sum_{t=1}^{T-\tau}\left(\x_t-\bar{\x}\right)\left(\x_{t+\tau}-\bar{\x}\right)$ is the estimated covariance function, $\bar{x} = \frac{1}{T}\sum_{i=1}^T \x_t$ is the mean value of $\x$.

We display the ACF of the train series and test series on each dataset in \cref{fig:dis_drift}, following the train-test split used in previous works~\cite{Zhou2021informer,Wu2021Autoformer,FedFormer}. From \cref{fig:dis_drift}, it is evident that distribution drift is present in each dataset, owing to various reasons.  For instance, in (MCIL, ETTh2) and (146, Electricity), the anomaly in the training series leads to distribution drift. In (OT, ILI) and (wv (m/s), Weather), variation in the trend is the main cause. The rest of two figures can not be concluded by simple reasons. Nevertheless, distribution drift is a prevalent phenomenon in real-world time series datasets.

\begin{figure*}[htp]
	\centering
	\begin{subfigure}[b]{0.32\textwidth}
		\centering
		\includegraphics[width=\textwidth]{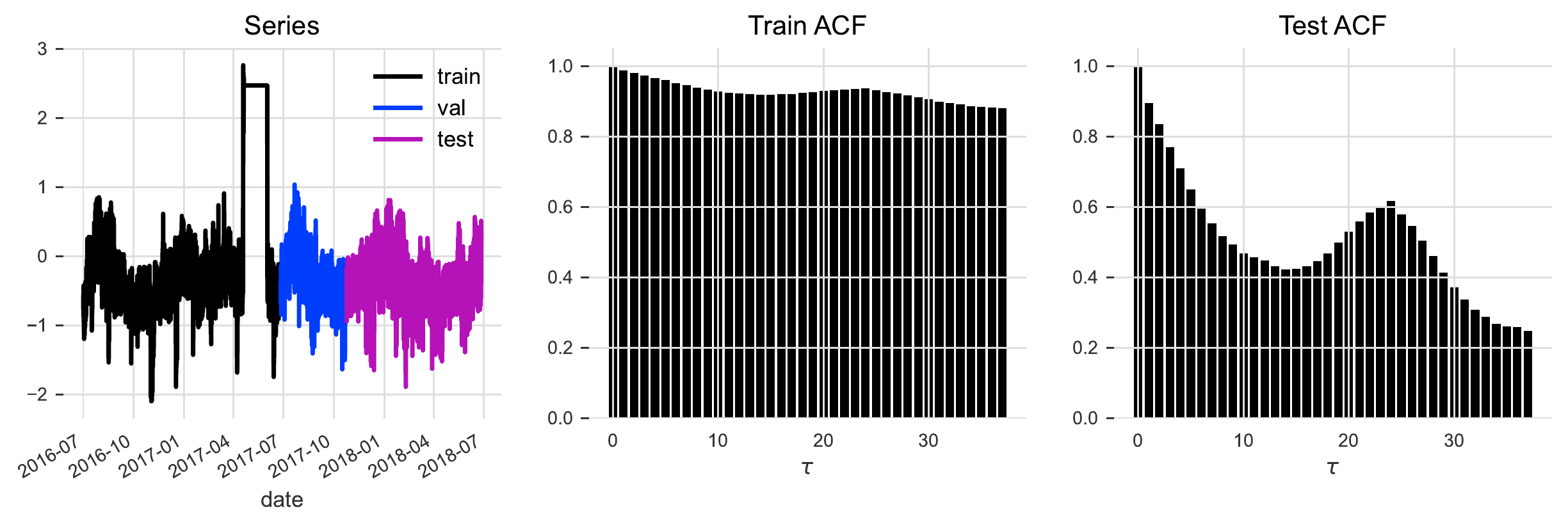}
		\caption{MUFL, ETTh2}
	\end{subfigure}
	\hspace{2mm}
	\begin{subfigure}[b]{0.32\textwidth}
		\centering
		\includegraphics[width=\textwidth]{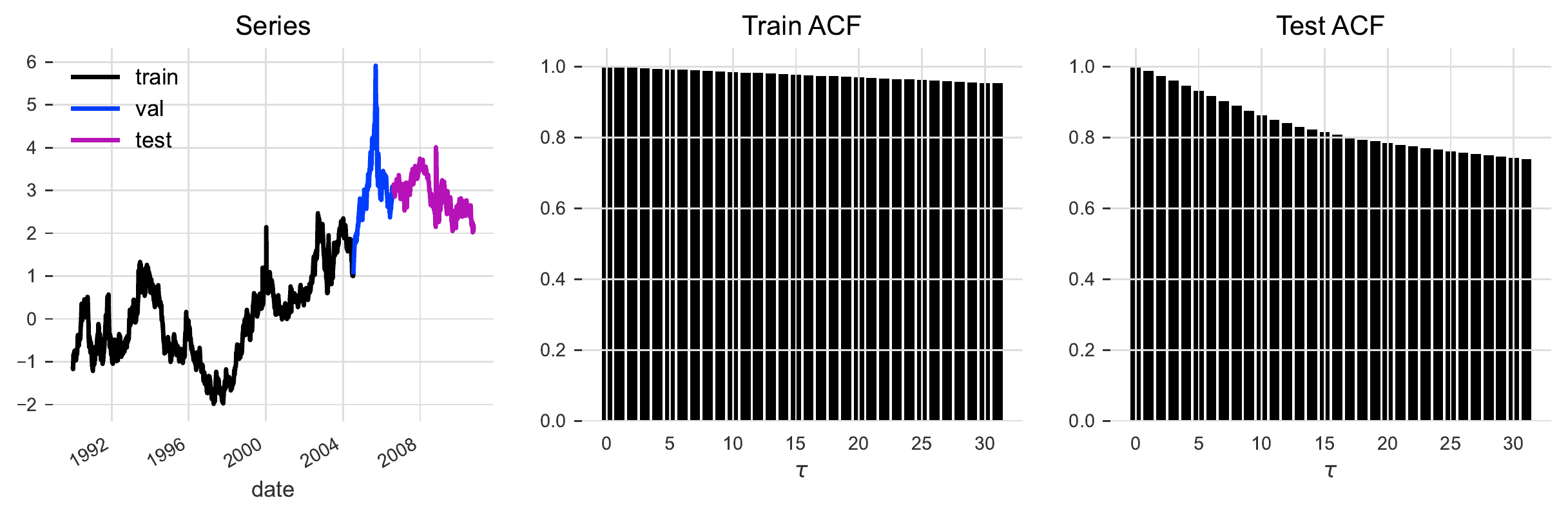}
		\caption{3, Exchange-Rate}
	\end{subfigure}
	\hspace{2mm}
	\begin{subfigure}[b]{0.32\textwidth}
		\centering
		\includegraphics[width=\textwidth]{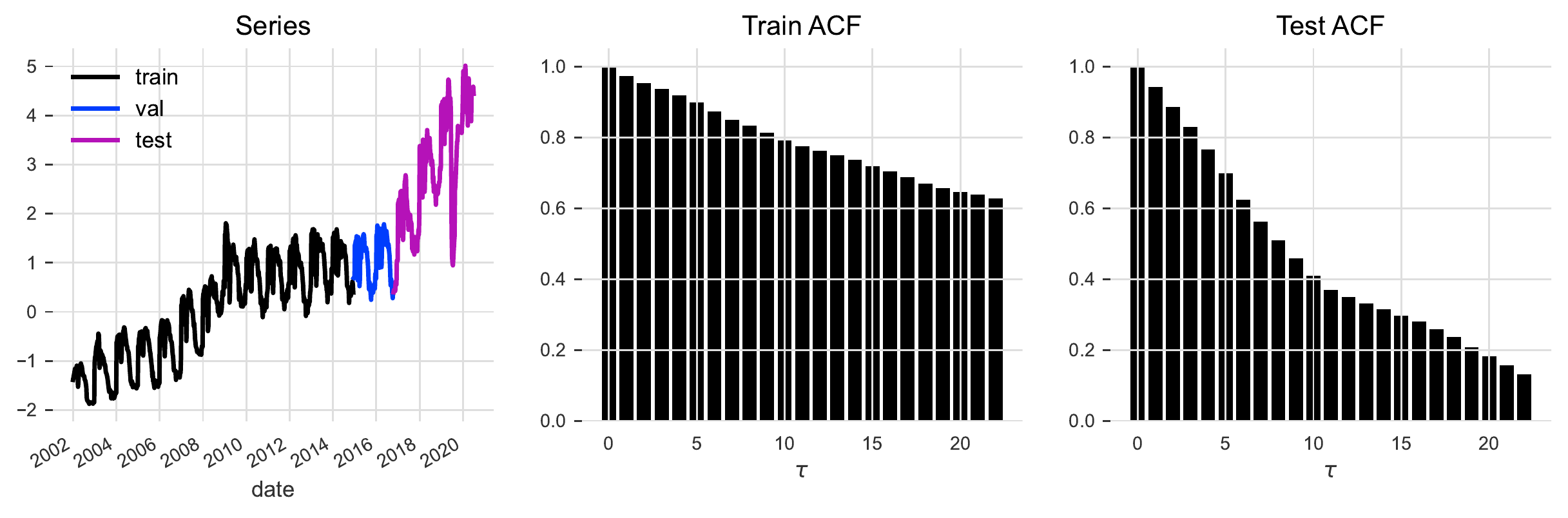}
		\caption{OT, ILI}
	\end{subfigure}
	
	\begin{subfigure}[b]{0.32\textwidth}
		\centering
		\includegraphics[width=\textwidth]{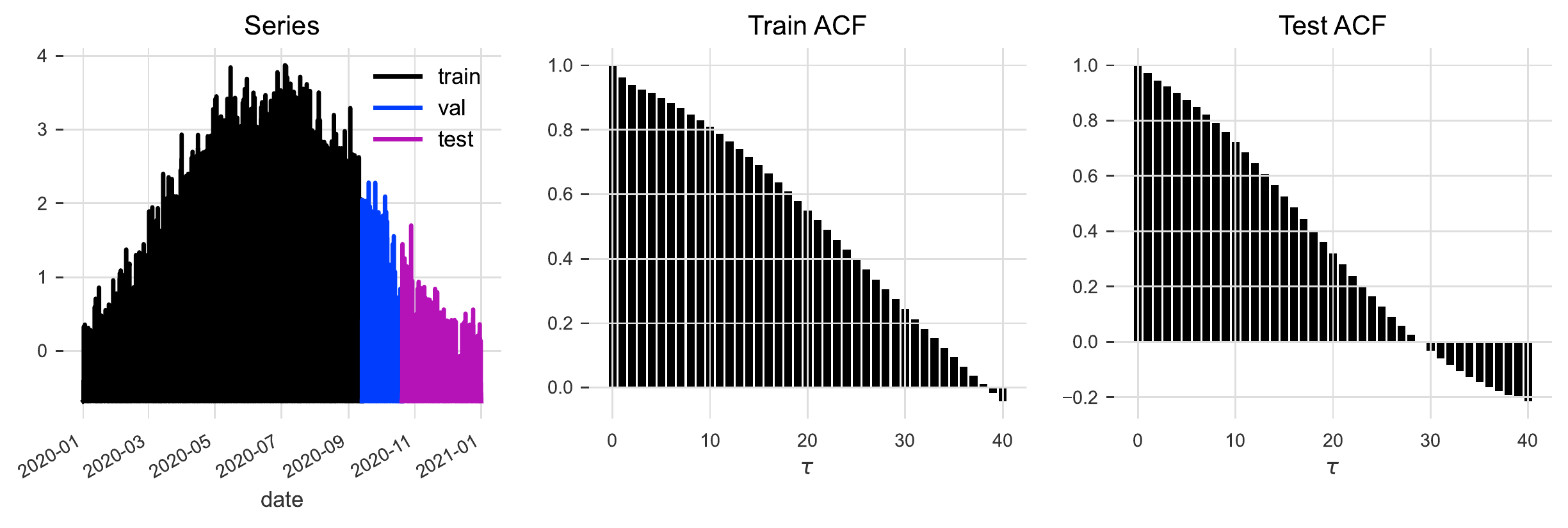}
		\caption{wv (m/s), Weather}
	\end{subfigure}
	\hspace{2mm}
	\begin{subfigure}[b]{0.32\textwidth}
		\centering
		\includegraphics[width=\textwidth]{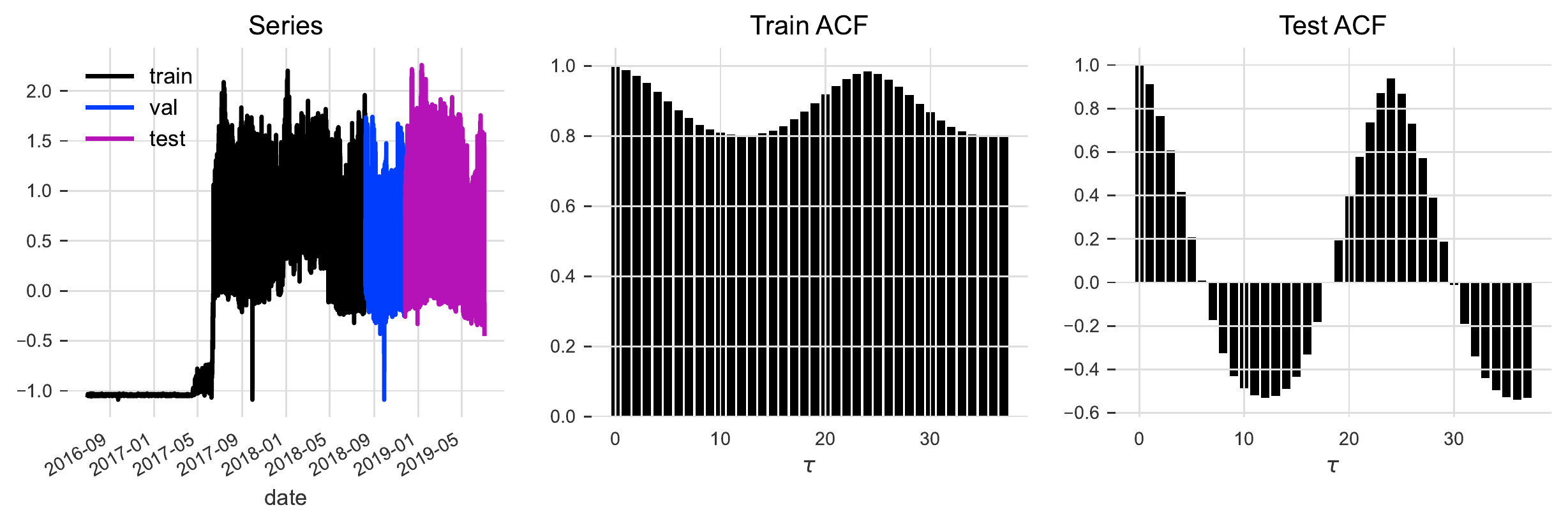}
		\caption{146, Electricity}
	\end{subfigure}
	\hspace{2mm}
	\begin{subfigure}[b]{0.32\textwidth}
		\centering
		\includegraphics[width=\textwidth]{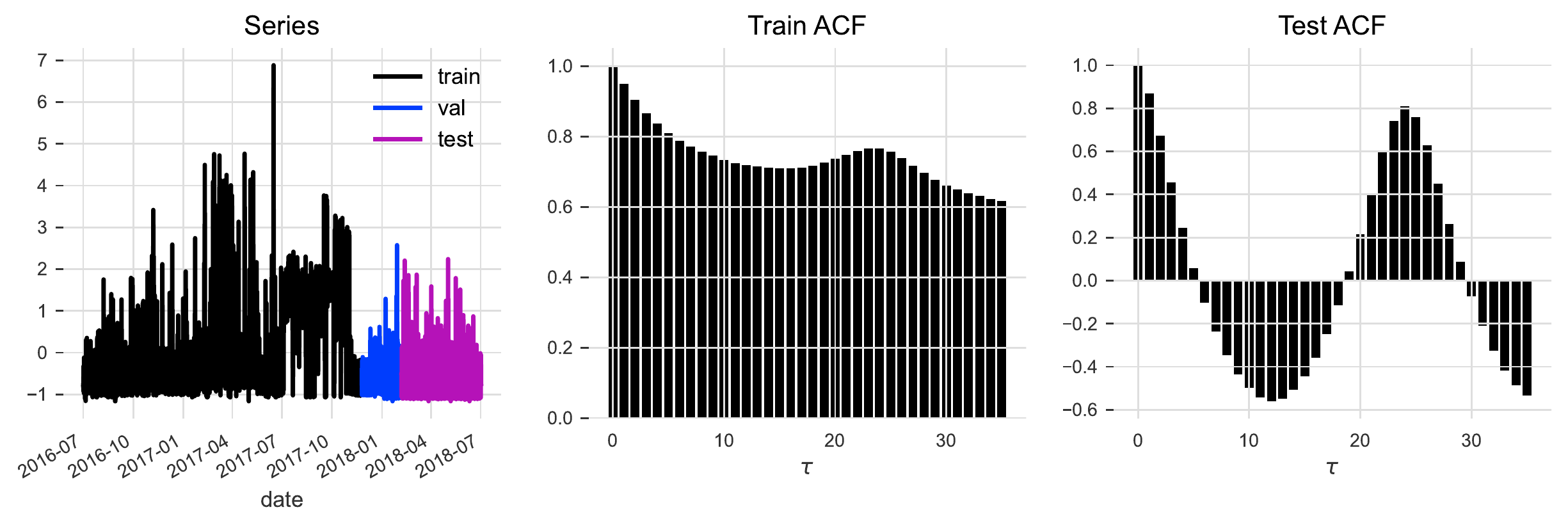}
		\caption{607, Traffic}
	\end{subfigure}
	\vspace{-5mm}
	\caption{The ACF of train series and test series. Captions of each subfigure represent the tuple (channel, dataset). For each subfigure, the leftmost plot displays the series split, with the training series in black, validation in blue, and test in purple. The middle and right display the ACF of train and test series respectively. The middle and right plots show the ACF of the training and test series, respectively. The results reveal a significant discrepancy in the statistics between the training and test series.}
	\label{fig:dis_drift}
\end{figure*}

The observed distribution drift has a profound impact on the performance of machine learning models, as the fundamental assumption of these models is that the training and test data are drawn from identical and independent distributions (i.i.d.)~\cite{mohri2018foundations}. This discrepancy undermines the accuracy of these models in predicting unseen data. For instance, we extend the autoregressive (AR) model to long-term forecasting tasks and demonstrate the adverse effects of distribution drift on the model's performance: 
\begin{proposition}[Yule-Walker equation~\cite{udny1927method,walker1931periodicity} extended]
	Assuming a long-term AR model on time series $\x$ with look-back window (order) $L$ and horizon $H$ is defined as:
	\begin{equation}
		(\x_{t+H-1},\ldots, \x_{t})^T = \W (\x_{t-1},\ldots, \x_{t-L})^\top
	\end{equation} 
	where $\W \in \R^{H\times L}$ is the coefficients of the model. Then the best estimation $\W^*$ can be computed by extended version of Yule-Walker equation~\cite{udny1927method,walker1931periodicity}:
	\begin{equation}
	\resizebox{0.95\linewidth}{!}{
			$\left[\begin{array}{cccc}
				\rho(1)& \rho(2) & \cdots & \rho(H) \\
				\rho(2) & \rho(3) & \cdots & \rho(H+1)\\
				\vdots &\vdots&\ddots&\vdots\\
				\rho(L) & \rho(L+1) & \cdots & \rho(H+L-1)
			\end{array}\right]=  \left[\begin{array}{cccc}
				\rho(0) & \rho(-1)  & \cdots & \rho(-L+1)\\
				\rho(1) & \rho(0)  & \cdots& \rho(-L+2) \\
				\vdots & \vdots& \ddots & \vdots  \\
				\rho(L-1) & \rho(L-2)   & \cdots& \rho(0)
			\end{array}\right]\W^*$
}
	\end{equation}
	
where $\rho(\tau) = \rho(-\tau)$ is the autocorrelation of time delay $\tau$.
\label{prop:yule_walker}
\end{proposition}
Proposition \ref{prop:yule_walker} establishes that the performance of an autoregressive (AR) model is closely linked to the autocorrelation function (ACF). Specifically, when there is a significant disparity between the ACF of the training and testing data, the resulting difference in the estimated values of $\W^*$ on the two datasets can be substantial. This can result in high error rates when applying the trained model to the test data.

Regrettably, real-world multivariate datasets often exhibit large disparities in ACF across channels. These distribution drifts in certain channels can significantly impact the performance of the trained model. However, we demonstrate in the next section that while the AR model using the CD strategy is sensitive to such drifts, the model using the CI strategy is more robust to them.

\subsection{CI Alleviates Distribution Drift}
In the previous section, we highlighted the presence of distribution drift (as measured by ACF) in real-world datasets. We also presented theoretical insights into how such drift can impact the performance of linear autoregressive (AR) models in univariate scenarios. n this section, we extend our analysis to multivariate tasks and demonstrate that the \emph{CI strategy can alleviate distribution drift in each channel, while the CD strategy is vulnerable to such drift.} 

\noindent{\textbf{Coefficients of CI and CD.}}
To facilitate our analysis, we reshape the set of series data. Specifically, given a set of data $\{(\X^{(i)}, \Y^{(i)})\}_{i=1}^N$, we rearrange the series of each channel to its unique matrix $\A^{(c)}\in \R^{N\times L}$ and $\B^{(c)}\in \R^{N\times H}$, \ie, $\A^{(c)}_{i,l} = \X^{(i)}_{l,c}$, $\B^{(c)}_{i,h} = \Y^{(i)}_{h,c}$. Basic on this representation, we can express the objectives of Linear (CD) and Linear (CI) as follows:

\begin{definition}[Objective of Linear (CD) and Linear (CI)] Assuming the series of each channel is centered, the ordinary least square objective of \textbf{Linear (CD)} can be defined as:
	\begin{equation}
		\mathcal{L}_{cd} =\lVert \A_{cd} \W_{cd} - \B_{cd} \rVert^2_F
		\label{eqn:obj_cd}
	\end{equation}

where $\A_{cd} = \left[\A^{(1)}  \A^{(2)} \ldots \A^{(C)}\right] \in \R^{N\times LC}$ is the vertical concatenation of $\A^{(1)}, \A^{(2)}, \ldots, \A^{(C)}$, $\B_{cd}$ the same. $\W_{cd} \in \R^{LC \times HC}$ is the coefficient.

Simiarly, the objective of \textbf{Linear (CI)} can be defined as:
\begin{equation}
			\mathcal{L}_{ci} =\lVert \A_{ci} \W_{ci} - \B_{ci} \rVert^2_F
			\label{eqn:obj_ci}
\end{equation}
where $\A_{ci} = \begin{bmatrix}
	 \A^{(1)} \\ \A^{(2)} \\ \vdots \\ \A^{(C)}
\end{bmatrix} \in \R^{NC\times L}$ is the horizontal concatenation of $\A^{(1)}, \A^{(2)}, \ldots, \A^{(C)}$, $\B_{ci}$ the same. $\W_{cd} \in \R^{L \times H}$ is the coefficient.
\label{def:linear_cd_ci}
\end{definition}
From \cref{def:linear_cd_ci}, we can see that the primary distinction between CD and CI strategy on Linear model is the way data are stacked. By solving the two objectives, the CD and CI coefficient of can be estimated according to the following proposition:

\begin{proposition}[Yule-Walker equation of Linear (CD) and Linear (CI)] Define the (auto-/cross-)correlation matrix: $$\boldsymbol{R}_{c_1,c_2} = \begin{bmatrix}
						\rho_{c_1,c_2}(0) & \rho_{c_1,c_2}(-1)  & \cdots & \rho_{c_1,c_2}(-L+1)\\
		\rho_{c_1,c_2}(1) & \rho_{c_1,c_2}(0)  & \cdots& \rho_{c_1,c_2}(-L+2) \\
		\vdots & \vdots& \ddots & \vdots  \\
		\rho_{c_1,c_2}(L-1) & \rho_{c_1,c_2}(L-2)   & \cdots& \rho_{c_1,c_2}(0)
	\end{bmatrix}\in \R^{L\times L}. $$
$$
\boldsymbol{R}'_{c_1,c_2}= \begin{bmatrix}
\rho(1)& \rho(2) & \cdots & \rho(H) \\
\rho(2) & \rho(3) & \cdots & \rho(H+1)\\
\vdots &\vdots&\ddots&\vdots\\
\rho(L) & \rho(L+1) & \cdots & \rho(H+L-1)
\end{bmatrix}\in \R^{L\times H}
$$
where $\rho_{c_1,c_2}(\tau)$ is the auto-/cross-correlation at time delay $\tau$ when $c_1 = c_2$ / $c_1\neq c_2$. 

Assuming the series of each channel has the same variance, then the Yule-Walker equation of Linear (CD) is:
\begin{equation}
		\begin{bmatrix}
		R'_{1,1} & R'_{1,2} & \ldots & R'_{1,C}\\
		R'_{2,1} & R'_{2,2} & \ldots & R'_{2,C} \\
		\vdots &\vdots &\ddots&\vdots \\
		R'_{C,1} & R'_{C,2} & \ldots & R'_{C,C}
	\end{bmatrix}
	= 
	\begin{bmatrix}
		R_{1,1} & R_{1,2} & \ldots & R_{1,C}\\
		R_{2,1} & R_{2,2} & \ldots & R_{2,C} \\
		\vdots &\vdots &\ddots&\vdots \\
		R_{C,1} & R_{C,2} & \ldots & R_{C,C}
	\end{bmatrix}\W^*_{cd}
\label{eq:yw_cd}
\end{equation}
and the Yule-Walker equation of Linear (CI) is:
\begin{equation}
	\sum_{c=1}^C R'_{c,c}
	= (\sum_{c=1}^C R_{c,c}) \W^*_{ci}
	\label{eq:yw_ci}
\end{equation}
\end{proposition}

\begin{proof}
	Taking the derivative of \cref{eqn:obj_cd}, we get the ordinary least square equation: 
	\begin{equation}
			\A^\top_{cd}\A_{cd}\W_{cd} = \A^\top_{cd}\B_{cd}
			\label{eq:ols}
	\end{equation}

	\begin{equation}
		\begin{aligned}
					&\A^\top_{cd}\A_{cd} =\left[\A^{(1)}  \A^{(2)} \ldots \A^{(C)}\right]^\top \left[\A^{(1)}  \A^{(2)} \ldots \A^{(C)}\right]\\ &= 	\begin{bmatrix}
				(\A^{(1)})^\top (\A^{(1)}) & (\A^{(1)})^\top (\A^{(2)}) & \ldots & (\A^{(1)})^\top (\A^{(C)})\\
				(\A^{(2)})^\top (\A^{(2)}) & (\A^{(2)})^\top (\A^{(2)}) & \ldots & (\A^{(2)})^\top (\A^{(C)})\\
				\vdots &\vdots &\ddots&\vdots \\
				(\A^{(C)})^\top (\A^{(C)}) & (\A^{(C)})^\top (\A^{(C)}) & \ldots & (\A^{(1)})^\top (\A^{(C)})
			\end{bmatrix}
		\end{aligned}
	\end{equation}
	is in a form of outer product.
	Each $(\A^{(c_1)})^\top (\A^{(c_2)})$ is a estimation of co-variance matrix~\cite{madsen2007time}. Since the variances of each channel series are assumed to be the same. So by dividing on both side of \cref{eq:ols} by the variance, we can get \cref{eq:yw_cd}. 
	
	Simiarly, $\A^\top_{cd}\A_{cd} = \left[\A^{(1)}  \A^{(2)} \ldots \A^{(C)}\right] \left[\A^{(1)}  \A^{(2)} \ldots \A^{(C)}\right]^\top$ is in the form of inner product. By the same process, we can get \cref{eq:yw_ci}.
\end{proof}
By analyzing the difference between~\cref{eq:yw_cd} and~\cref{eq:yw_ci}, we can draw an important conclusion-- \textbf{coefficients of CD is determined by the (auto-/cross-)correlation function of each channel, while the coefficients of the CI strategy are determined solely by the summation (or mean) of the ACF of all channels.} 
\begin{tcolorbox}[title={Takeaways}]
The optimal coefficients of the Linear model using the CD strategy are determined by the ACF of all channels, while the optimal coefficients of the model using the CI strategy are only determined by the \textbf{sum} of the ACF across all channels.
\end{tcolorbox}

\noindent{\textbf{CI strategy leads to less distribution drift.}} It is noteworthy that the summation operation used in the CI strategy mitigates the distribution gap between the training and test series. To demonstrate this, we examine the differences in the values of the ACF between the training and test portions. Specifically, we denote the ACF of the training portion in channel $c$ as $\rho_c^{(tr)}$ and the corresponding test ACF as $\rho_c^{(te)}$. Then, we calculate the ACF difference in channel $c$ as follows:
$$
\operatorname{Diff}_{c} = \sum_{t=0}^{T}(\rho_c^{(tr)}(t) - \rho_c^{(te)}(t))^2.
$$
When employing the CD strategy, the linear model is susceptible to the distribution drift of each channel. However, the CI strategy ensures that the linear model is solely determined by the sum of ACF over all channels. Therefore, we only need to evaluate the changes in the sum of ACF when using the CI strategy. Hence, we calculate the difference in the ACF summation between the training and test sets using the following equation, referred to as the \textbf{sum diff}:
$$
\operatorname{Diff}_{\text{sum}} = \sum_{t=0}^{T}(\frac{1}{C}\sum_{c=1}^{C}  \rho_c^{(tr)}(t) - \frac{1}{C}\sum_{c=1}^{C}\rho_c^{(te)}(t))^2.
$$
Considering the scale, we compute the mean instead of the sum. But we still name it sum diff. Sum diff can be regarded as the ACF difference when using CI strategy. 

We present the ACF difference ($\operatorname{Diff}{c}$) using bar plots, sorted in descending order, and indicate the \textbf{sum diff} ($\operatorname{Diff}{\text{sum}}$) with a horizontal line. Our results are summarized in \cref{eq:acf_diff}.Several observations can be made from the figure: (1) \textbf{Most real-world datasets exhibit channels with significant ACF differences between the training and test data, indicating severe distribution drift in the time series of these channels.}  In ETT datasets (a)-(d), for instance, the largest ACF differences range between 5 and 8, which is substantial given that ACF values typically fall within $[0,1]$. ETTh1 and ETTm1 show relatively uniform ACF differences, while ETTh2 and ETTm2 feature two channels with particularly large ACF differences compared to the rest. Other datasets exhibit similar patterns, with Exchange (e) displaying two channels with ACF differences that greatly exceed those of other channels, and ILI featuring a largest difference that is more than twice that of the rest. In datasets with many channels, such as weather (g), electricity (h), and traffic (i), the largest difference can be up to 15, 32, and 24, respectively, which is much greater than that of most other channels. Furthermore, we observe a rapid decay of the ACF difference as the channel index increases. (2) \textbf{The sum diff is typically smaller than the ACF difference of most channels, suggesting that the distribution drift with CI strategy is less severe than with CD strategy.}  Across the 9 benchmarks, 7 datasets have a sum diff that is smaller than that of more than 50\% of the channels, indicating that the distribution drift, as measured by the ACF difference, is smaller than that of most channels when using CI. Even in the two exceptions, ETTm2 (d) and Exchange-Rate (e), the sum diff is still much smaller than the head channels. On ETTm2 (d), for example, the sum diff is 0.4678, significantly smaller than the head values, which are nearly 5. The sum diff is also not much larger than channels 4-7. Similar observations hold for Exchange-Rate (e), where the sum diff is only lower than two channels, but its value of 0.1044 is much smaller than 0.75. On the remaining 7 datasets, not only is the sum diff smaller than that of most channels, but it is also much smaller in value. For instance, on ETTh2 (b), the sum diff of 0.3713 is 20 times smaller than that of channel 1, while on Electricity (h) and Traffic (i), it is 500 and 8000 times smaller, respectively. 

\begin{figure*}[h]
	\begin{subfigure}[b]{0.19\linewidth}
		\includegraphics[width=\linewidth]{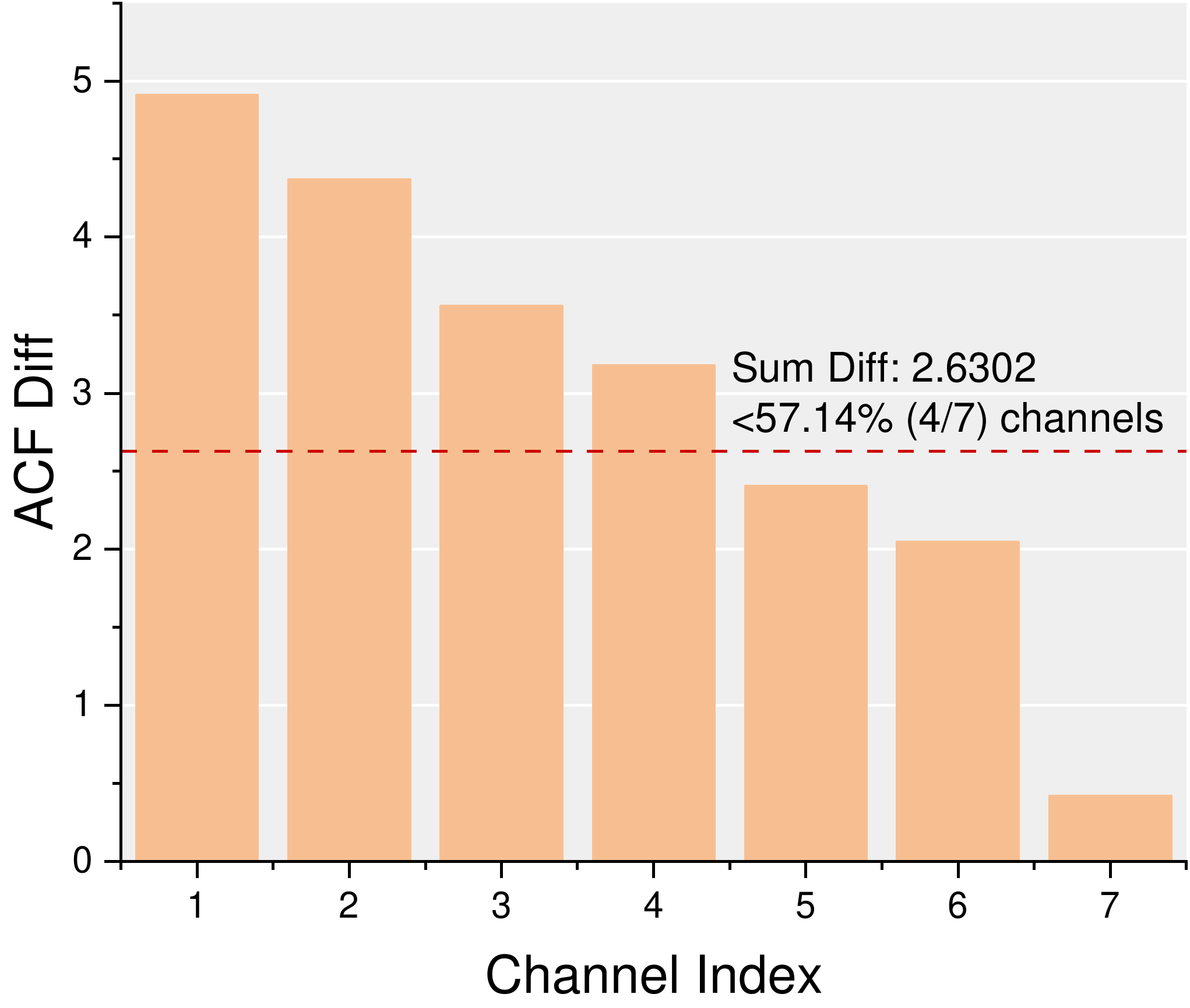}
		\caption{ETTh1}
	\end{subfigure}
	\begin{subfigure}[b]{0.19\linewidth}
	\includegraphics[width=\linewidth]{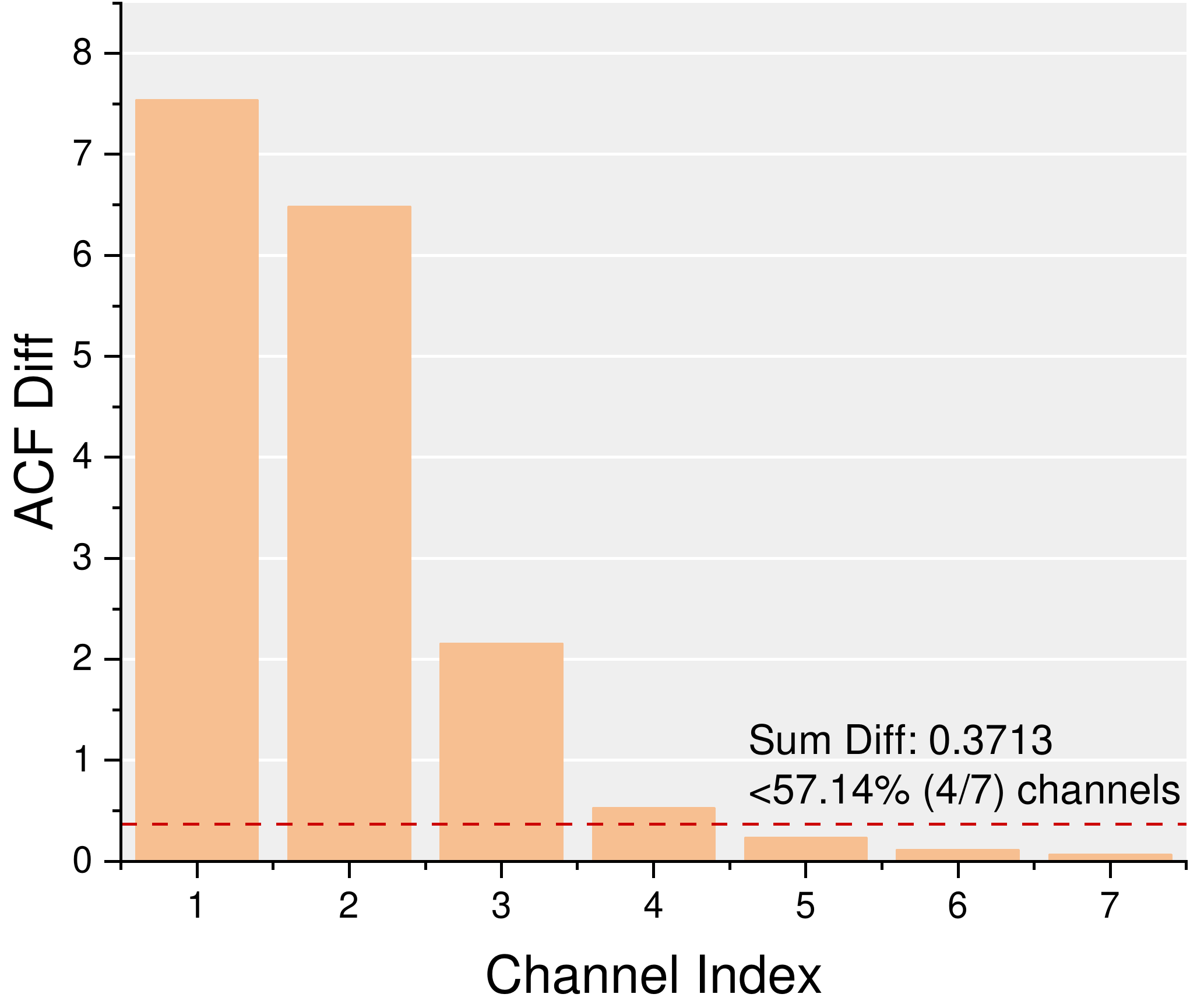}
	\caption{ETTh2}
\end{subfigure}
	\begin{subfigure}[b]{0.19\linewidth}
	\includegraphics[width=\linewidth]{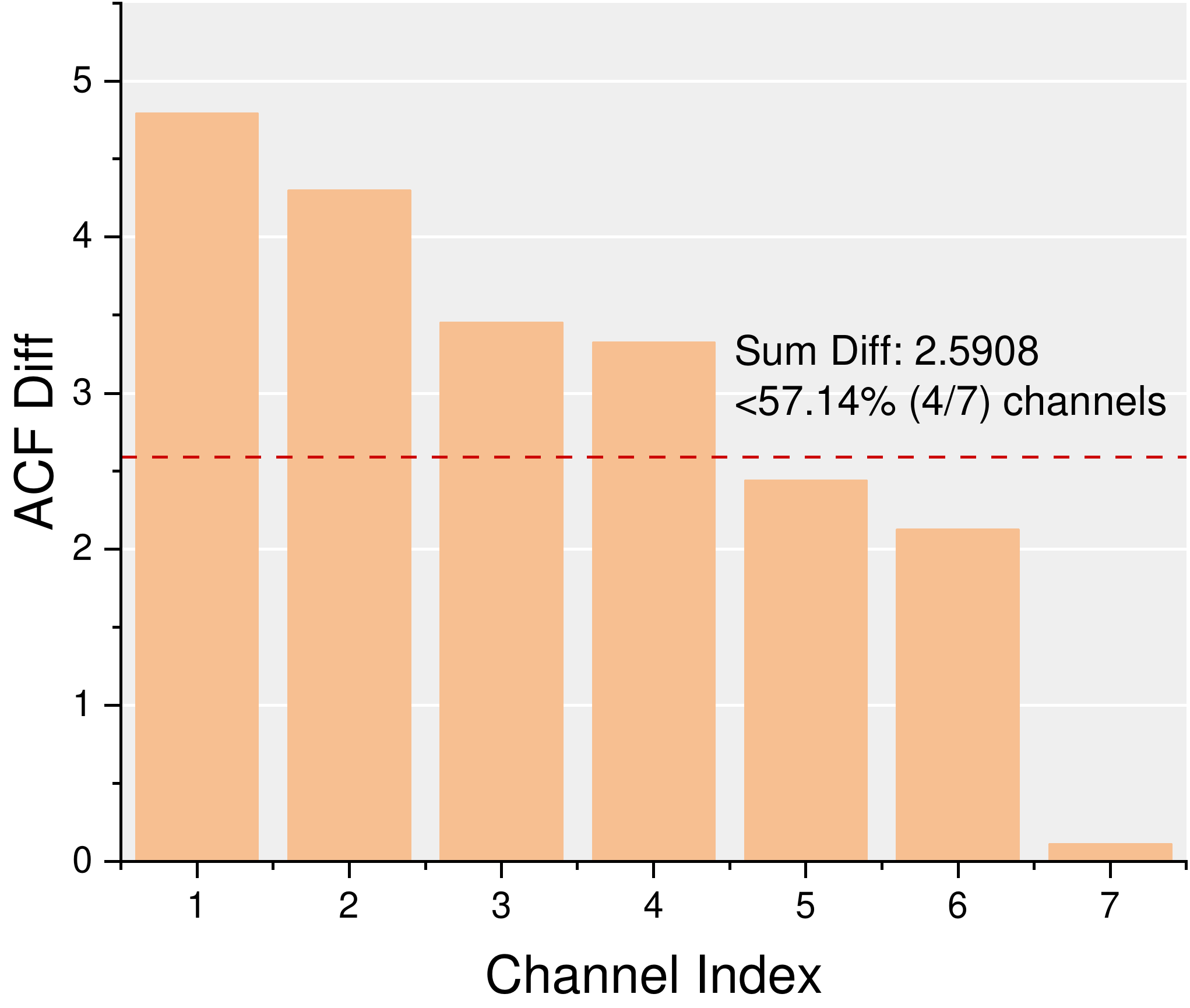}
	\caption{ETTm1}
\end{subfigure}
	\begin{subfigure}[b]{0.19\linewidth}
	\includegraphics[width=\linewidth]{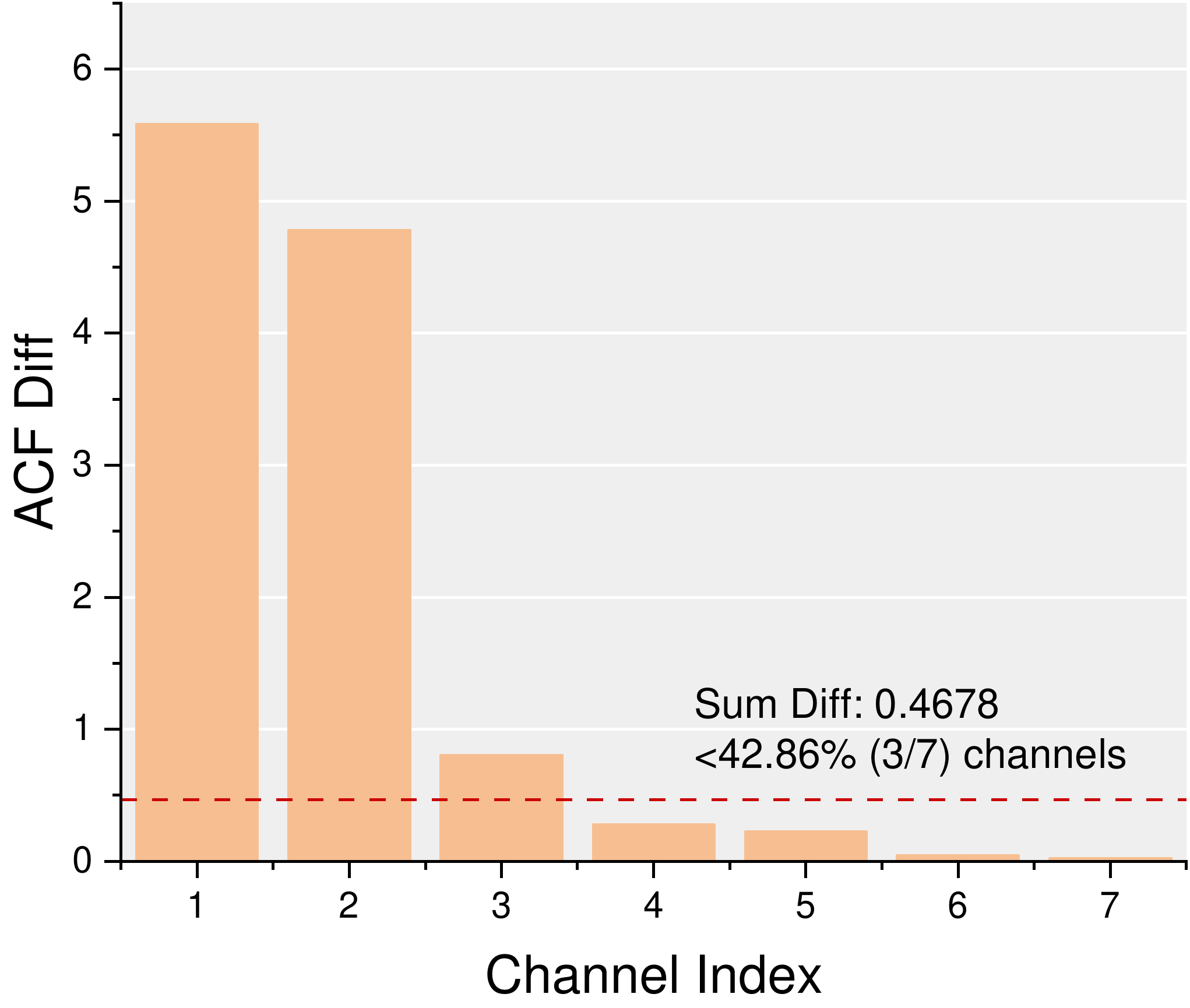}
	\caption{ETTm2}
\end{subfigure}
	\begin{subfigure}[b]{0.19\linewidth}
	\includegraphics[width=\linewidth]{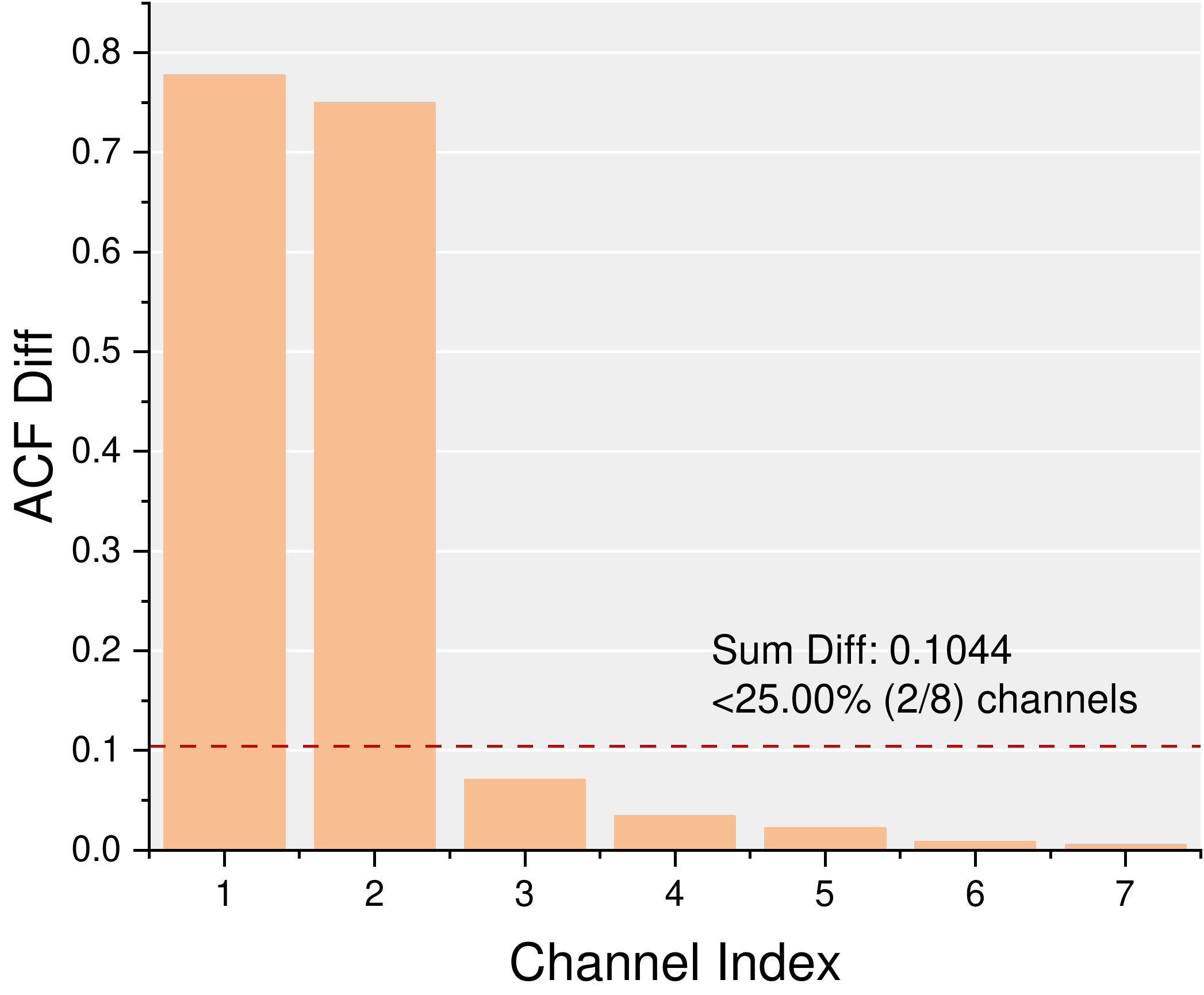}
	\caption{Exchange-Rate}
\end{subfigure}
	\begin{subfigure}[b]{0.23\linewidth}
	\includegraphics[width=\linewidth]{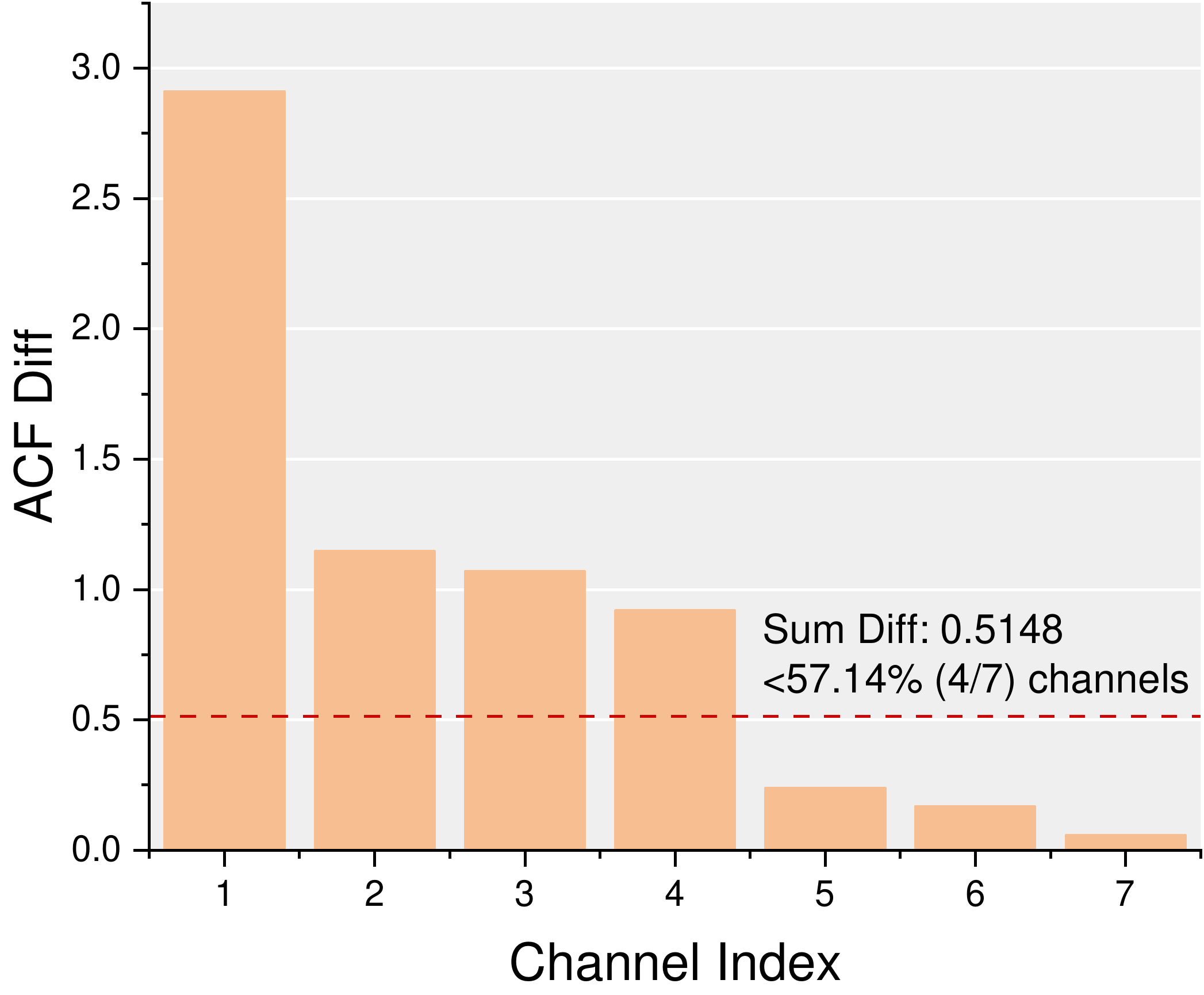}
	\caption{ILI}
\end{subfigure}
	\begin{subfigure}[b]{0.23\linewidth}
	\includegraphics[width=\linewidth]{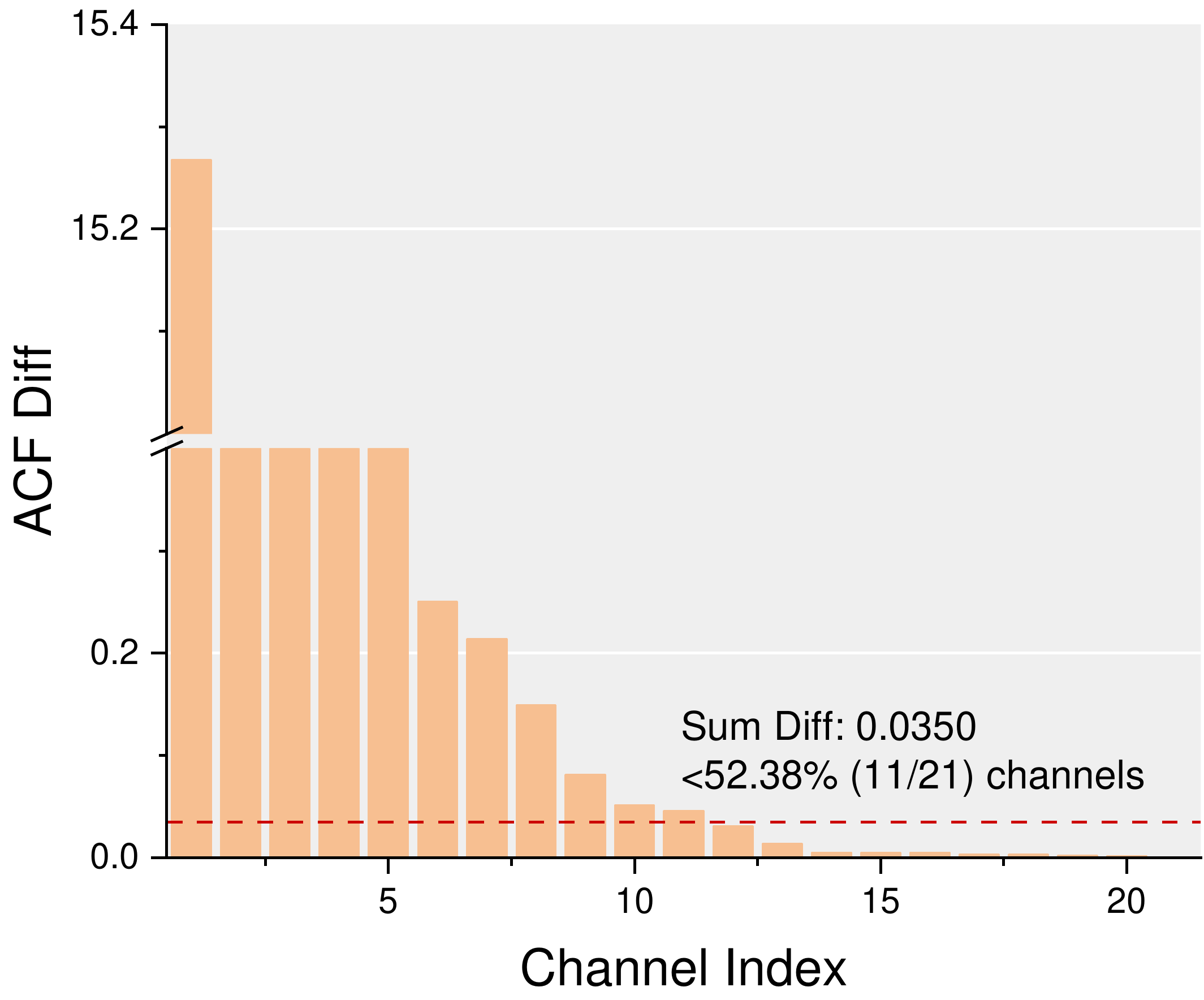}
	\caption{Weather}
\end{subfigure}
\begin{subfigure}[b]{0.23\linewidth}
\includegraphics[width=\linewidth]{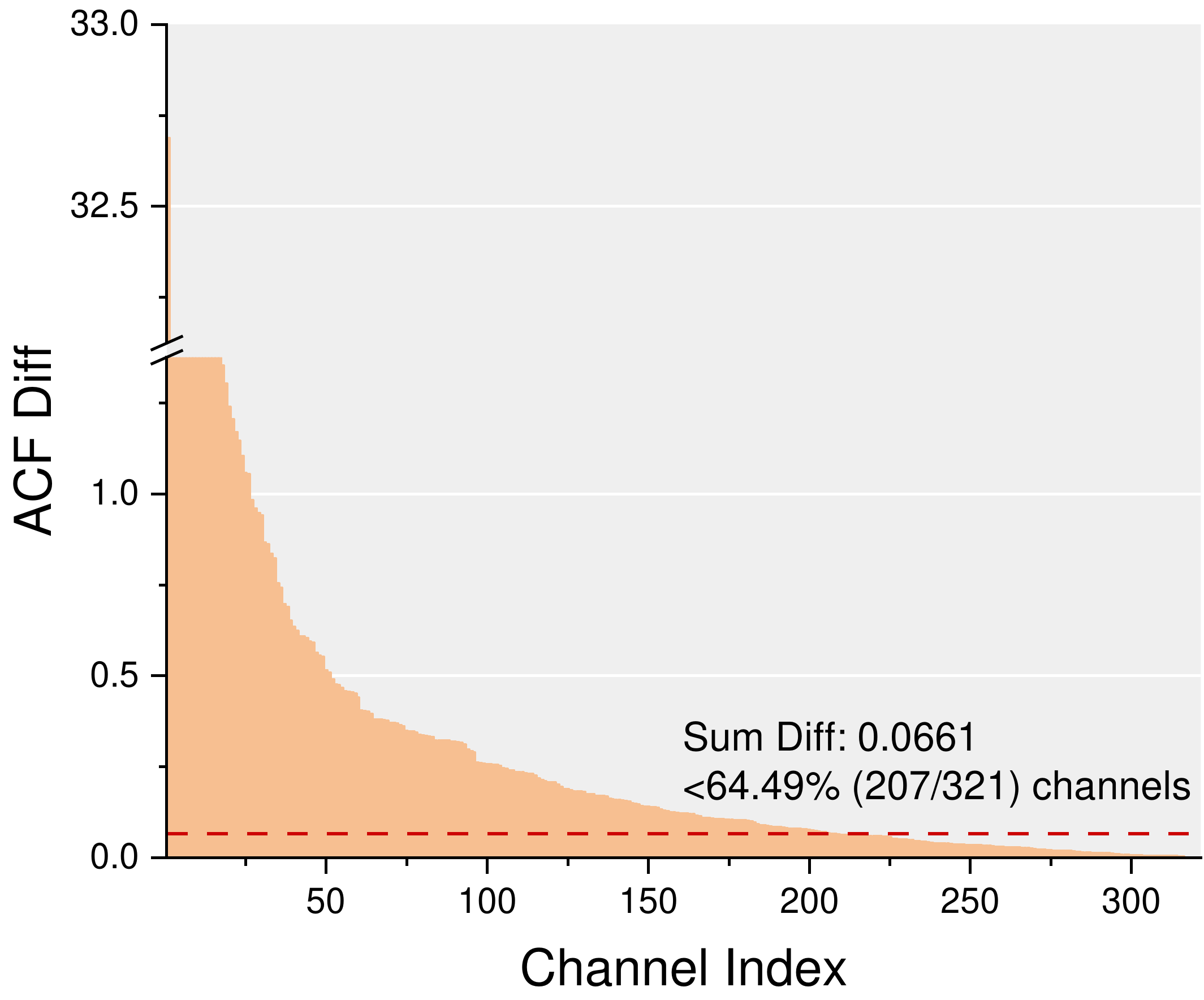}
\caption{Electricity}
\end{subfigure}
\begin{subfigure}[b]{0.23\linewidth}
\includegraphics[width=\linewidth]{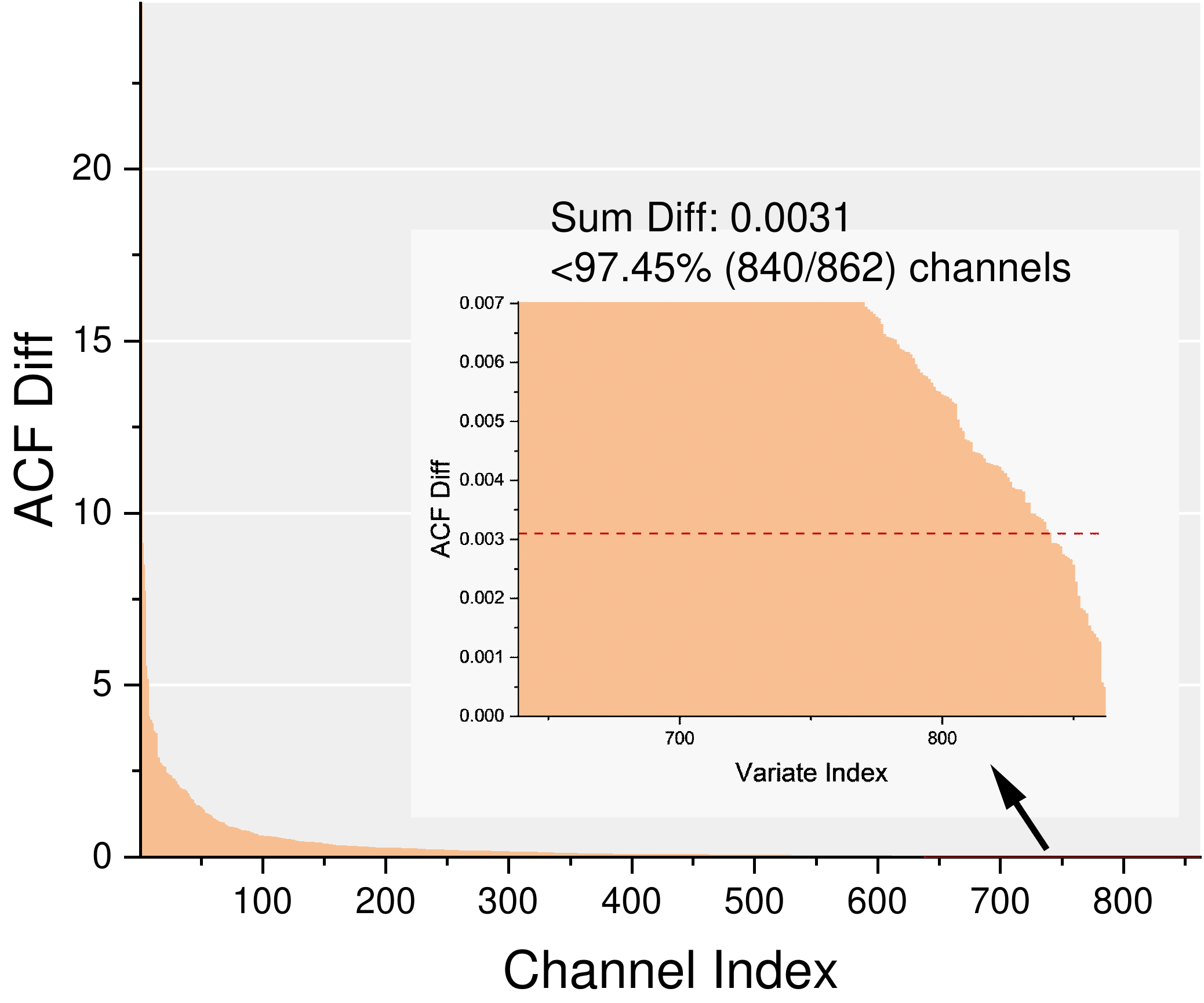}
\caption{Traffic}
\end{subfigure}
\vspace{-4mm}
\caption{The difference of ACF between training data and test data. The ACF difference for each channel is depicted in bar charts, arranged in descending order. The sum diff, which represents the overall ACF difference under the CI strategy, is shown as a horizontal line. the sum diff is smaller than the ACF difference of each channel, indicating that the CI strategy can effectively mitigate distribution drift.}
\label{eq:acf_diff}
\end{figure*}

\begin{tcolorbox}[title={Takeaways}]
The sum of ACF differences between training and test data exhibits less variation than the ACF differences of most individual channels. This means employing the CI strategy results in reduced distribution drift.
\end{tcolorbox}

\subsection{Capacity and Robustness}
Although the CI strategy can reduce the distribution gap between training and test data, we cannot conclude that it leads to better generalization performance. This is primarily due to the fact that the hypothesis spaces under CI and CD strategies are not the same, where $\W_{ci} \in \R^{L \times H}$ and $\W_{cd} \in \R^{LC \times HC}$. To analyze the risks associated with these strategies, we follow the risk analysis framework in machine learning proposed by Mohri et al.~\cite{mohri2018foundations} and decompose the risk according to the following equation:
\begin{equation}
		\cR(\hat{\W})=\underbrace{\left(\cR(\hat{\W})-\inf_{\W \in \mathcal{W}} \cR(\W)\right)}_{\text {non-robustness}}+\underbrace{\inf_{\W \in \mathcal{W}} \cR(\W)}_{\text{incapacity}}. \\
	\label{eq:risk_decomp}
\end{equation}
where $\cR(\cdot)$ is the risk defined by \cref{eq:multi_expect} and $\hat{\W}$ denotes the model obtained by minimizing the empirical CD loss as in \cref{eq:multi_emp}, or the empirical CI loss as in \cref{eq:uni_emp}.
We use the notation $\mathcal{W}$ to refer to the hypothesis space. In this paper, we consider two types of hypothesis space: the hypothesis space of CI, denoted by $\mathcal{W}{ci} = \R^{L \times H}$, and the hypothesis space of CD, denoted by $\mathcal{W}{cd} = \R^{LC \times HC}$.  

Unlike the traditional terminology~\cite{mohri2018foundations}, we interpret the first term as the \textbf{(non-)robustness} of a model, which is the risk gap between the model trained on the training set and the optimum model on the test data distribution. It measures the ability of the model to handle unseen data and achieve nearly optimal performance. A lower value of this term indicates a more robust model. The \textbf{(in)capacity} measures how well the optimum model fits the data, with a lower value indicating a better ability to fit the data. Simple algorithms like linear regression usually have low capacity (high incapacity), while complex algorithms like neural networks have high capacity. In addition to the choice of algorithm, different training strategies such as CI and CD also affect the robustness and capacity of the obtained model. In the following sections, we provide empirical results to further illustrate this concept.

It is not possible to calculate the risk $\cR$ directly as access to the underlying data distribution is unavailable. In this paper, we opt to approximate it using the empirical risk on the test data. For ease of reference, we represent the training and test set utilizing the CI strategy as $(A^{(tr)}{ci},B^{(tr)}{ci})$ and $(A^{(te)}{ci},B^{(te)}{ci})$, while the training and test set using the CD strategy is denoted as $(A^{(tr)}{cd},B^{(tr)}{cd})$ and $(A^{(te)}{cd},B^{(te)}{cd})$. We compute the subsequent statistics to demonstrate the performance of CI and CD on the benchmarks: 
\begin{enumerate}
	\item \textbf{Train Error (Incapacity).} The training error $\mathcal{L}^{(tr)}_i$ is computed as the following:
	\begin{equation}
		\mathcal{L}^{(tr)}_i = \lVert \A^{(tr)}_i \W^{(tr)}_i - \B^{(tr)}_i \rVert^2_F, \quad i\in \{ci,cd\}
	\end{equation}
where
$$\W^{(tr)}_i = \argmin_\W \lVert \A^{(tr)}_i \W - \B^{(tr)}_i \rVert^2_F$$ is the optimum parameter for the training data. \textbf{Train Error} is also a measure of capacity but empirically computed on the training set.

\item \textbf{Test Error (Incapacity).} The test error $\mathcal{L}^{(te)}_i$ is computed as the following:
\begin{equation}
			\mathcal{L}^{(te)}_i = \lVert \A^{(te)}_i \W^{(te)}_i - \B^{(te)}_i \rVert^2_F, \quad i\in \{ci,cd\}
\end{equation}
where:
$$\W^{(te)}_i = \argmin_\W \lVert \A^{(te)}_i \W - \B^{(te)}_i \rVert^2_F$$ is the optimum parameter for the test data. Test loss describes the best error a linear model can achieve on the test data. It is an approximation of $\inf_{\W \in \mathcal{W}}\cR(\W)$ in \cref{eq:risk_decomp}. 

\item \textbf{Gen Error ($\cR(\hat{\W})$).} The generalization error $\mathcal{L}^{(gen)}_i$is computed as:
$$
\mathcal{L}^{(gen)}_i = \lVert \A^{(te)}_i \W^{(tr)}_i - \B^{(te)}_i \rVert^2_F. \quad i\in \{ci,cd\}
$$
It is the performance measure on the benchmarks.

\item \textbf{W Diff (Non-Robustness).} It is an approximation of non-robustness in \cref{eq:risk_decomp}. Its value is computed as:
\begin{equation}
	\operatorname{Diff}_{W_i} = \lVert \A^{(te)}_i (\W^{(tr)}_i - \W^{(te)}_i)\rVert^2_F.
	\label{eq:w_diff}
\end{equation}
\cref{eq:w_diff} is inspired by ordinal least square in fixed design settings~\cite{mohri2018foundations}, where the estimation error is computed as the Mahalanobis distance between $\W^{(tr)}_i$ and $\W^{(te)}_i$. \cref{eq:w_diff} is an extension of Mahalanobis distance, since:
\begin{equation}
	\begin{aligned}
		&\operatorname{Diff}_{W_i} = \lVert \A^{(te)}_i (\W^{(tr)}_i - \W^{(te)}_i)\rVert^2_F\\ =& \operatorname{tr}((\W^{(tr)}_i - \W^{(te)}_i)^\top (\A^{(te)}_i)^\top \A^{(te)}_i (\W^{(tr)}_i - \W^{(te)}_i)) \\
		=& \operatorname{tr}((\W^{(tr)}_i - \W^{(te)}_i)^\top \hat{\Sigma}^{(te)}_i (\W^{(tr)}_i - \W^{(te)}_i))
	\end{aligned}.
\label{eq:mah}
\end{equation}
$\operatorname{tr}$ is the trace operation for a matrix and $\hat{\Sigma}^{(te)}_i= (\A^{(te)}_i)^\top \A^{(te)}_i$ is the unnormalized sample covariance matrix. When  $\W^{(tr)}_i$ and $\W^{(te)}_i$ become vectors, \cref{eq:mah} falls back to the Mahalanobis distance parameterized by Mahalanobis matrix $\hat{\Sigma}^{(te)}_i$. 
In this sense, the W diff can also be considered as a measure of distribution drift, since it is a distance measure between $\W^{(tr)}_i$ and $\W^{(te)}_i$, and $\W^{(tr)}_i$ and $\W^{(te)}_i$ are derived by the ACF of train and test data.

Another interpretation of $\operatorname{Diff}_{W_i}$ takes it as a lower bound for the estimation error:
$$
\begin{aligned}
	\operatorname{Diff}_{W_i} &=\lVert \A^{(te)}_i (\W^{(tr)}_i - \W^{(te)}_i)\rVert^2_F  \\
	&\le\lVert \A^{(te)}_i \W^{(te)}_i - \B^{(te)}_i \rVert^2_F - \lVert \A^{(te)}_i \W^{(tr)}_i - \B^{(te)}_i \rVert^2_F \\
	& = \mathcal{L}^{(gen)}_i - \mathcal{L}^{(te)}_i, \quad i\in \{ci,cd\}
\end{aligned}
$$
\end{enumerate}

\begin{figure*}[h]
	\begin{subfigure}[b]{0.19\linewidth}
		\includegraphics[width=\linewidth]{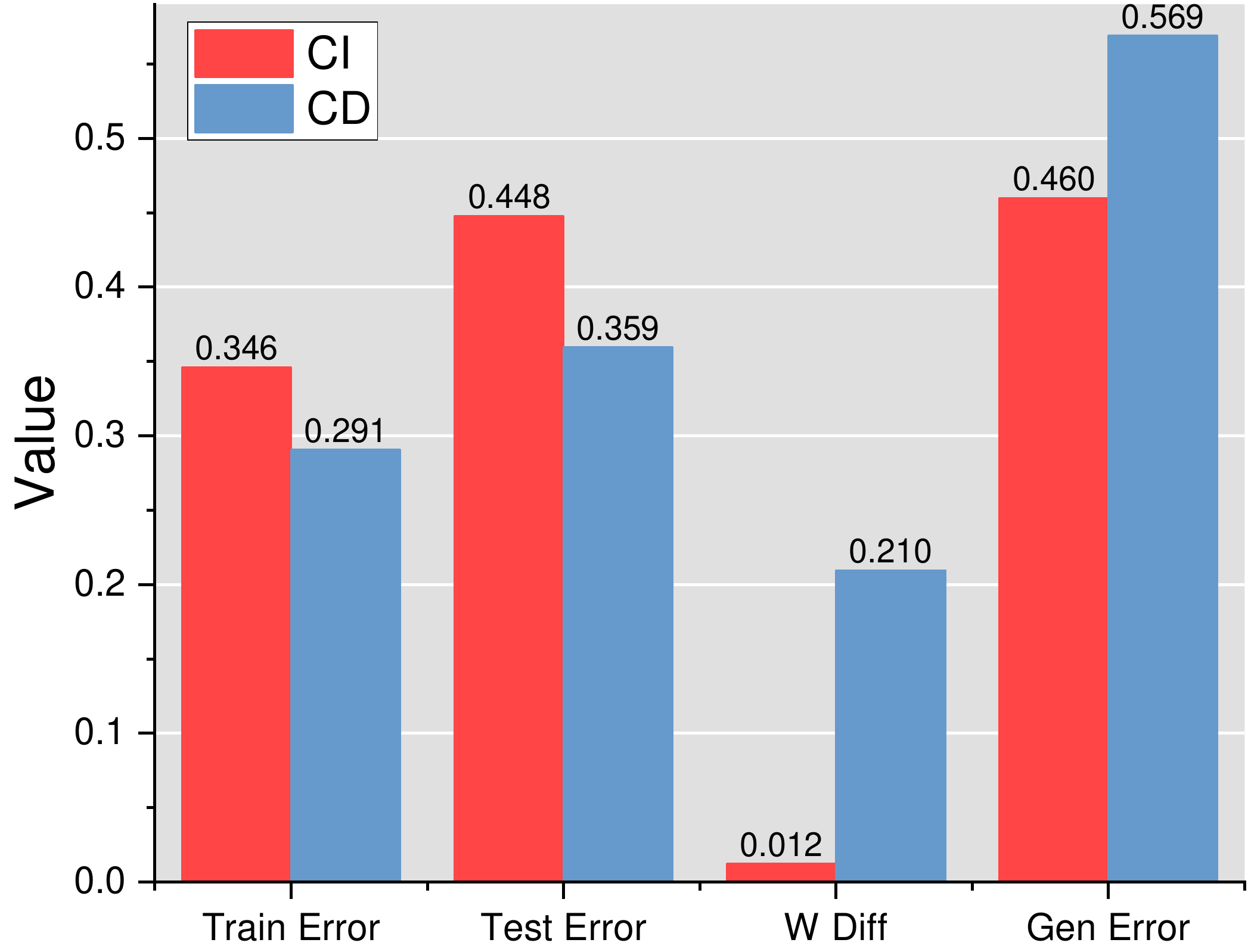}
		\caption{ETTh1}
	\end{subfigure}
	\begin{subfigure}[b]{0.19\linewidth}
		\includegraphics[width=\linewidth]{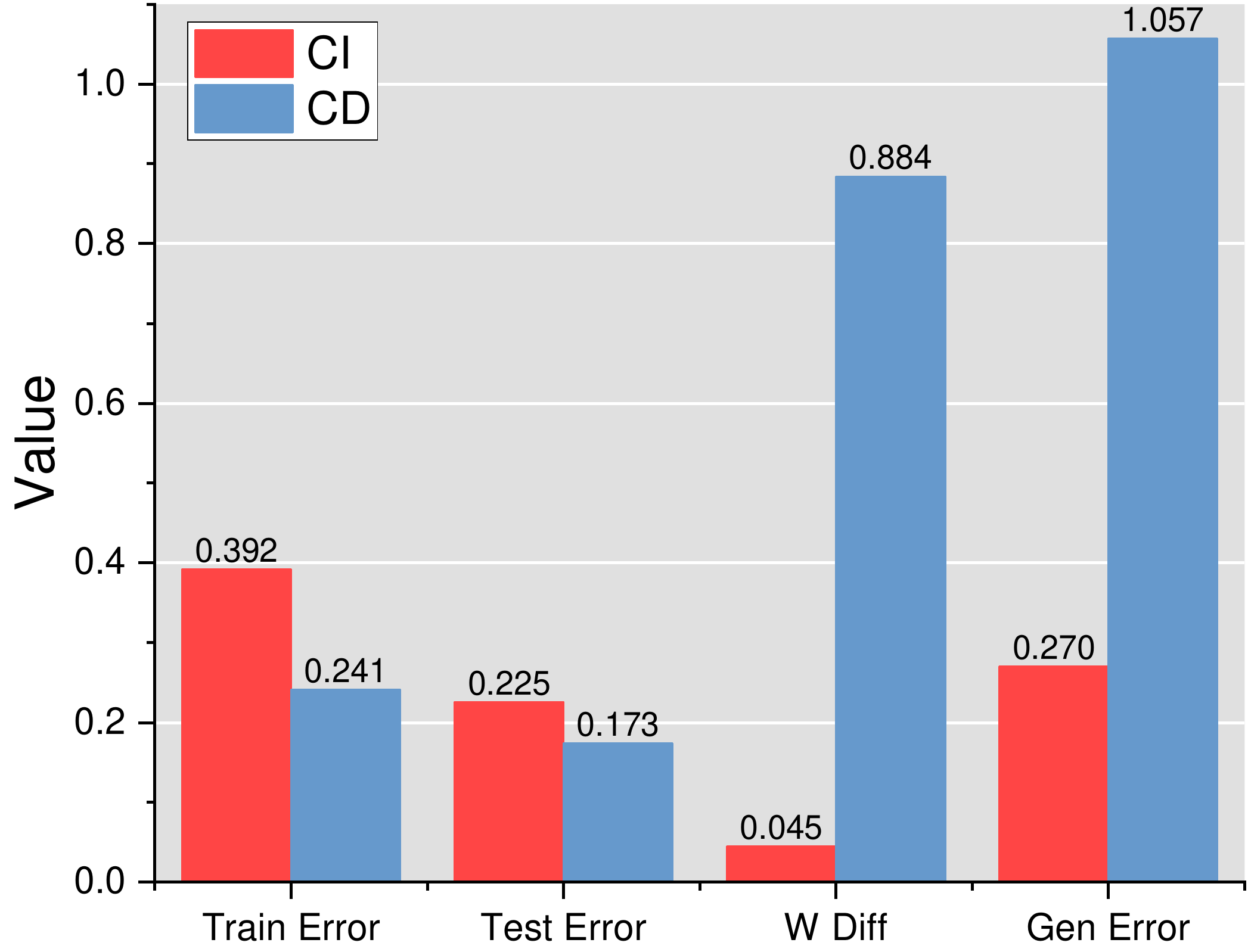}
		\caption{ETTh2}
	\end{subfigure}
	\begin{subfigure}[b]{0.19\linewidth}
		\includegraphics[width=\linewidth]{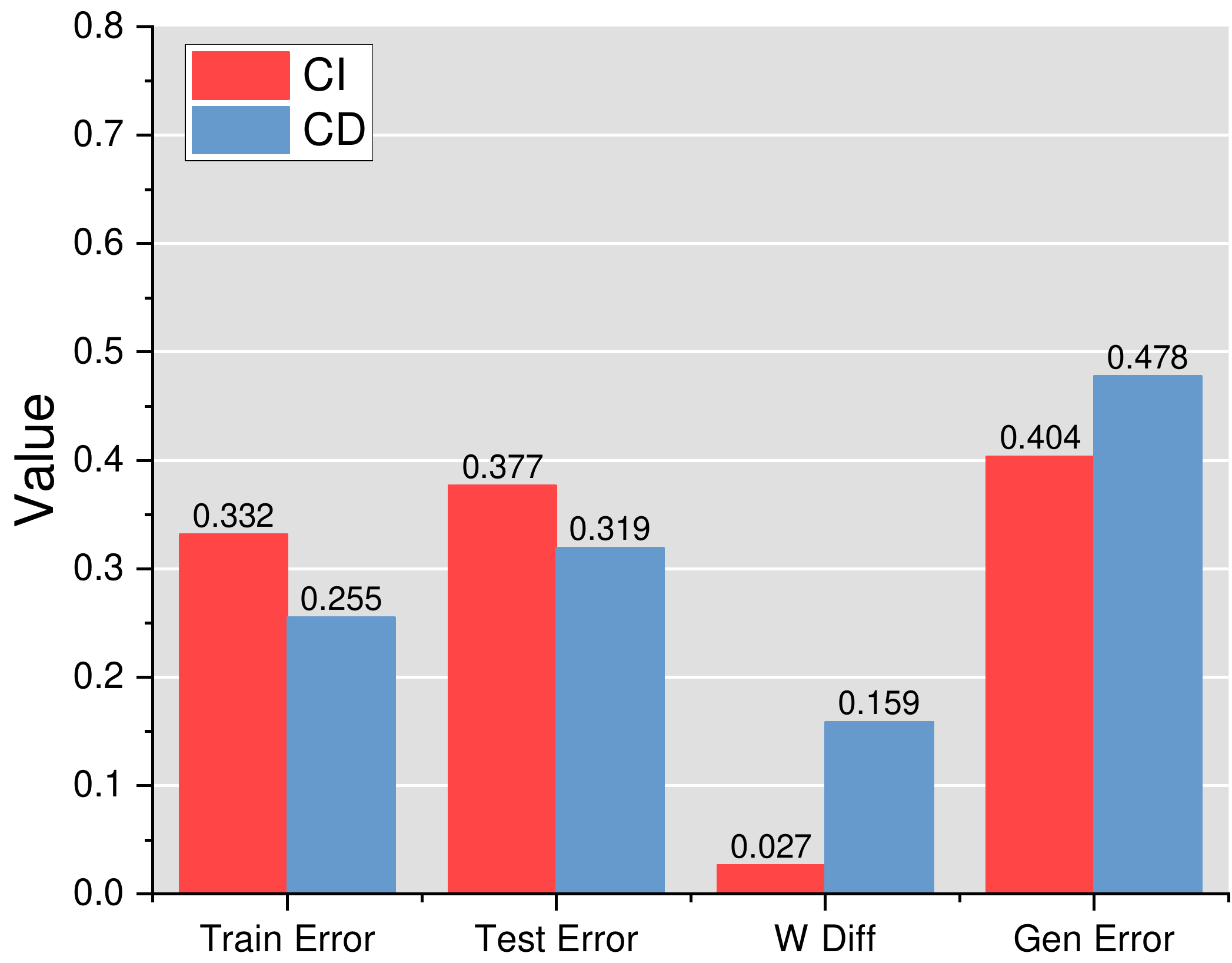}
		\caption{ETTm1}
	\end{subfigure}
	\begin{subfigure}[b]{0.19\linewidth}
		\includegraphics[width=\linewidth]{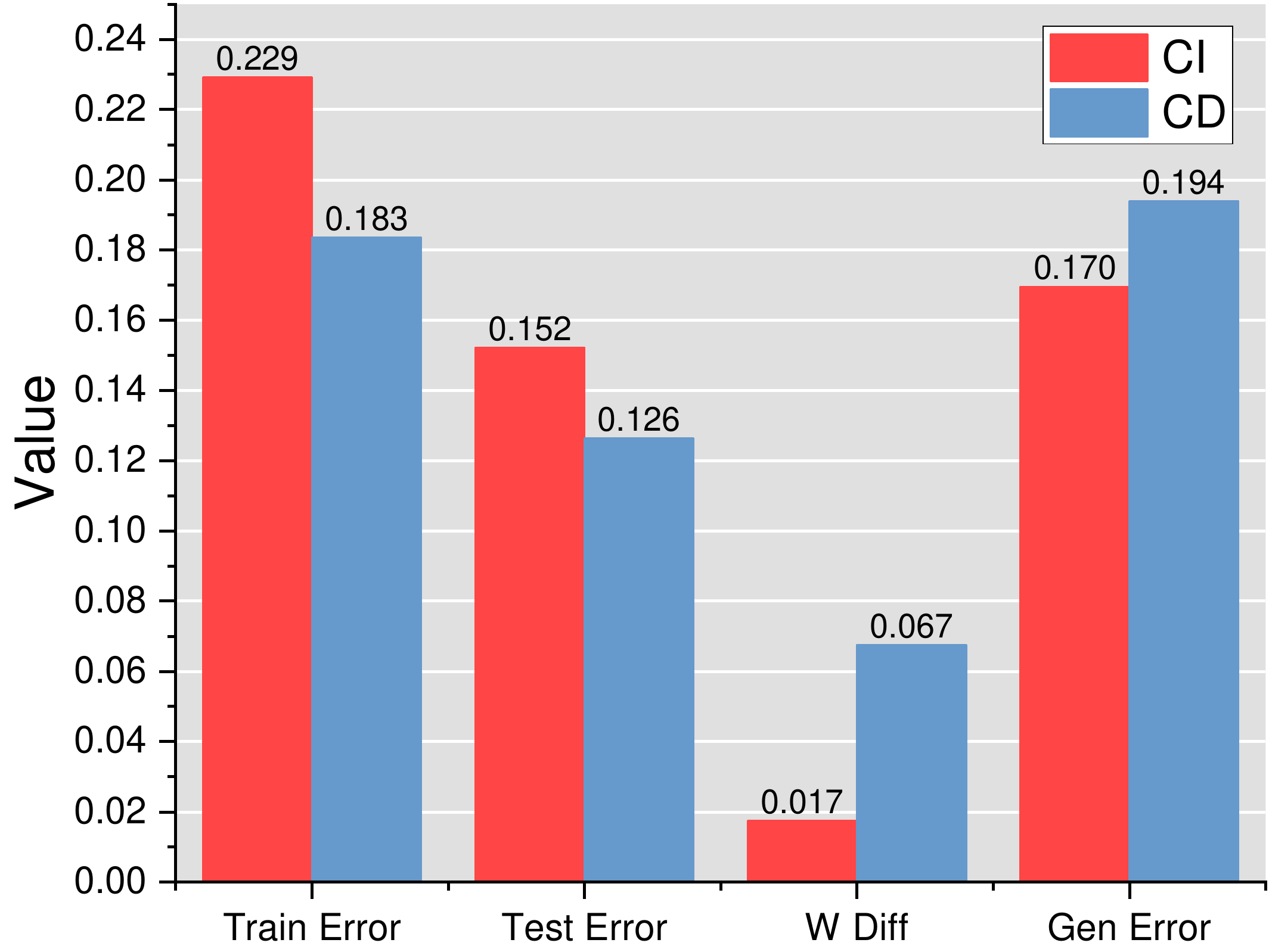}
		\caption{ETTm2}
	\end{subfigure}
	\begin{subfigure}[b]{0.19\linewidth}
		\includegraphics[width=\linewidth]{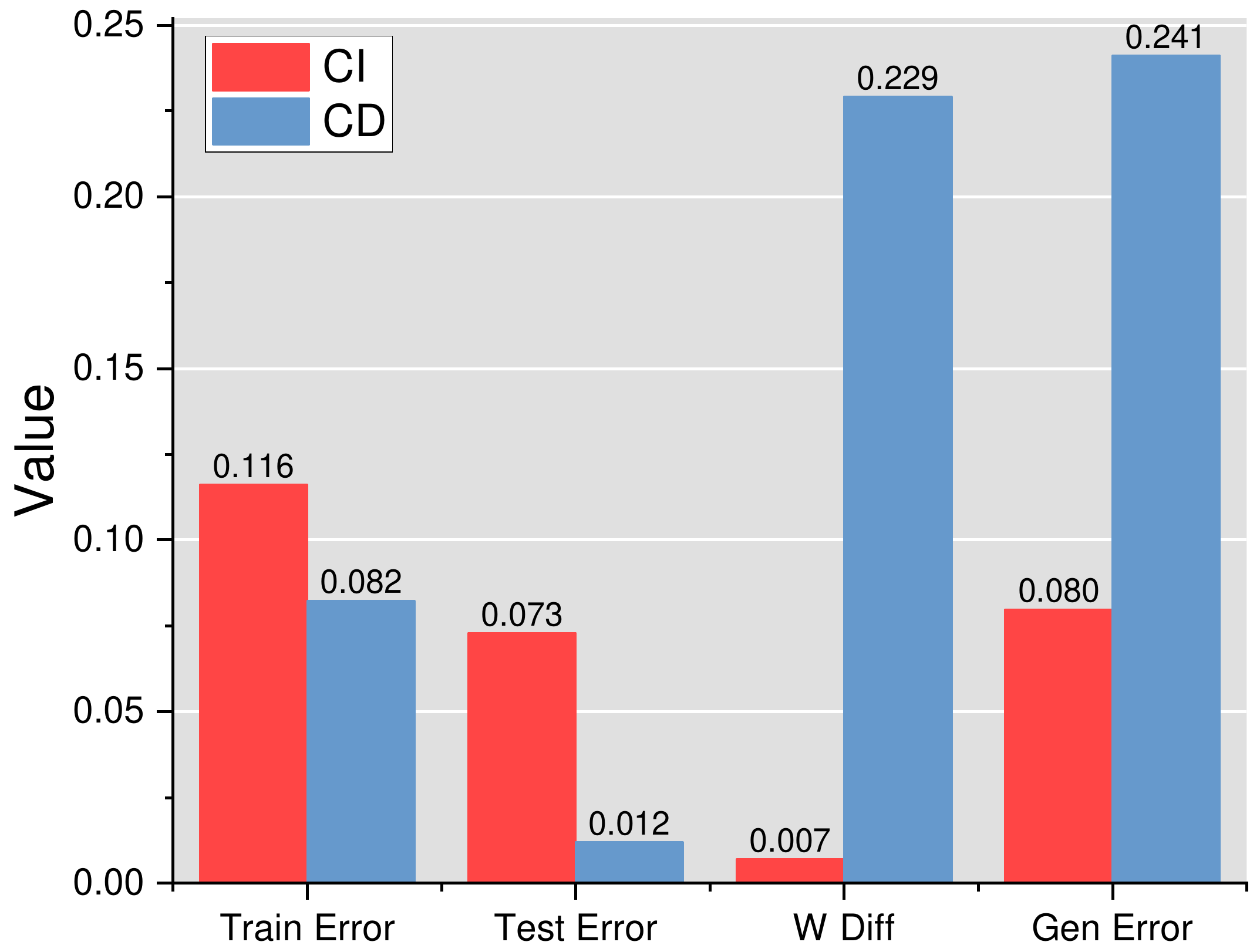}
		\caption{Exchange-Rate}
	\end{subfigure}
	\begin{subfigure}[b]{0.23\linewidth}
		\includegraphics[width=\linewidth]{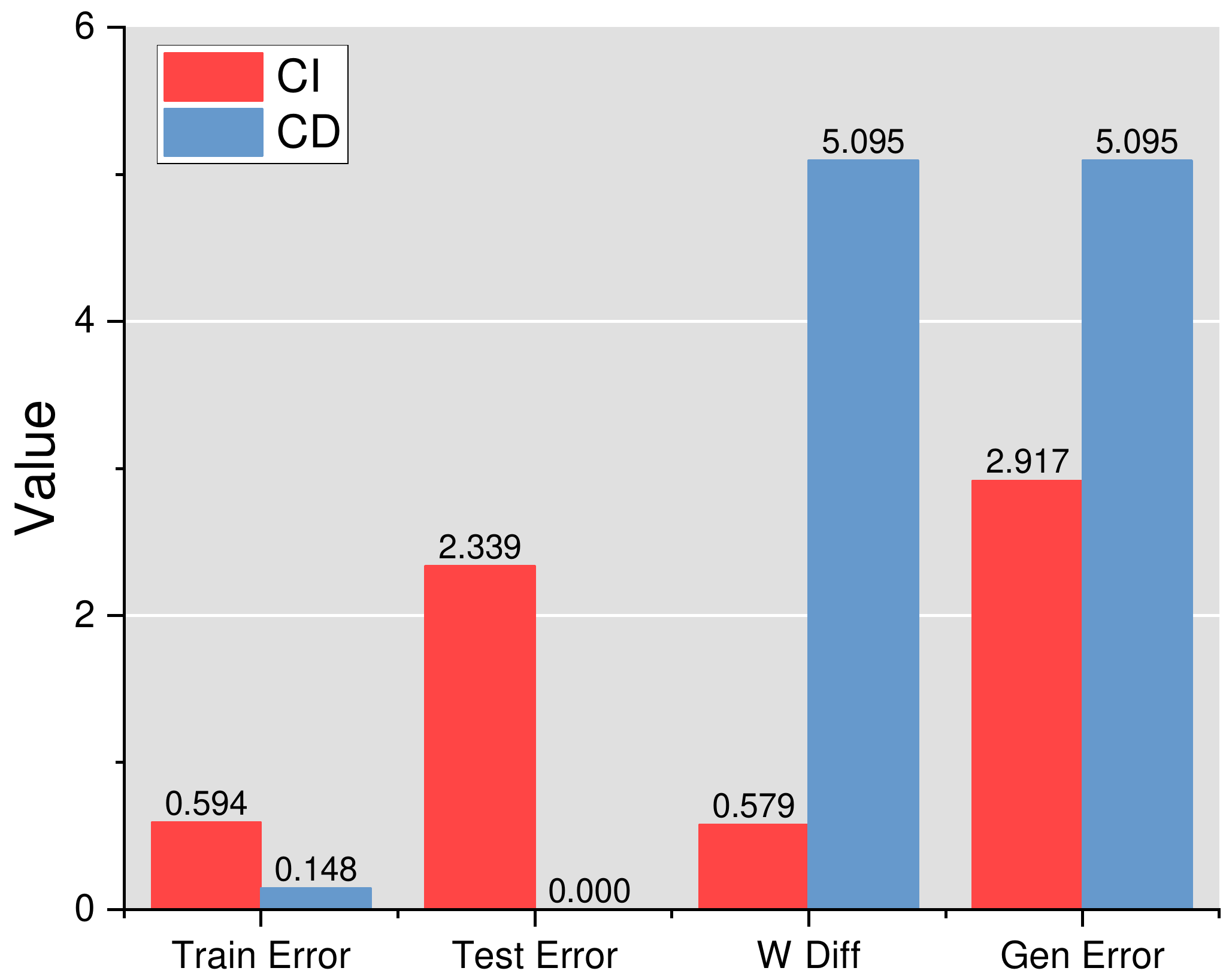}
		\caption{ILI}
	\end{subfigure}
	\begin{subfigure}[b]{0.23\linewidth}
		\includegraphics[width=\linewidth]{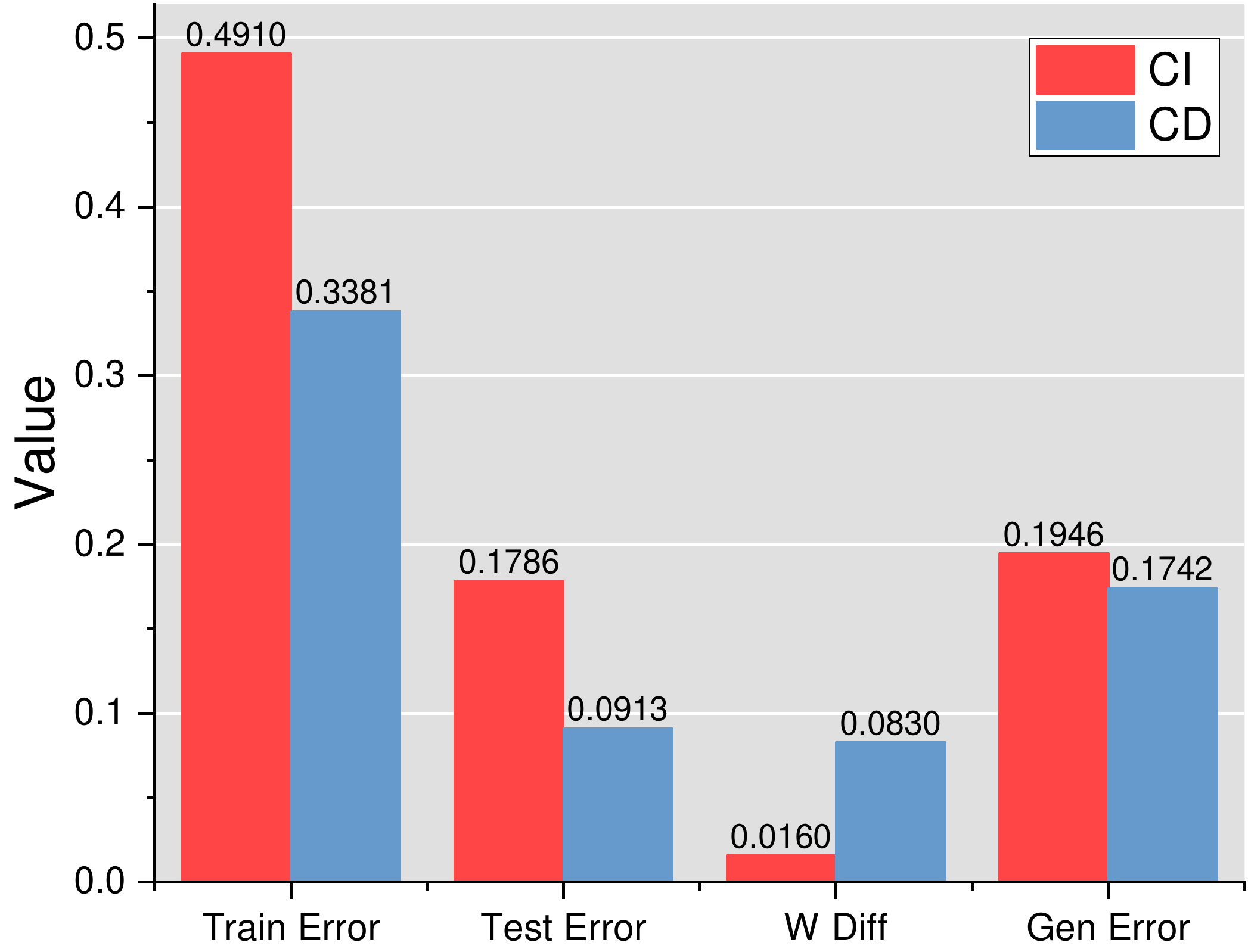}
		\caption{Weather}
	\end{subfigure}
	\begin{subfigure}[b]{0.23\linewidth}
		\includegraphics[width=\linewidth]{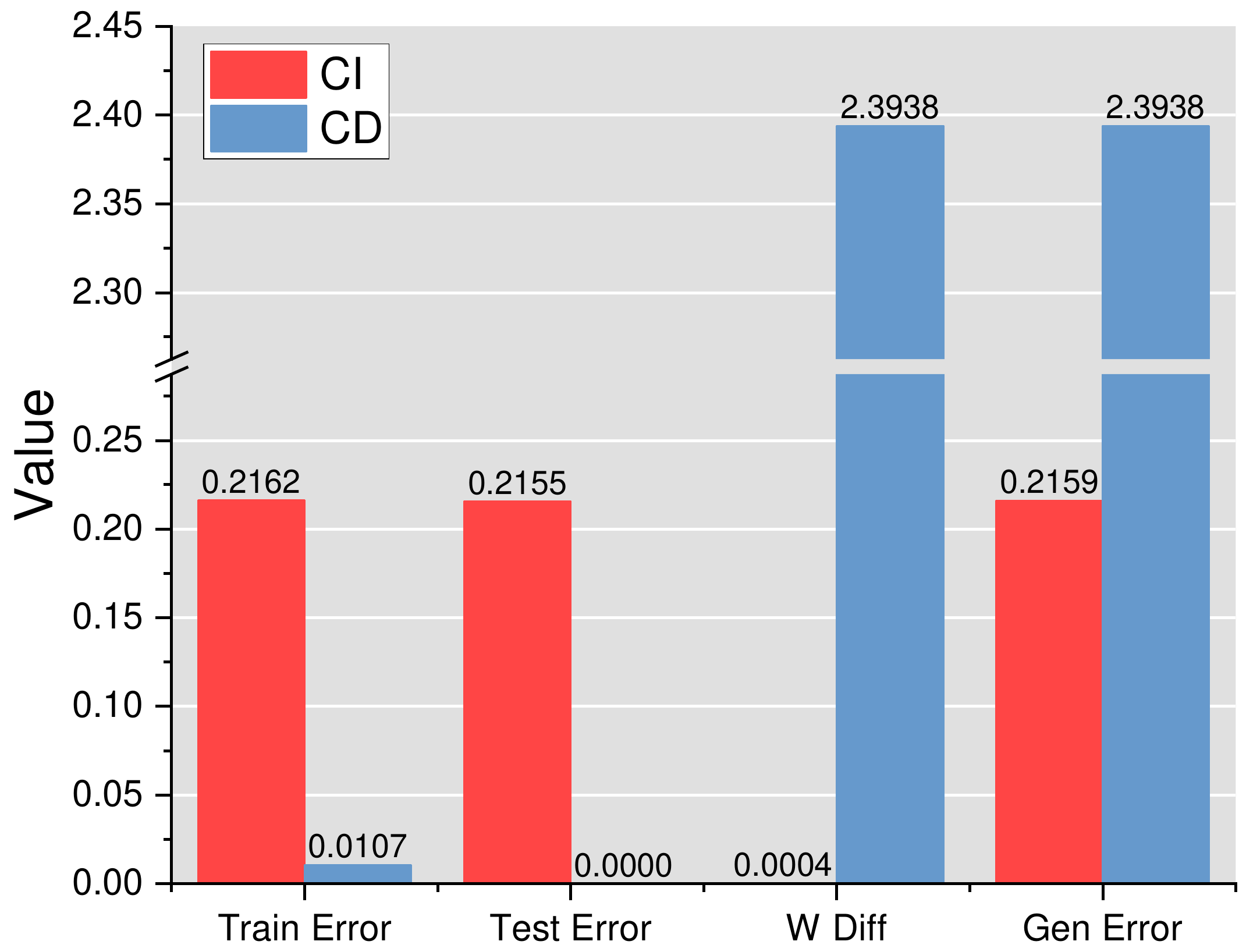}
		\caption{Electricity}
	\end{subfigure}
	\begin{subfigure}[b]{0.23\linewidth}
		\includegraphics[width=\linewidth]{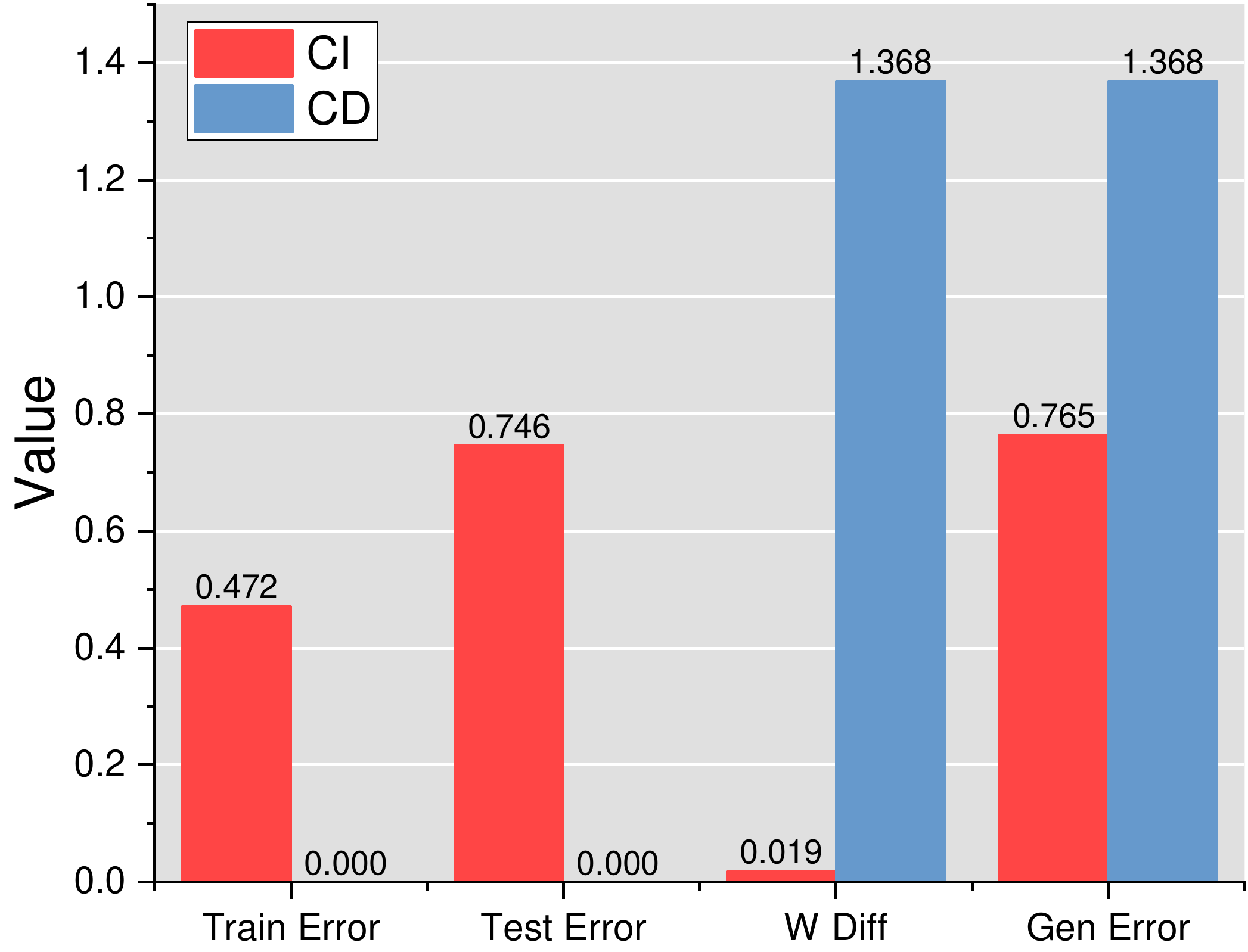}
		\caption{Traffic}
	\end{subfigure}
	\caption{The train error, test error, W diff and gen error when using CI and CD strategy on the 9 datasets. Train/test error measures model capacity on train/test data. W diff measures the difference between the optimal model on train and test data. It reveals the robustness of a model. Gen error measures the risk of an algorithm. Although CD can achieve lower optimal error, it is much less robust to the distribution drift than CI. Consequently, in most cases, CI outperforms CD.}
	\label{fig:risks}
\end{figure*}

Similar to the risk decomposition (\cref{eq:risk_decomp}), we can also decompose the gen loss by the following equation:
\begin{equation}
	\mathcal{L}^{(gen)}_i = \underbrace{(\mathcal{L}^{(gen)}_i - \mathcal{L}^{(te)}_i)}_{\approx \operatorname{Diff}_{W_i}} + \mathcal{L}^{(te)}_i. \quad i\in \{ci,cd\}
	\label{eq:error_decomp}
\end{equation}

We illustrate the above statistics on the 9 datasets in \cref{fig:risks} and draw the following conclusions. (1) \textbf{CD models exhibit lower train/test loss as compared to CI models, indicating that CD strategy trains a model with higher capacity.} On all of the 9 models, the train/test error is always lower than CI. When the number of channels is large, CD model may have 0 errors. This trend is consistent across all 9 models, with the CD model registering zero errors for larger channel numbers. This outcome is anticipated due to the fact that the hypothesis space of CI is a subset of CD, thereby enabling the construction of a CD linear model by replicating elements of a CI linear model. As per the definition, the best CD model will inevitably have a lower error rate than CI. (2) \textbf{CD models have a significantly larger W diff than CI models, indicating that CI is much more robust than CD.} This trend is apparent across all 9 datasets, with CD models having W diff values usually over 10 times the value of CI. The phenomenon is especially conspicuous in datasets with multiple channels such as Electricity (h) and Traffic (i). The distribution drift between train and test data is responsible for this trend. In the previous section, we have shown that the ACF coefficients of CD and CI models are determined by the ACF of all channels and the sum of ACF across all channels, with the latter exhibiting a lower difference than the former. This difference contributes to the gap between optimal models on train and test data being different for CD and CI strategies. CI strategy leads to a lighter distribution gap, resulting in a smaller W Diff value. (3) W Diff values are more significant than Test Error in most cases, leading to CI models having lower Gen Error than CD models. \ie, \textbf{Robustness is more crucial than capacity.} For the 4 ETT benchmarks, there is not much difference in test error between CI and CD models, but W diff values are significantly distinct. Hence, CI models perform better than CD models in terms of gen loss. On datasets (e), (f), (h), (i), the test loss varies considerably, but the W diff values differ significantly more than the test loss, leading to CI models performing better than CD models. The only exception to this trend is the weather benchmark (g), where test error holds greater significance than W diff values. Consequently, CD strategy performs better in this case. \Cref{fig:equipotential} summarizes these findings.

\begin{figure}[h]
	\includegraphics[width=0.80\linewidth]{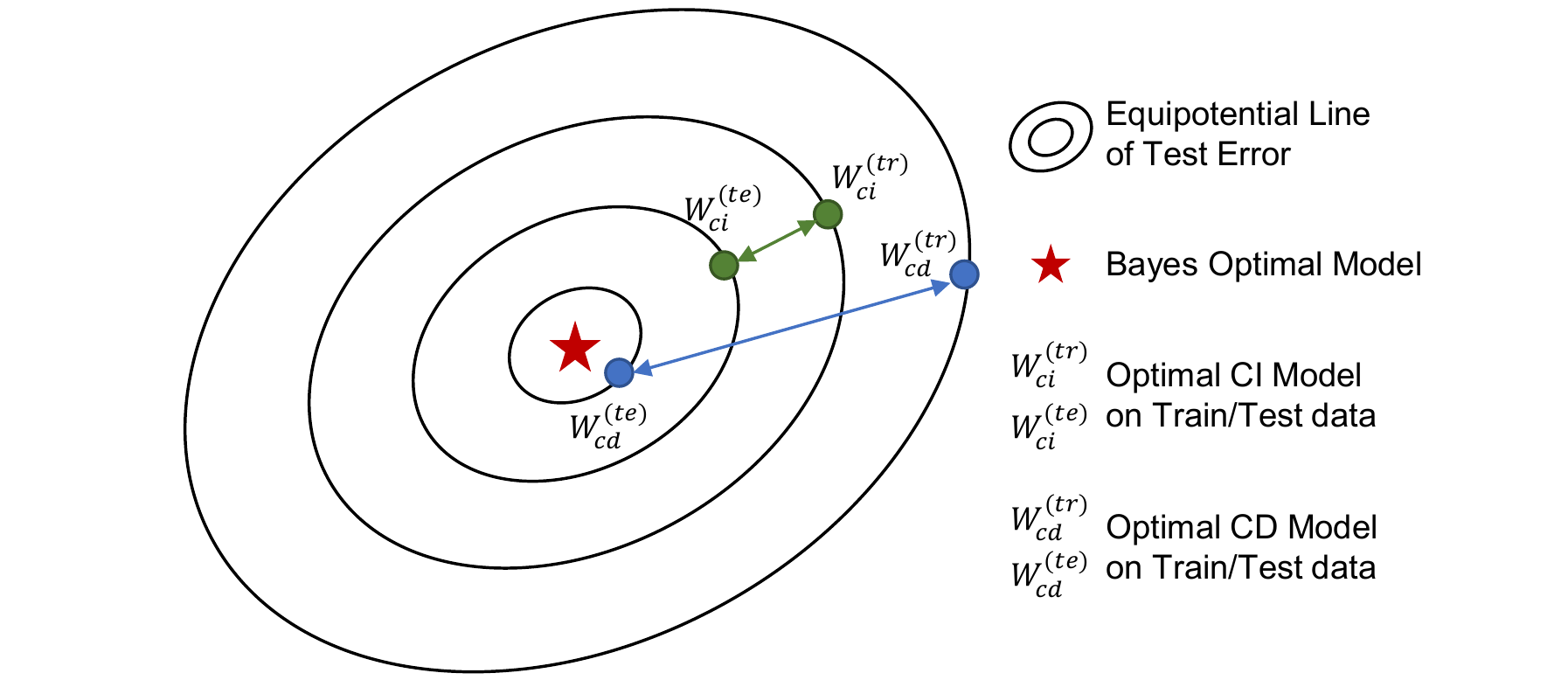}
 \vspace{-4mm}
	\caption{Illustration of the conclusions drawn from \cref{fig:risks}. CD have better optimal model on test data but larger distance between the optimal train and test model. CI is the contrary. In most cases, model difference matters more than the optimal model. Thus CI often achieves better performance. }
	\label{fig:equipotential}
\end{figure}

From analyses in this section, we draw an overall conclusion:

\begin{tcolorbox}[title={Section Conclusion}]
	The Channel Dependent (CD) strategy has high capacity but low robustness. The Channel Dependent (CI) strategy has low capacity but high robustness. In numerous real-world non-stationary time series with distribution drifts, robustness is a more crucial factor than capacity in forecasting tasks. Consequently, CI strategy often delivers better performance.
\end{tcolorbox}


To provide readers with a clear understanding, we demonstrate the differences in prediction outcomes between the CD and CI strategies using a visual representation in \cref{fig:pred}. These examples were selected as they are indicative of the findings obtained across the experiments. From these four figures, We observe that the CD approach produces sharp predictions, while the CI approach generates smoother predictions. This discrepancy can be attributed to the sum over effect, which we have analysed in detail in the previous subsection. Unfortunately, the non-robust and sharp nature of CD predictions make them unsuitable for accurately predicting real-world non-stationary time-series. \Cref{fig:pred}.(a) displays a scenario where both strategies capture the correct trend and seasonal component, but the CI approach more closely aligns with the ground-truth compared to the CD approach. (b) shows that the CD approach may predict incorrect trends, while the CI approach is less prone to making such errors. When faced with anomalous time-series like (c), the CI approach is more robust and produces less oscillation. Nonetheless, there are instances where the CD approach outperforms the CI approach, as shown in \cref{fig:pred}.(d), where the CD model's high capacity can be advantageous for capturing complex but predictable patterns. Intuitively, real-world time series that exhibit regular patterns and smooth changes are predictable. Conversely, drastically oscillating time series, such as the one depicted in \cref{fig:pred}.(c), are anomalies and therefore unpredictable. In such cases, conservative and robust predictions are preferable. As a result, the CI strategy yields superior results on average compared to the CD strategy.
\begin{figure*}[htp]
\centering
    \begin{subfigure}[b]{0.24\textwidth}
		\centering
		\includegraphics[width=\textwidth]{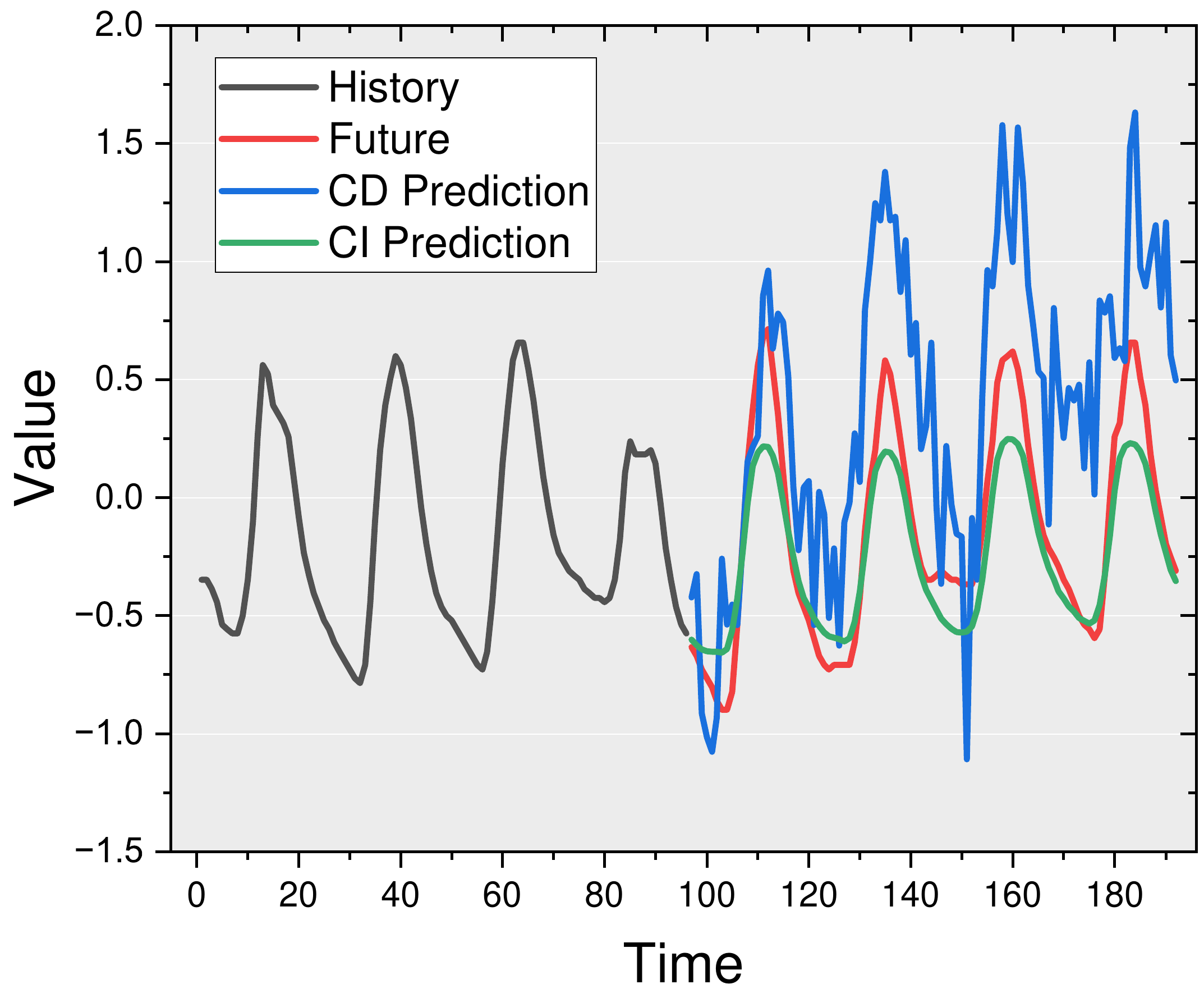}
  \caption{}
	\end{subfigure}
    \begin{subfigure}[b]{0.24\textwidth}
		\centering
		\includegraphics[width=\textwidth]{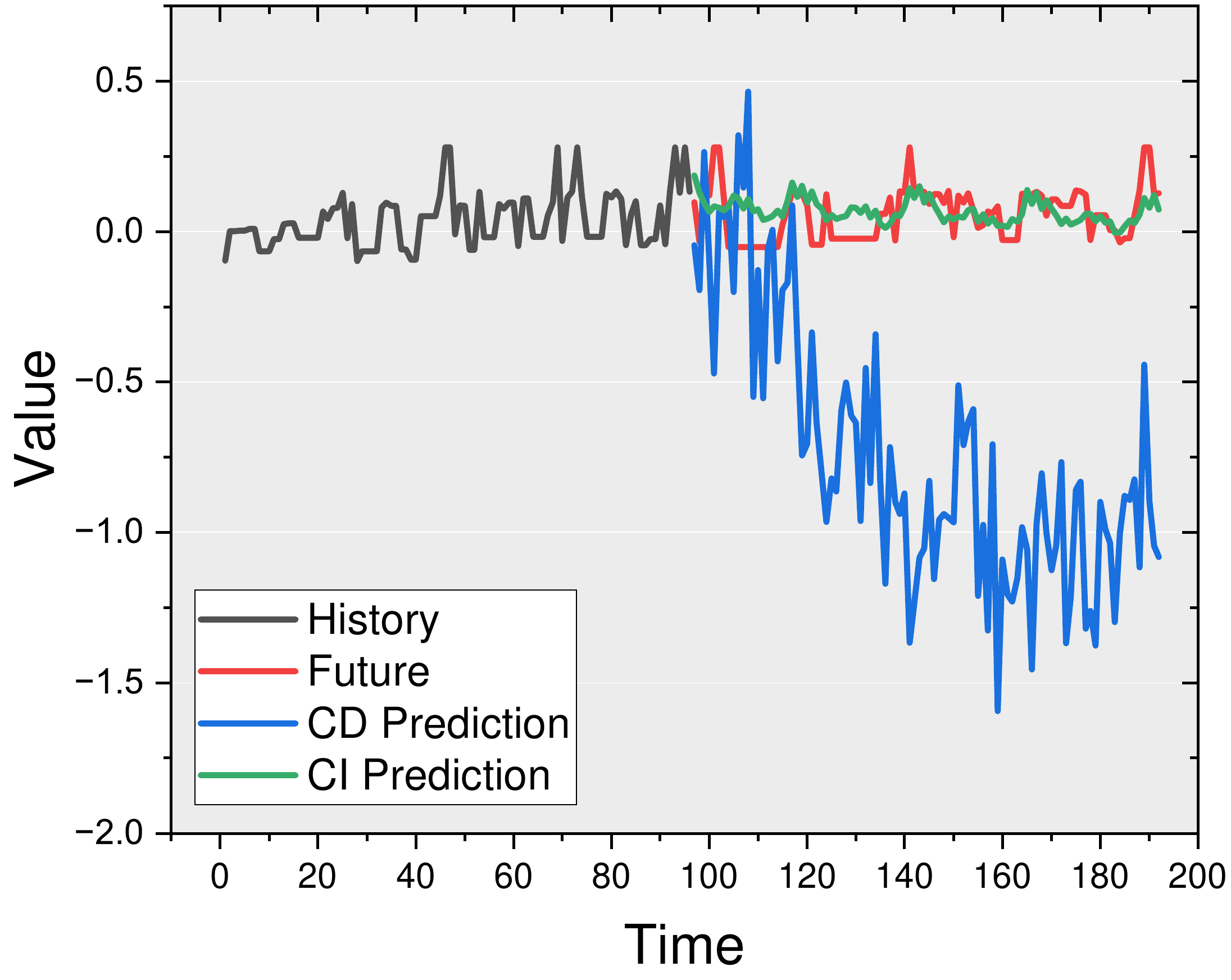}
  \caption{}
	\end{subfigure}
    \begin{subfigure}[b]{0.24\textwidth}
		\centering
		\includegraphics[width=\textwidth]{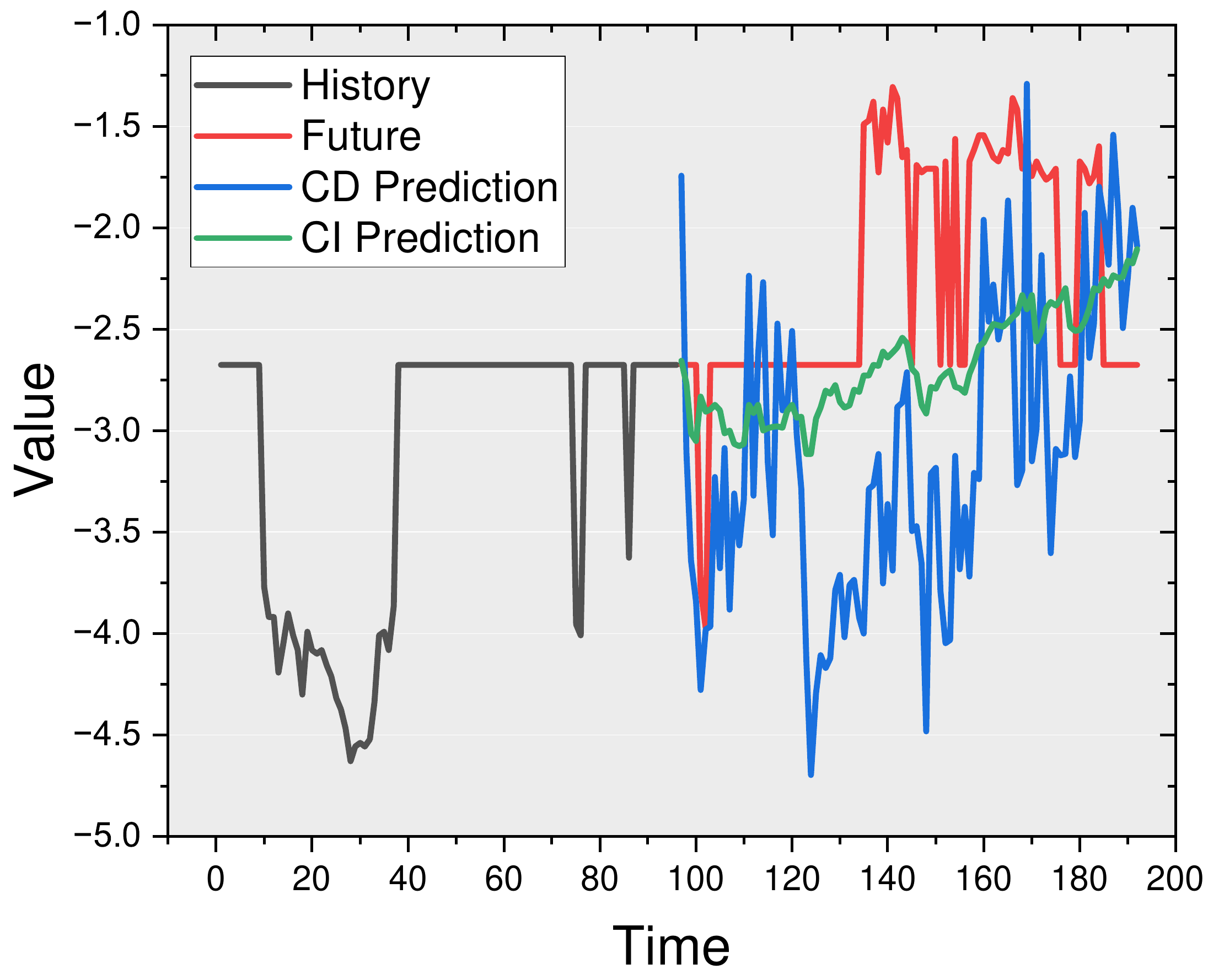}
  \caption{}
	\end{subfigure}
    \begin{subfigure}[b]{0.24\textwidth}
		\centering
		\includegraphics[width=\textwidth]{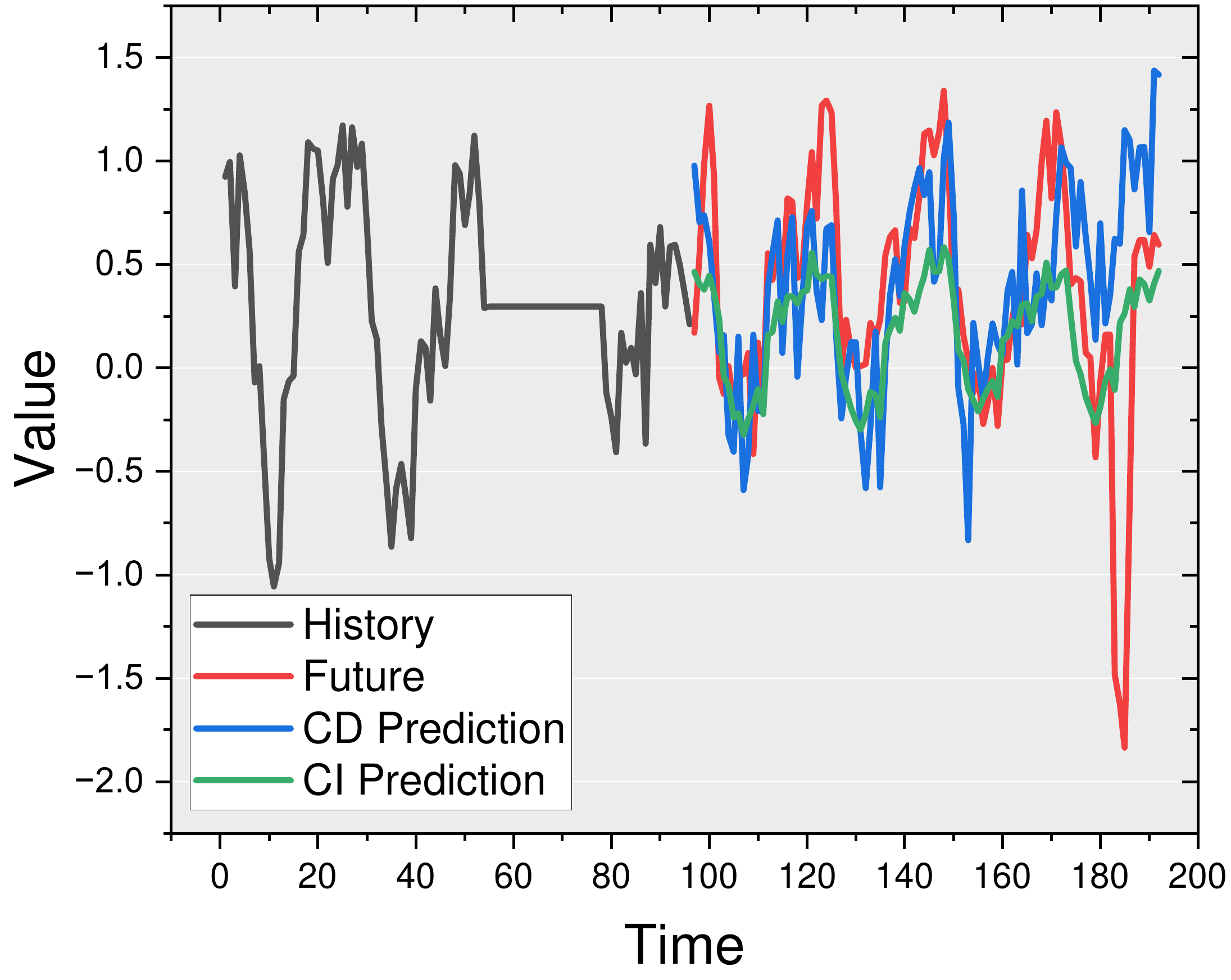}
  \caption{}
	\end{subfigure}
 \vspace{-4mm}
 \caption{We select four representative examples to demonstrate the differences between CD and CI strategies. The experiment is conducted on ETTh2 with Linear model. (a) When both of them capture the correct trend and seasonal component. CD tends to generate  ``sharp'' predictions, while the CI produces smoother ones. (b) The CD could predict wrong trends, while CI is less likely to do so. (c) When faced with abnormal series, CI are more robust with less oscillation. (d) The high capacity of CD models may be beneficial for capturing complex but predictable patterns. CI is unable to capture them.}
 \label{fig:pred}
\end{figure*}

\section{Practical Guides}
The previous section's analyses have revealed that the CD strategy exhibits high capacity but low robustness, making it unsuitable for handling real-world non-stationary time series where distribution drifts are significant. Conversely, the CI approach displays higher robustness, resulting in superior performance compared to the CD strategy. These findings provide valuable guidance for designing or improving existing multivariate forecasting models. Specifically, we recommend increasing the robustness of CD models and increasing the capacity of CI models.

In this section, we propose a simple modified CD objective that can help models surpass the CI strategy. Additionally, we discuss several factors that may influence the performance of CD or CI models. By considering these factors, we can further optimize and tailor the models for specific use cases.
\subsection{Predict Residuals with Regularization}

\begin{figure}[h]
	\centering
	\includegraphics[width=0.8\linewidth]{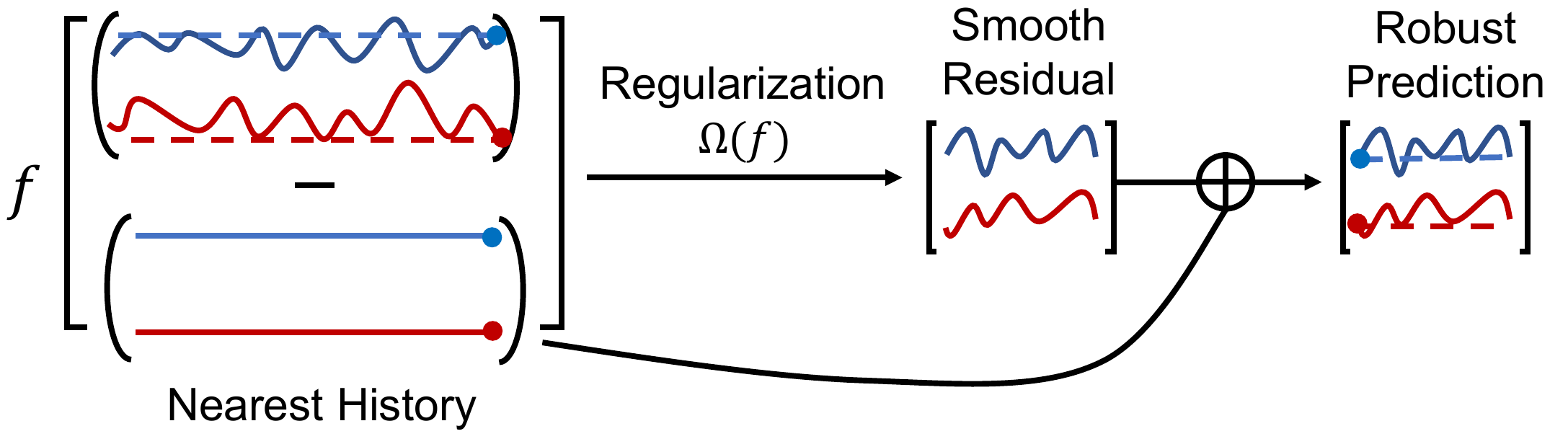}
	\vspace{-4mm}
	\caption{The idea of Predict Residuals with Regularization (PRReg). The input series is subtracted by the last value. Then the predictor is regularized to predict smoothed residual. Robust prediction is made by adding back the nearest history.}
	\label{fig:prreg}
\end{figure}

Our analysis of \cref{fig:pred} in the previous section has led us to conclude that the primary disadvantage of CD models is their tendency to generate "sharp" and non-robust predictions that often diverge from the actual trend. To address this issue, we propose a simple method to improve the performance of CD models called Predict Residuals with Regularization (PRReg), inspired by measures taken in N-BEATS~\cite{nbeats} and NLinear~\cite{Zeng2022Transformers}. The core idea of PRReg is to ensure that the prediction remains close to the nearest known history and that the forecasted series remains smooth. To achieve this objective, we reformulate the CD objective into the following form:
\begin{equation}
	\label{eq:PRReg}
	\min_f \frac{1}{N} \sum_{i=1}^N \ell(f(\X^{(i)}- N^{(i)})+ N^{(i)},\Y^{(i)})+\lambda \Omega(f).
\end{equation}
where $N^{(i)}= \X^{(i)}_{:,L}$ is the last values of each channel in $\X^{(i)}$. With this objective, the goal of $f$ is changed from accurately predicting future values to the variety from the nearest history. The regularization term $\Omega$ serves a dual purpose: to restrict the predictions within a reasonable distance from $N^{(i)}$ and to encourage smoothness in the predicted values. In our study, we adopted $L_2$ regularization, which was implemented as weight decay in PyTorch~\cite{NEURIPS2019_9015_pytorch}. Our proposed objective is applicable to various forecasting models, and its effectiveness is illustrated in \cref{fig:prreg} for Linear~\cite{Zeng2022Transformers} and Transformer~\cite{attention_is_all_you_need} models. \Cref{tb:prreg} presents the results, where we compare the PRReg strategy with CD and CI. We observe that PRReg outperforms both CD and CI in most cases when the regularization strength $\lambda$ is chosen appropriately. A too-small $\lambda$ fails to provide the required robustness since PRReg is fundamentally a CD strategy, while an excessively large regularization strength causes underfitting with sufficient capacity. Thus, choosing a suitable value of $\lambda$ results in an optimal trade-off between capacity and robustness and leads to the best possible results.

The PRReg objective offers several benefits. Firstly, it is model agnostic, which implies that it can be used with various multivariate forecasting models. Secondly, it is a modified version of the CD strategy that incorporates the correlations between different channels. Lastly, it outperforms the CI strategy. It is intuitive that we should not treat each channel independently since they represent the features of the same object. However, the CD strategy often exhibits inferior performance due to its lack of robustness. The PRReg objective successfully resolves this issue, striking a balance between capacity and robustness, thereby achieving superior results.
\begin{table*}[htp]
\centering
\caption{Comparison among forecasters trained with PRReg (varying $\lambda$), CD and CI strategies. The base forecasters are Linear and Transformer. Performance is measured by MSE. The best results of each setting (row) are marked \textbf{bold}. PRReg is able to surpass CD and CI if the $\lambda$ is selected properly. Meaning that it produces a suitable balance between capacity and robustness.}
\scalebox{0.95}{
\begin{tabular}{ccccccccccc}
\toprule
   & \multirow{2}{*}{Model} & \multirow{2}{*}{CD} & \multirow{2}{*}{CI} & \multicolumn{7}{c}{PRReg (ours)}   \\\cline{5-11}
   &   &&& $\lambda=10^{-6}$ & $\lambda=10^{-5}$ & $\lambda=10^{-4}$ & $\lambda=10^{-3}$& $\lambda=10^{-2}$& $\lambda=10^{-1}$& $\lambda=1$\\
   \midrule
\multirow{2}{*}{Electricity}& Linear   & -     & -     & - & - & - & -& -& -& -\\
   & Transformer     & 0.250 & \textbf{0.185}      & 0.218    & 0.219    & 0.227    & 0.269   & 0.378   & 0.742   & 1.527   \\
   \midrule
\multirow{2}{*}{ETTh1}      & Linear   & 0.402 & 0.345 & 0.346    & 0.345    & 0.344    & \textbf{0.342} & 0.355   & 0.426   & 0.737   \\
   & Transformer     & 0.861 & 0.655 & 0.624    & 0.624    & 0.623    & 0.625   & \textbf{0.539} & 0.744   & 1.164   \\
   \midrule
\multirow{2}{*}{ETTh2}      & Linear   & 0.711 & \textbf{0.226}      & 0.335    & 0.329    & 0.296    & 0.248   & 0.239   & 0.259   & 0.298   \\
   & Transformer     & 1.031 & 0.274 & 0.319    & 0.319    & 0.323    & 0.373   & 0.424   & \textbf{0.273} & 0.332   \\
   \midrule
\multirow{2}{*}{ETTm1}      & Linear   & 0.404 & 0.354 & 0.315    & 0.315    & 0.314    & \textbf{0.311} & 0.318   & 0.378   & 0.668   \\
   & Transformer     & 0.458 & 0.379 & 0.374    & 0.374    & 0.375    & \textbf{0.349} & 0.370   & 0.536   & 1.148   \\
   \midrule
\multirow{2}{*}{ETTm2}      & Linear   & 0.161 & 0.147 & 0.142    & 0.142    & 0.140    & \textbf{0.136} & 0.141   & 0.163   & 0.195   \\
   & Transformer     & 0.281 & 0.148 & 0.171    & 0.164    & 0.160    & \textbf{0.144} & 0.149   & 0.182   & 0.211   \\
   \midrule
\multirow{2}{*}{Exchange\_Rate}    & Linear   & 0.119 & 0.051 & 0.050    & 0.050    & 0.049    & 0.046   & 0.043   & \textbf{0.042} & \textbf{0.042} \\
   & Transformer     & 0.511 & 0.101 & 0.098    & 0.098    & 0.092    & 0.074   & 0.056   & \textbf{0.044} & \textbf{0.044} \\
   \midrule
\multirow{2}{*}{Weather}    & Linear   & 0.142 & 0.169 & 0.133    & 0.133    & 0.133    & 0.132   & \textbf{0.131} & 0.141   & 0.169   \\
   & Transformer     & 0.251 & \textbf{0.168}      & 0.189    & 0.191    & 0.195    & 0.180   & 0.189   & 0.297   & 0.193   \\
   \midrule
\multirow{2}{*}{ILI} & Linear   & 2.343 & 2.847 & 2.599    & 2.599    & 2.597    & 2.581   & 2.467   & \textbf{2.299} & 2.693   \\
   & Transformer     & 5.309 & 4.307 & 3.258    & 3.258    & 3.257    & \textbf{3.254} & 3.310   & 3.848   & 4.793   \\
   \bottomrule
\end{tabular}
}
\label{tb:prreg}
\end{table*}

\subsection{Some Other Factors}
We list some of the factors that may influence the performance of CD or CI models by altering their capacity and robustness. It is important to note that capacity and robustness are intertwined in the model selection process, and increasing one often requires decreasing the other. Thus, these factors can impact CD and CI strategies in various ways. We list some factors only for inspiration.

\noindent{\textbf{Low rank layer.}} The low rank assumption is widely used in robust learning~\cite{yuan2009sparse,LiuLY10Robust}. Following the approach proposed in~\cite{Hu22LoRA}, we replace the linear output projection of each attention layer with a low rank linear layer. Specifically, if the weight of the original linear layer is $W \in \R^{m\times n}$, we replace it with $M_1 M_2$, where $M_1 \in \R^{m\times r}$ and $M_2 \in \R^{r\times n}$. We varied the rank reduction rate in ${2,4,8,16,32,64,128,256,512}$, which means that if the rate is 2, $r = \floor{\frac{\min{m,n}}{2}}$. Figure \ref{fig:rank} presents the results of our experiments. As the rank reduction rate increases, the error initially drops and then rises. Low rank regularization reduces the capacity and increases the robustness of a model. Thus, an appropriate choice of the rank can help a CD model perform better.

\begin{figure}[h]
 \vspace{-2mm}
	\centering
	\includegraphics[width=0.5\linewidth]{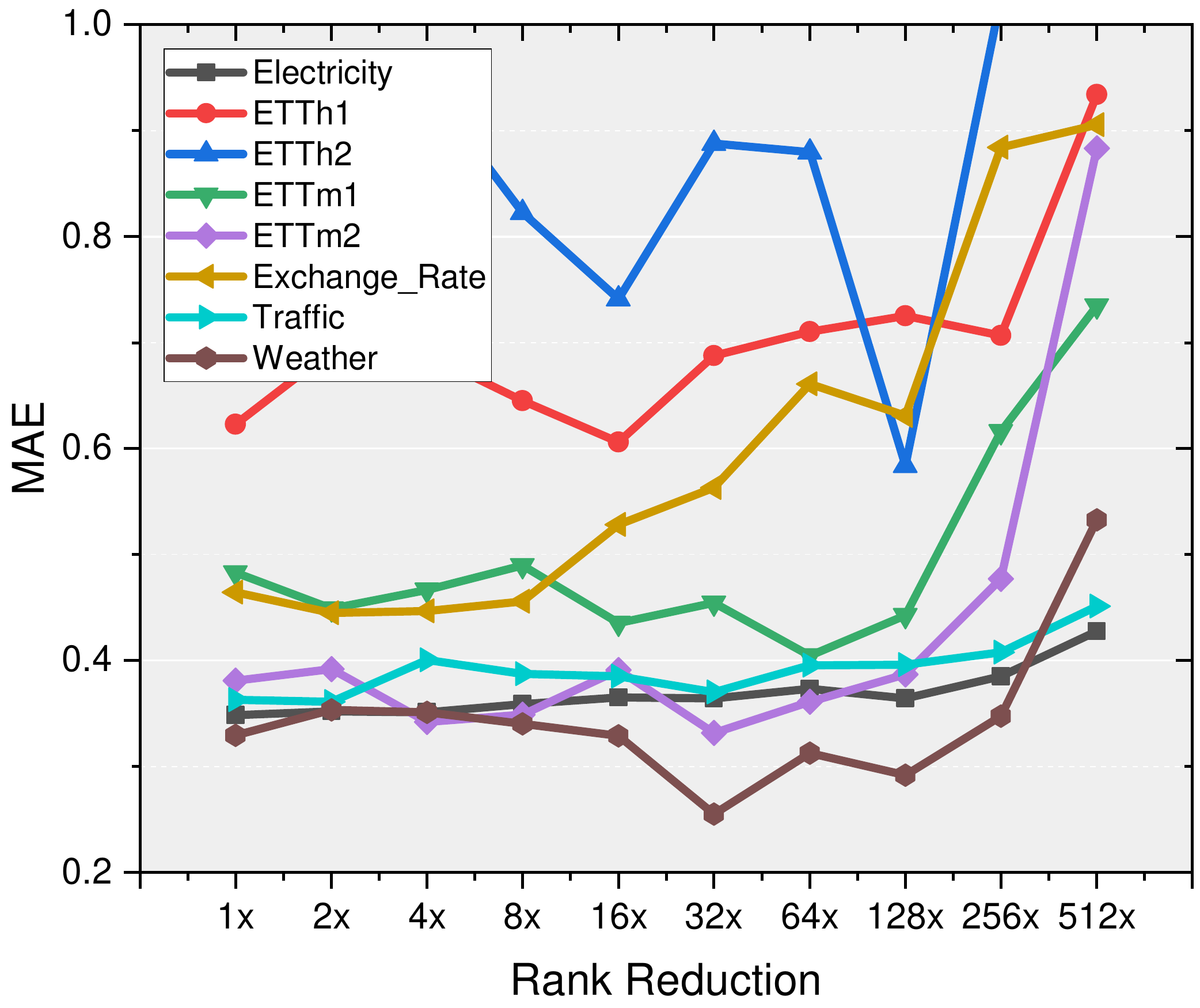}
	\vspace{-3mm}
	\caption{The MAE error of Transformer (CD) model with low rank linear layer on different datasets. We vary the rank reduction rate from 1x to 512x, gradually reducing the rank. The errors drop and rise with increasing reduction rate. Thus, suitable rank regularization helps.}
	\label{fig:rank}
\end{figure}
\noindent{\textbf{Robust loss.}} The Mean Absolute Error (MAE), also known as the L-1 loss, has been demonstrated to be resilient to noisy labels~\cite{GhoshMS15Making,Ghosh17Robust}. Hence, by applying the L-2 loss to a model trained using the CD strategy, its robustness can be improved. To avoid ambiguity, we will refer to the loss used during training as L-2/L-1, and the evaluation metric as Mean Squared Error (MSE)/MAE. The outcomes of applying L-1 and L-2 losses to the Transformer (CD) are presented in Table \ref{tb:l2_l1_loss}. We observe that L-1 loss enhances the robustness of the model, resulting in more accurate predictions.
 
\begin{table}[htp]
	\caption{Performance of Transformer (CD) with L-1 and L-2 loss. When using the L-1 loss which is more robust, Transformer (CD) forecast more accurately. }
 \vspace{-3mm}
	\begin{tabular}{ccccc}
		\toprule
		  Metric         & \multicolumn{2}{c}{MAE} & \multicolumn{2}{c}{MSE}         \\
		\midrule
		    Loss     & L-2    & L-1            & L-2            & L-1            \\
		\midrule
		Electricity    & 0.352  & \textbf{0.347} & \textbf{0.250} & 0.253          \\
		ETTh1          & 0.734  & \textbf{0.604} & 0.861          & \textbf{0.669} \\
		ETTh2          & 0.829  & \textbf{0.536} & 1.031          & \textbf{0.463} \\
		ETTm1          & 0.458  & \textbf{0.421} & 0.458          & \textbf{0.403} \\
		ETTm2          & 0.404  & \textbf{0.328} & 0.281          & \textbf{0.210} \\
		Exchange\_Rate & 0.571  & \textbf{0.451} & 0.511          & \textbf{0.316} \\
		Traffic        & 0.364  & \textbf{0.347} & \textbf{0.645} & 0.668          \\
		Weather        & 0.343  & \textbf{0.274} & 0.251          & \textbf{0.229} \\
		ILI            & 1.508  & \textbf{1.373} & 5.309          & \textbf{4.457}\\
		\bottomrule
	\end{tabular}
\label{tb:l2_l1_loss}
\end{table}

\noindent{\textbf{Length of look-back windows.}}
The length of the look-back window determines the amount of memory that a forecasting model can utilize. In the context of multiple series forecasting, increasing the memory capacity can improve the performance of the global model~\cite{montero2021principles}. Our study demonstrates that the window length also affects the performance of CD and CI strategies in different ways. We varied the sequence length of the input look-back window from 48 to 432, while keeping the horizon fixed at 48. We selected Linear and Transformer models as representative methods to illustrate this phenomenon. The results of comparing the performance of these two models with CD and CI strategies on some datasets are presented in \cref{fig:seq_len}. A longer length of the look-back window provides more information about the historical data but also increases the capacity of the model. For CD models, when the history is not too short, increasing the length often leads to worse performance. On the other hand, CI always benefits from longer window lengths.
\begin{figure}[htp]
 \vspace{-2mm}
    \includegraphics[width=0.24\linewidth]{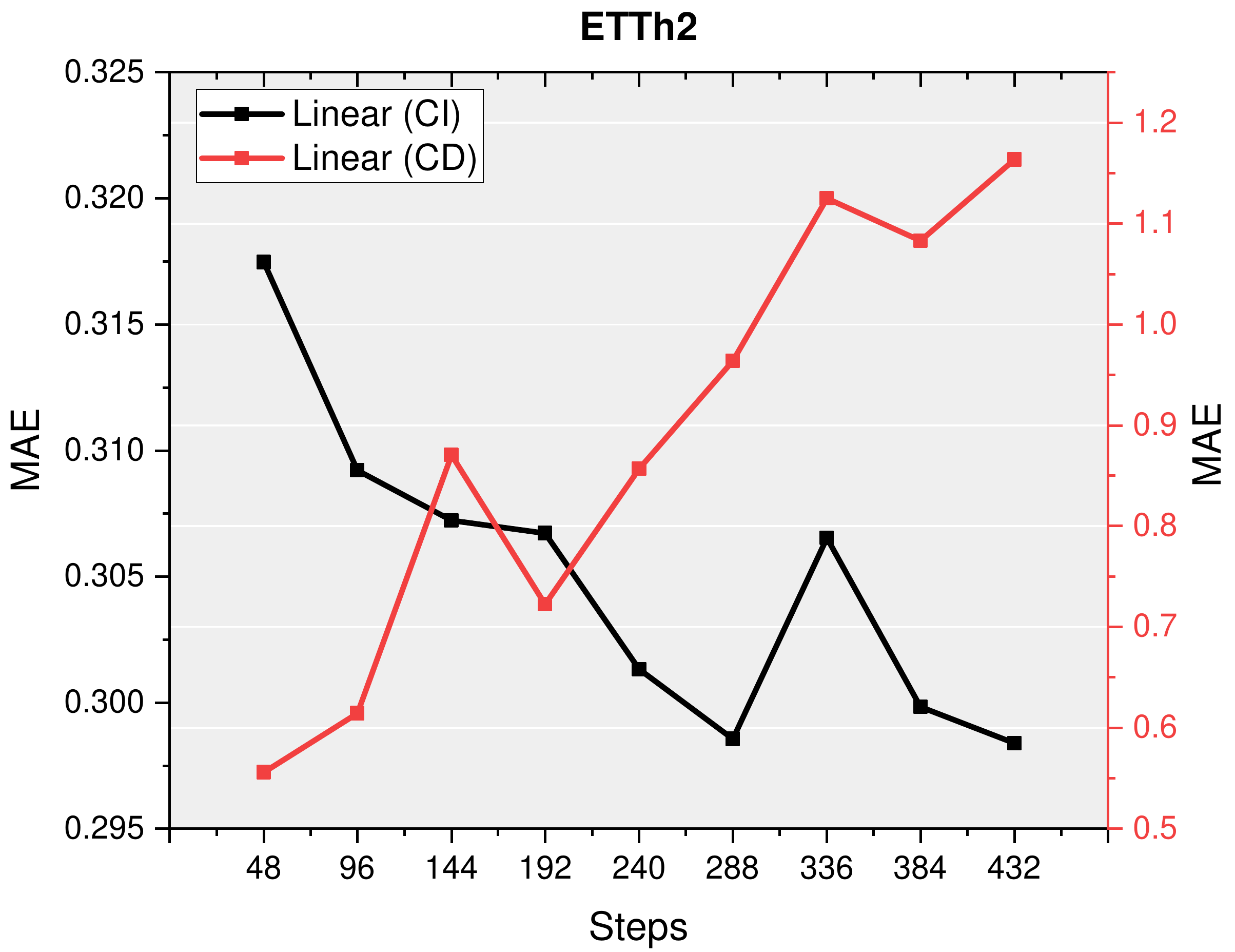}
    \includegraphics[width=0.24\linewidth]{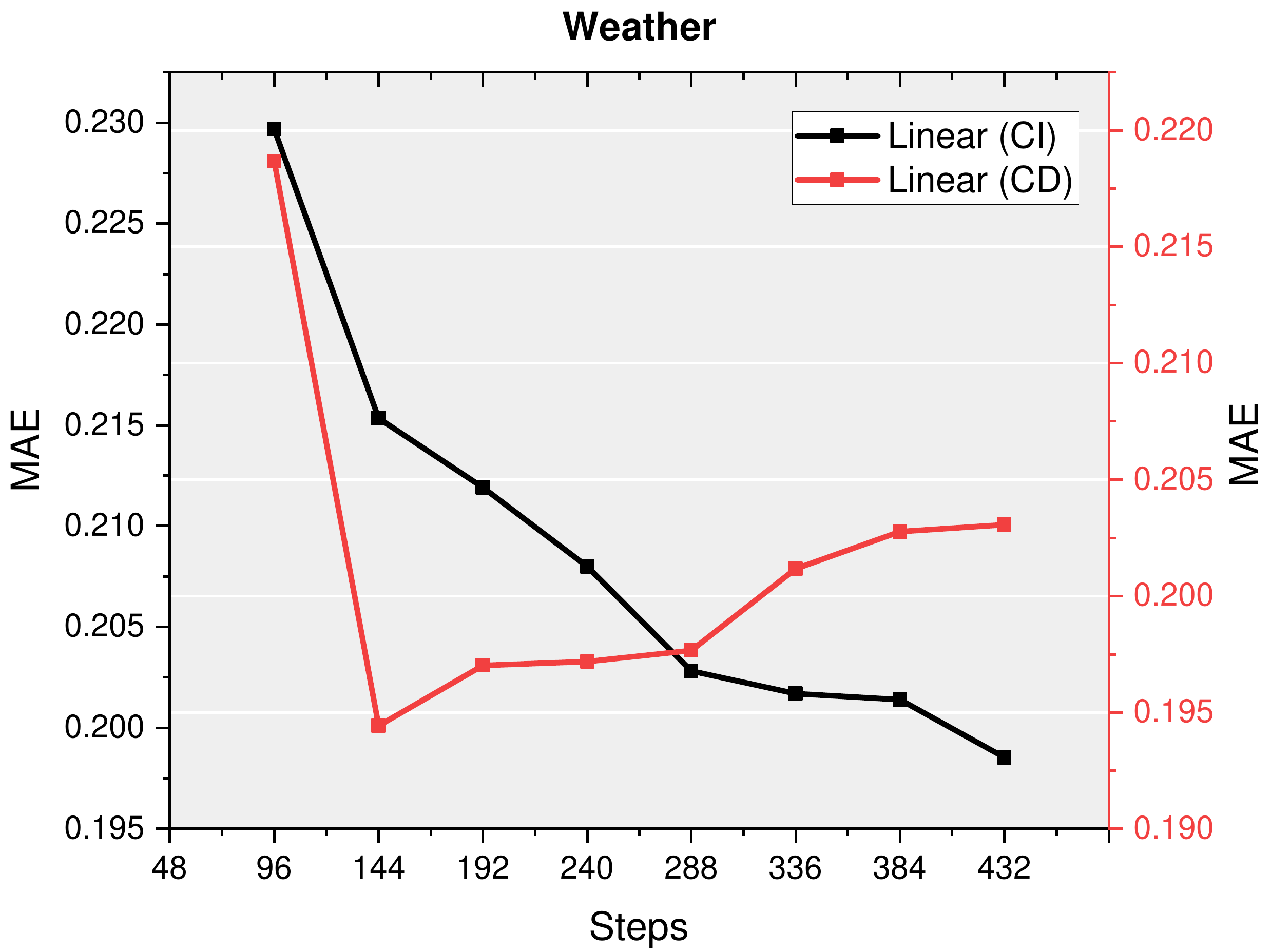}
    \includegraphics[width=0.24\linewidth]{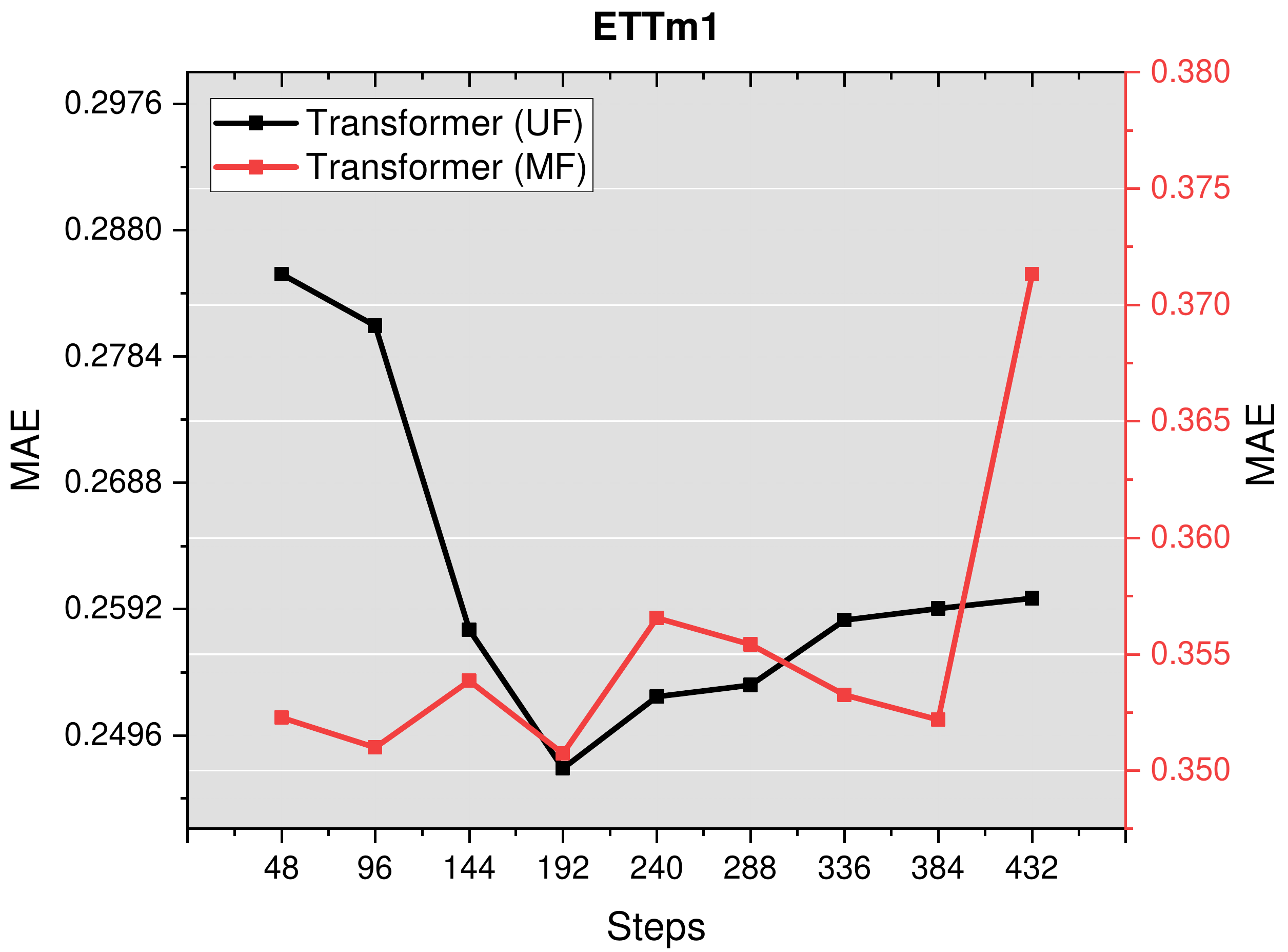}
    \includegraphics[width=0.24\linewidth]{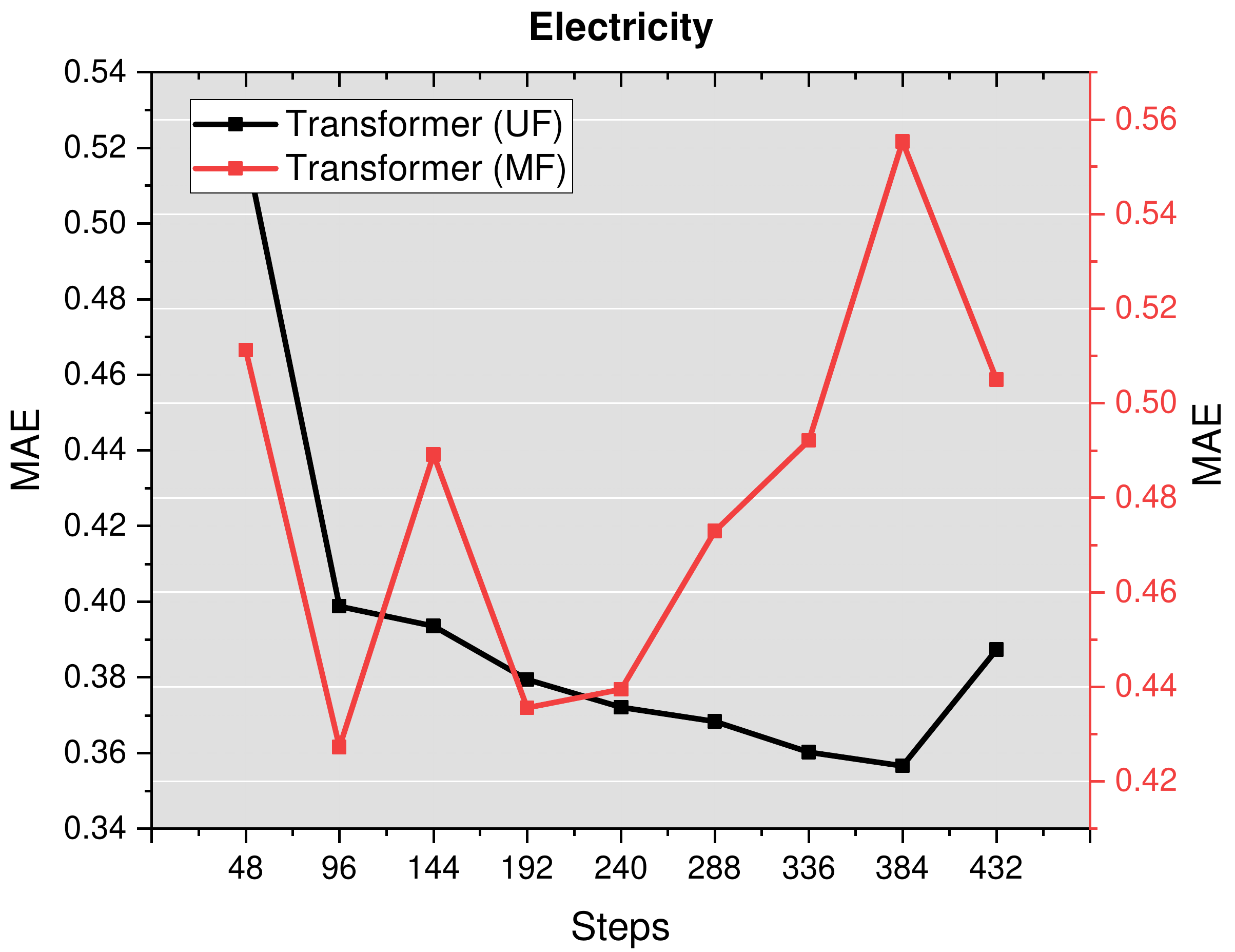}
    \vspace{-3mm}
    \caption{MAE performance of Linear and Transformer with CI and CD strategy on certain datasets. The X-axis represents the length of look-back window. The capacity of both CI and CD model are increased when we input a longer window. We can see that a longer window may do harm to the performance of CD models, while CI can benefit from it.}
    \vspace{-2mm}
    \label{fig:seq_len}
\end{figure}

\section{Discussion about Limitations}

It is important to stress that the conclusions drawn in this paper are closely tied to the characteristics of the datasets employed. While the Channel Independent (CI) training approach generally outperforms the Channel Dependent (CD) strategy, there are exceptions. For instance, as indicated in Table \ref{tb:mae_analysis} for the ILI dataset, CD performs better on average. Nonetheless, the analysis of CI and CD provides valuable insights into the peculiarities of real-world time series and how different strategies can leverage them.

Analyses of this paper may also be limited to numerical channels. Nevertheless, it is possible to handle numerical and non-numerical features separately. We defer an analysis of strategies on more general types of time-series data to future research.
 
 \section{Conclusion}
Recent years have seen the emergence of several methods for long-term Multivariate Time Series Forecasting (MTSF), with some adopting the channel independent (CI) strategy to achieve good performance. By the analyses of this paper, we show the performance boost generated by these methods is often not due to their design, but rather to the training strategy. Despite lower model capacity, the CI strategy exhibits higher robustness, making it better suited for non-stationary time series in practice. We hope that this article will alert the researchers the characteristics of MTSF benchmarks and inspire researchers to better deal with multivariate time series forecasting problems.

\begin{acks}
This research was supported by NSFC (61773198, 62006112,61921006), Collaborative Innovation Center of Novel Software Technology and Industrialization, NSF of Jiangsu Province (BK20200313).
\end{acks}


\bibliographystyle{ACM-Reference-Format}
\bibliography{references}

\end{document}